\documentclass[twoside]{article}
\usepackage{tikz}
\usetikzlibrary{arrows.meta} 
\usepackage[ruled,vlined]{algorithm2e}

\usepackage{hyperref}       
\usepackage{url}            
\usepackage{booktabs}       
\usepackage{amsfonts}       
\usepackage{nicefrac}       
\usepackage{microtype}      
\usepackage{xcolor}         

\usepackage{amsmath, bbm, mathtools, amssymb, amsfonts, mathrsfs,comment, xcolor, amsthm}
\usepackage{thm-restate}

\usepackage[accepted]{aistats2026}
\newcommand{\Expt}[1]{\mathbb E\left( #1 \right)}

\newcommand{\Var}[1]{\mathrm{Var} \left( #1 \right)}
\newcommand{\Cov}[2]{\mathrm{Cov} \left( #1,#2 \right)}

\newcommand{\blue}[1]{\textcolor{blue}{#1}}
\newcommand{\red}[1]{\textcolor{red}{#1}}

\newcommand{\magenta}[1]{\textcolor{magenta}{#1}}

\newcommand{\colorEta}[1]{\textcolor{blue}{#1}}
\newcommand{\colorSuffStat}[1]{\textcolor{red}{#1}}
\newcommand{\colorPsi}[1]{\textcolor{magenta}{#1}}

\newcommand{\MatrixForm}[1]{\mathbf{#1}}

\newcommand{\trace}[1]{\mathrm{tr}\left(#1\right)}

\newcommand{\xset}{\mathcal{X}}
\newcommand{\LCY}{Y}
\newcommand{\thetaknown}{{\nu}_{\mathrm{K}}}
\newcommand{\thetaunknown}{{\nu}_{\mathrm{E}}}

\newcommand{\restrictedETA}{\vec{\eta}(\thetaunknown)}

\newcommand{\R}{\mathbb{R}}

\usepackage{natbib}
\usepackage{graphicx}
\usepackage{subcaption} 
\usepackage{caption}
\usepackage{amsmath}
\newtheorem{theorem}{Theorem}
\newtheorem{corollary}{Corollary}
\newtheorem{definition}{Definition}
\newtheorem{lemma}{Lemma}
\newtheorem{proposition}{Proposition}

\begin{document}

%
\runningtitle{It's all in the (Exponential) Family}

%
\runningauthor{Kang, Wang, Zhang, Pratap, Verma, and Wong}

\twocolumn[
\aistatstitle{It’s All In The (Exponential) Family: An Equivalence Between Maximum Likelihood Estimation and Control Variates For Sketching Algorithms}

\aistatsauthor{
  Keegan Kang\\ Department of  Mathematics \\ and Statistics,\\ Bucknell University, USA \And
  Kerong Wang\\Department of  Mathematics \\ and Statistics,\\ Bucknell University, USA \And
  Ding Zhang\\School of Data Science,\\University of Virginia, USA \AND
  Rameshwar Pratap\\Department of CSE, \\ IIT Hyderabad, India \And
  Bhisham Dev Verma\\Department of Computer Science, \\Wake Forest University, USA \And
  Benedict H. W. Wong\\Institute for Infocomm Research,\\ (I$^2$R), A*STAR, Singapore
}
\vspace*{1cm}


]

\begin{abstract}
Maximum likelihood estimators (MLE) and control variate estimators (CVE) have been used in conjunction with known information across sketching algorithms and applications in machine learning. We prove that under certain conditions in an exponential family, an optimal CVE will achieve the same asymptotic variance as the MLE, giving a fixed point algorithm for the MLE. Experiments show the fixed point algorithm is faster and numerically stable compared to other root finding algorithms for the MLE for the bivariate Normal distribution, and we expect this to hold across distributions satisfying these conditions. We show how this algorithm leads to reproducibility for algorithms using MLE / CVE, and demonstrate how the algorithm leads to finding the MLE when the CV weights are known.
\end{abstract}

\section{INTRODUCTION}\label{Sec: Intro}

Many applications of estimation via sketching algorithms use information about known parameters, such as marginal information, in conjunction with maximum likelihood estimation \citep{shao2003mathematical} or control variates \citep{lavenberg1981perspective}. Examples range from estimating inner products \citep{li2006improving} or angles \citep{kang2018improving} via random projections, to estimating spectral densities of large matrices \citep{adams2018estimating}, to generating sketches of large-scale data \citep{pratap2021variance,pratap2022improving}, and  to sampling from sparse data \citep{li2008one}.

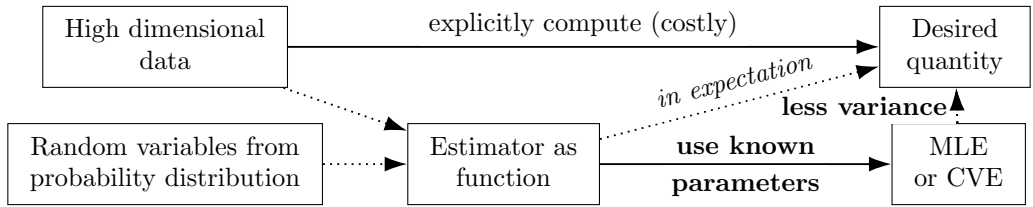
\begin{figure*}[h]
\begin{center}
\begin{tikzpicture}
    \node[draw, rectangle, minimum width=1cm, minimum height=1cm] (databox) at (0,0.06) {$\begin{array}{c}
    \text{High dimensional} \\
    \text{data}
    \end{array}$};

    \node[draw, rectangle, minimum width=2cm, minimum height=1cm] (qtybox) at (10.5,0.06) {$\begin{array}{c}
    \text{Desired} \\
    \text{quantity}
    \end{array}$}; 
    
    \node[draw, rectangle, minimum width=1cm, minimum height=1cm] (rvbox) at (0,-1.5) {$\begin{array}{c}
    \text{Random variables from} \\
    \text{probability distribution}
    \end{array}$};

    \node[draw, rectangle, minimum width=1cm, minimum height=1cm] (estbox) at (4.5,-1.5) {$\begin{array}{c}
    \text{Estimator as} \\
    \text{function}
    \end{array}$};

    \draw[->, dotted, line width=0.25mm, -{Latex[length=3mm, width=2mm]}] (databox) --  (estbox);

    \draw[->, dotted, line width=0.25mm, -{Latex[length=3mm, width=2mm]}] (rvbox) --  (estbox);
    
    \node[draw, rectangle, minimum width=1cm, minimum height=1cm] (mlebox) at (10.55,-1.5) {$\begin{array}{c}
    \text{MLE} \\
    \text{or CVE}
    \end{array}$};

    \draw[->, line width=0.25mm, -{Latex[length=3mm, width=2mm]}] (estbox) --  (mlebox) node[midway, above, fill=white] {{\bf use known}} node[midway, below, fill=white] {{\bf parameters}};

    \draw[->, dotted, line width=0.25mm, -{Latex[length=3mm, width=2mm]}] (mlebox) --  (qtybox) node[midway, left, fill=white] {{\bf less variance}};

    \draw[->, dotted, line width=0.25mm, -{Latex[length=3mm, width=2mm]}] (mlebox) --  (qtybox) ;
        
    \draw[->, dotted, line width=0.25mm, -{Latex[length=3mm, width=2mm]}] (estbox) --  (qtybox) node[midway, sloped, above, fill=white] {{\it in expectation}};

    \draw[->, line width=0.25mm, -{Latex[length=3mm, width=2mm]}] (databox) --  (qtybox) node[midway, above, fill=white] {explicitly compute (costly)};

\end{tikzpicture}
\caption{Idea behind estimation via sketching algorithms. {\bf The goal is to use known parameters with MLE {\it or} CVE to get a reduced variance estimate}. We combine {\bf MLE {\it and} CVE} in Algorithm~\ref{FP_algo_format}  to get an estimator converging to the MLE faster with less numerical error than MLE/CVE alone, under mild conditions.\label{pict_fig1}}
\end{center}
\end{figure*}

Figure~\ref{pict_fig1} shows the common theme in the above applications. Instead of explicitly computing the desired quantity (e.g. inner products, matrix trace), a random variable from a chosen probability distribution is generated, and an estimator (a function of the data and the random variable) of the desired quantity is found. These applications further make use of information about the parameters in conjunction with maximum likelihood estimators (MLE) or a control variate estimator (CVE). This reduces the variance of these estimates, resulting in less error on average. 

Information about parameters can come from the structure of data, e.g. observations in data might be normalized to a length of 1 or centered with zero mean, and this information can be used to get better estimates \citep{li2006improving}. Marginal information could also be cheaply generated on the fly, where estimates with lower variance outweigh the cost of computing marginal information \citep{kang2021correlations}.

When both MLEs and CVEs are used for the same purpose, e.g. feature hashing \citep{verma2022variance} or random projections \citep{kang2021correlations}, the asymptotic theoretical variance reductions are identical. This is unsurprising, as MLEs are  strongly consistent and asymptotically efficient under mild conditions and achieves the Cram{\'e}r Rao lower bound \citep{lehmann1999elements}. When the number of observations (i.e. size of sketch) is small, CVEs can outperform MLEs. {\it Yet there are two core problems behind using CVEs.}

{\bf Problem One:} Replicating the CVE performance on same datasets give mixed results, leading to reproducibility issues. We give an explanation why this is so (and resolve this) in Sections~\ref{sect_discussion} and \ref{sec_expt}.

{\bf Problem Two:} Variance calculations with CVEs involve theoretical parameters for CV weights. But if CV weights rely on estimated data, then the theoretical variance is incorrect, hence it is tricky to analyze CVE performance. We propose Algorithm~\ref{FP_algo_format} with  discussion in Appendix~\ref{ALGODISCUSSION} and~\ref{CONVERGENCEDISCUSSION}.

These problems disappear when the number of observations is large, but the purpose of using sketching algorithms is to reduce error (variance) with as few observations as possible. {\bf Our paper aims to resolve these two issues and understand the performance of CVEs in sketching algorithms. }

We consider distributions from exponential families, previously known as Koopman-Darmois-Pitman families \citep{efron2022exponential}, and give examples across three common exponential families. Feature hashing \citep{weinberger2009feature,verma2022variance} and estimating inner products \citep{indyk1998approximate,li2006improving}) involve the Normal distribution. Frequency estimators \citep{cormode2005improved,pratap2021variance} and angle estimation \citep{charikar2002similarity,kang2018improving} involve the multinomial distribution. Estimating a matrix trace \citep{hutchinson1989stochastic,adams2018estimating} involves the $\chi^2$ distribution. We use geometric properties of exponential families to prove results in this paper. As exponential families may be unfamiliar to some readers, we describe an example of the zero-mean bivariate Normal  in Appendix~\ref{secA1} to give intuition.

\subsection{REVIEW OF EXPONENTIAL FAMILIES}
\label{subsection_efamily_review}
Let $\mathcal X := \{\vec{x}_i\}_{i=1}^n$ comprising $n$ observations $\vec{x}_i \in \mathbb R^D$. Let $p(\xset|\vec{\nu})$ be a probability density function and $\vec{\nu} \in \mathbb R^p$ containing the parameters of the distribution. We write in exponential family form 
\begin{align}
p\left(\xset~|~\vec{\nu}\right)  & \equiv \frac{\exp\left\{n\left(\sum_{j=1}^p \eta_j(\vec{\nu})y_j(\xset) \right)\right\}}{\exp\left\{n\psi(\vec{\eta}(\vec{\nu})) \right\}}  g\left(\xset\right)\notag \\
& = \frac{\exp\left\{n \vec{\eta}^T\vec{y}\right\}}{\exp\left\{n\psi(\vec{\eta}) \right\}} g\left(\xset\right) .\label{basic_exp_fam_eqn}
\end{align}
The vector $\vec{\eta} \in \mathbb R^p$ gives the canonical parameters of the distribution, and is (usually) not $\vec{\nu}$. There is a 1-1 map between $\vec{\eta}$ and $\vec{\nu}$, where each $\eta_j$ can be expressed as a function of $\nu_1,\hdots, \nu_p$, i.e. $\eta_j = f_j(\nu_1,\hdots,\nu_p)$ and equivalently  $\nu_j = g_j(\eta_1,\hdots,\eta_p)$. We use $\eta_j(\vec{\nu})$ in Equation~\eqref{basic_exp_fam_eqn} to denote $\eta_j$ as a function of $\vec{\nu}$, but drop the parenthesis containing $\vec{\nu}$ for ease of notation. The components of $\vec{y}$ are sufficient statistics in terms of the data $\xset$ corresponding to components of $\vec{\eta}$, where the $y_i$s are linearly independent, and we also drop the parentheses containing $\{\vec{x}_i\}_{i=1}^n$ in $y_j$ for ease of notation. $\psi(\vec{\eta})$ is the scaling factor which allows the probability density to integrate to 1.  $g\left(\xset\right)$ is independent of $\vec{\eta}$.  Each $(\eta_i,y_i)$ pair is unique up to an invertible linear transformation of the sufficient statistics. In particular, let $\MatrixForm{A}_{p \times p}$ be invertible, and define $\widetilde{y} = \MatrixForm{A}^{-1}\vec{y}$, $\widetilde{\eta} = \MatrixForm{A}^{T}\vec{\eta}$, and $\widetilde{\psi}(\widetilde{\eta}) = \psi\left(\MatrixForm{A}^{-T}\widetilde{\eta}\right)$.
Then
\begin{align}
p(\xset~|~\vec{\nu}) &= \frac{ \exp\left\{ n\left(\MatrixForm{A}^{T}\vec{\eta}\right)^T \left(\MatrixForm{A}^{-1}\vec{y}\right) \right\}}{\exp\left\{n\psi\left(\vec{\eta}\right) \right\}} g\left(\xset\right) \notag \\
&= \frac{ \exp\left\{ n\widetilde{\eta}^T\widetilde{y} \right\}}{ \exp\left\{n\widetilde{\psi}\left(\widetilde{\eta}\right)
\right\}}g\left(\xset\right).
\label{exp_family_transformed}
\end{align}

\subsection{OUR CONTRIBUTIONS AND PAPER ORGANIZATION}

Suppose data comes from a probability distribution belonging to an exponential family, and we want an estimate of unknown parameters from the distribution, given that some parameters are known. Under this framework, the {\bf main goal} of this paper is to {\it demonstrate an equivalence between MLE and CVE} ({\bf Contribution One}), leading to two important results. a) {\it Reproducibility between algorithms using MLE / CVE} ({\bf Contribution Two}), and b) {\it finding MLEs when the CV weights are known} ({\bf Contribution Three}). 

{\bf Contribution One:} We formally show conditions for equivalence of asymptotic variance estimates via CVE and MLE and state an informal version of Theorem~\ref{equal_variances_theorem_thingy} which appears in Section~\ref{section_equiv}.
\begin{theorem} \label{informal_equal_variances_theorem_thingy}
(Informal Theorem~\ref{equal_variances_theorem_thingy})\\
Suppose our observations come from an exponential family, with parameters divided into a set $\mathcal A$ to be estimated, and $\mathcal B$ which is known, and each $\mu_i \equiv  \frac{\partial \psi(\vec{\eta})}{\partial \eta_i} = \Expt{y_i}$ is equal to $\nu_i$. Then the asymptotic variance reduction of the MLE in estimating parameters in $\mathcal A$ is equal to the  variance reduction given by a CVE where components in $\mathcal B$ are used as control variates.
\end{theorem}

We give a high level overview on how to prove Theorem~\ref{equal_variances_theorem_thingy} in Section~\ref{section_equiv}, and put the main proof in Appendix~\ref{secA3}.

While Theorem~\ref{equal_variances_theorem_thingy} sounds like a known theorem as no other estimator can do better asymptotically than the MLE, there is no formal proof of CVE being equivalent to MLE to the best of our knowledge.  
For example: \cite{glynn2002some} looks at an equivalence of CVE to non-parametric MLE, while we look at equivalence under exponential families. Reference books on exponential families \citep{mclachlan2007algorithm,efron2022exponential} do not give this equivalence. The closest references to hint at an equivalence come from information geometry \citep{amari1987differential,amari2000methods}, where to show asymptotic equivalence between the CVE and the MLE, then the optimal CV weights {\it must minimize the variance within the tangent space of the score functions and making the ancillary submanifold orthogonal within the local linearization around the model manifold}. 
We remark that these references only mention the condition for equivalence, but do not mention CVE at all.

We therefore prove Theorem~\ref{equal_variances_theorem_thingy} using results from exponential families rather than information geometry, in order to give intuition behind our work and make it accessible to more practitioners.

{\bf Contribution Two:} Given Theorem~\ref{equal_variances_theorem_thingy}, Theorem~\ref{theorem_cv_fixed_point_equations} shows that the CVE can be written as fixed point equations that are equivalent to the likelihood stationary equations. Under uniqueness and local convergence conditions in Corollary~\ref{corollary_fixed_points_stationary_points} and Theorem~\ref{theorem_local_convergence_cv_fixed_point}, Algorithm~\ref{FP_algo_format} converges to the MLE from sufficiently close initial values. We find that Algorithm~\ref{FP_algo_format} is empirically faster and  numerically stable compared to several root-finding methods for the MLE with feature hashing \citep{verma2022variance} and random projection \citep{li2006improving,kang2021correlations} and can also resolve the reproducibility issue of CVE.

{\bf Contribution Three:} We describe two heuristics that leads to efficiently finding the MLE via CVE weights (via Algorithm~\ref{FP_algo_format}). We describe these heuristics in Section~\ref{sec_new_HUTCH} and Appendix~\ref{secA5}. More specifically, we can get expressions for the MLE that are tedious to compute such as  Hutchinson’s Trace Estimator.

Section \ref{Sec: main results} of our paper reviews MLEs and CVEs. Section \ref{section_equiv} describes our main results, and Section \ref{sect_discussion} discusses the implications of our results. Section \ref{sec_expt} gives an example of how our work can be used to ensure reproducibility in evaluating algorithms using MLE/CVE. Section \ref{sec_new_HUTCH} demonstrates an example of how our results can be used to find a better (and more efficient) estimate (MLE) from CVE weights. Finally, we conclude in work in Section~\ref{conclusionsection}.

\section{BACKGROUND INFORMATION AND SETTING UP THE PROBLEM}\label{Sec: main results}

We give two theorems related to exponential families which are used in the main proof of Theorem~\ref{equal_variances_theorem_thingy}.

\begin{restatable}{theorem}{meanvarthm}
\label{mean_and_var_exponential_family_thm}
Given $n$ observations, the mean vector $\vec{\mu} \equiv \Expt{\vec{y}}$ and the variance-covariance matrix $\MatrixForm{V}_n \equiv \mathrm{Cov}\left(\vec{y}\right)$ for $\vec{y}$ in the exponential family is given by $\vec{\mu} = \left( \frac{\partial \psi(\vec{\eta})}{\partial \eta_1}, \hdots, \frac{\partial \psi(\vec{\eta})}{\partial \eta_p}\right)$ and $n\MatrixForm{V}_n =  \MatrixForm{V}$ where $\MatrixForm{V} = \left( \begin{array}{c c c}
\frac{\partial^2 \psi(\vec{\eta})}{\partial \eta_1^2} & \hdots & \frac{\partial^2 \psi(\vec{\eta})}{\partial \eta_1 \partial \eta_p} \\
\vdots & \ddots & \vdots \\
\frac{\partial^2 \psi(\vec{\eta})}{\partial \eta_p \eta_1} & \hdots & \frac{\partial^2 \psi(\vec{\eta})}{\partial \eta_p^2}
\end{array}\right) = \left(\begin{array}{c c c}
\frac{\partial \mu_1}{\partial \eta_1} & \hdots & \frac{\partial \mu_1}{\partial \eta_p} \\
\vdots & \ddots & \vdots \\
\frac{\partial \mu_p}{\partial \eta_1} & \hdots & \frac{\partial \mu_p}{\partial \eta_p} \\
\end{array}\right)$.
\end{restatable}

\begin{proof}
A proof can be found in \cite{efron2022exponential} and in the appendix on Page~\pageref{proof_meanvarthm}.
\end{proof}

\begin{theorem} \label{diff_rs_thm}\citep{efron1978geometry,efron2022exponential}  
In exponential families, the corresponding functions $\vec{\eta}$ and $\vec{\mu}$ have 1-1 mappings, given by the differential relationship $d\mu = \MatrixForm{V} d \eta$, and $d\eta = \MatrixForm{V}^{-1}d\mu$.
\end{theorem}

\subsection{REVIEW OF MAXIMUM LIKELIHOOD ESTIMATION}

In exponential families, the log-likelihood with respect to $\vec{\eta}$ is $l\left(\xset|\vec{\eta}\right) = n\left(\vec{\eta}^T\vec{y} - \psi(\vec{\eta})\right) + \log(g(\xset))$, and the score function with respect to $\vec{\eta}$ is given by $l'\left(\xset|\vec{\eta}\right)=\frac{\partial l\left(\xset|\vec{\eta}\right)}{\partial \vec{\eta}} = n(\vec{y} - \vec{\mu})$ by applying Theorem~\ref{mean_and_var_exponential_family_thm}. The MLE for $\vec{\eta}$ is achieved with $\hat{\mu}_i = y_i$ by setting $l'\left(\xset|\vec{\eta}\right)$ to be $0$. The Fisher Information  with respect to $\vec{\eta}$ is  $i_{\vec{\eta}} = \Expt{l'(\xset|\vec{\eta})l'(\xset|\vec{\eta})^T} = nV_n$. The distribution of the MLE of $\vec{\eta}$ converges to a multivariate Normal $\mathcal N(\vec{\eta}, i_{\vec{\eta}}^{-1})$ \citep{lehmann1999elements}. The MLE of $\vec{\nu}$ is found by computing the score function with respect to $\vec{\nu}$ and equating it to zero. Equivalently, if the MLEs $\hat{\eta}_i$ have been computed under the exponential family framework, then the MLE of $\hat{\nu}_i = g_i(\hat{\eta}_1,\hdots,\hat{\eta}_p)$.  

The MLE of $\hat{\nu}$ similarly converges to $\mathcal N(\vec{\nu}, i_{\vec{\nu}}^{-1})$, where $i_{\vec{\nu}}$ is the Fisher Information of $\vec{\nu}$. Without loss of generality, let $\thetaunknown := \{\nu_i\}_{i=1}^t$ be the set of parameters to be estimated, and $\thetaknown := \{\nu_j\}_{j=t+1}^p$ be the set of parameters known. Instead of each $\eta_i$ as a function $f_i(\nu_1,\hdots,\nu_p)$, we restrict $\eta_i$ to $\thetaunknown$, i.e. $\eta_i = \tilde{f}_i(\nu_1,\hdots,\nu_t)$, and write this as $\restrictedETA$. 
With some abuse of notation, the log-likelihood is restricted to $\thetaunknown \in \mathbb R^t$, i.e. $l\left(\xset~\big|~\restrictedETA\right)  = n\left(\restrictedETA^T\vec{y}\right. 
- \psi\left.\left(\restrictedETA\right)\right) + \log(g(\xset))$, and the score function with respect to $\thetaunknown$ is
{\scriptsize
\begin{align}
\frac{\partial~l(\xset|~\restrictedETA)}{\partial \thetaunknown} & = n \left( \frac{\partial \restrictedETA}{\partial \thetaunknown}  \right)^T\vec{y} - n \left(\frac{\partial \psi\left(\restrictedETA\right)}{\partial \restrictedETA}\right)^T\frac{\partial \restrictedETA}{\partial \thetaunknown} \notag \\
 & = n \left(\frac{\partial \restrictedETA}{\partial \thetaunknown}\right)^T \left(\vec{y} - \vec{\mu} \right).\label{curved_ll2} 
\end{align}
}

We defer computing the Fisher Information with respect to $\thetaunknown$ to the proof of Theorem~\ref{equal_variances_theorem_thingy}.

\subsection{REVIEW OF CONTROL VARIATES}
\label{section_CV}
Suppose we want to estimate the mean of a random variable $X$, given by $\Expt{X}$. Further suppose we know the random process that generated $X$, and can use the same process to generate random variables $Y_1, \hdots, Y_k$, where $\Expt{Y_1}, \hdots, \Expt{Y_k}$ are known. Then $Z = X + \sum_{i=1}^k c_i(Y_i - \Expt{Y_i})$ is a CVE for $X$. The $Y_i$s are called control variates (CV). $\Expt{Z} = \Expt{X}$, but the variance of $Z$ is given by $\Var{Z}  = \Var{X} + \sum_{i=1}^k c_i^2 \Var{Y_i} + 2\sum_{i=1}^k c_i\Cov{X}{Y_i} + 2 \sum_{i >j}^k c_ic_j\Cov{Y_i}{Y_j} $. Let the matrix $\MatrixForm{W}_{ij} = \Cov{Y_i}{Y_j}$, $1\leq i,j, \leq k$, noting that $\Cov{Y_i}{Y_i} \equiv \Var{Y_i}$, and let $\vec{d} = (\Cov{X}{Y_1}, \hdots, \Cov{X}{Y_k})^T$. Then the optimal values of $\vec{c}$ minimizing $\Var{Z}$ are given by $\MatrixForm{W} \vec{c} = -\vec{d}$ and are called the CV corrections. The optimal variance given by the CVE is
\begin{align}
\Var{Z} & = \Var{X} + \vec{c}^T\MatrixForm{W}\vec{c} + 2 \vec{c}^T\vec{d} \notag \\
 & = \Var{X} + \vec{c}^T\vec{d}
=\Var{X} - \vec{d}^T\MatrixForm{W}^{-1}\vec{d}. \label{optimal_cv_var_reduct}
\end{align}
From Equation~\eqref{optimal_cv_var_reduct}, $\Var{Z} \leq \Var{X}$, with equality only when $\vec{d} = \vec{0}$ since the covariance matrix $\MatrixForm{W}$ is positive semi-definite and invertible, hence $\MatrixForm{W}^{-1}$ is positive semi-definite as well. Suppose $\thetaunknown = \{\nu_i\}_{i=1}^t$ and we want an estimate of some linear combination of $X = \sum_{i=1}^t \alpha_i\nu_i, \alpha_i \in \mathbb R$, with $\thetaknown = \{\nu_{j}\}_{j = t+1}^p$. We make the following assumptions.

{\bf Assumption One:} A CV leading to the most variance reduction when parameters are linear must necessarily come from the sufficient statistics $y_j$ as sufficient statistics yield the most information about the data. We look at linear combinations of $y_j$s, rather than any arbitrary functions of the $y_j$s, due to the form of the exponential family. We stress that we are not focusing on coming up with zero-variance CVEs along the lines of \cite{oates2017control,south2023regularized}, but more of relating the MLE to the CVE.

{\bf Assumption Two:} Given $(p-t)$ independent known parameters $\nu_{t+1},\hdots,\nu_p$, there ought to exist  $(p-t)$ CV with known means. We describe how to find them in (the proof of) Corollary~\ref{LC_corollary}.

\section{EQUIVALENCE OF MLE AND CV}
\label{section_equiv}

We state our main theorem, and give the high level overview on our proof strategy.

\begin{restatable}{theorem}{maintheoremMLECVE}
\label{equal_variances_theorem_thingy}
Let observations $\xset := \{\vec{x}_i\}_{i=1}^n$ come from an exponential family which is regular and minimal, with $\psi$ twice continuously differentiable and $\MatrixForm{V}=\nabla^2\psi(\vec{\eta})\succ 0$ at the parameter value under consideration. Let $\thetaunknown := \{\nu_i\}_{i=1}^t$ be the set of parameters to be estimated, and $\thetaknown := \{\nu_{j}\}_{j=t+1}^p$ be the set of parameters known. Suppose every $\mu_i \equiv  \frac{\partial \psi(\vec{\eta})}{\partial \eta_i} = \Expt{y_i}$ is equal to $\nu_i$. The asymptotic variance reduction given by the MLE of a linear combination of $\thetaunknown$ is equal to the  variance reduction given by a CVE where components in $\vec{y}$ are used as control variates.
\end{restatable}

The tricky part of proving Theorem~\ref{equal_variances_theorem_thingy} is to show that the {\bf asymptotic variance} of the MLE is the same as the {\bf variance given by the CVE}, by ``standardizing" the  (very different!) terminology in both fields to exponential family notation. 

In brief, we compute the Fisher Information with respect to $\thetaunknown$ (from Equation~\ref{curved_ll2}) which involves the matrix of partial derivatives $\left[\frac{\mathbf{\partial} \eta_i}{\partial \mu_j} \right]_{i,j}, 1 \leq i, j \leq t$ in order to get an expression for the {\bf asymptotic variance of the MLE}. We compute the expression for the {\bf variance of the CVE} which involves the matrix of partial derivatives $\left[\frac{\mathbf{\partial} \mu_i}{\partial \eta_j} \right]_{i,j}, 1 \leq i, j \leq t$. We then apply Lemma~\ref{lemma_invertible_jacobian} (in Appendix~\ref{secA3}) to show how the two matrices of partial derivatives are related, and finally prove the equivalence. Figure~\ref{pict_fig2} gives the gist of our approach, and a full and complete proof is on Page~\pageref{startMLECVEproof} in Appendix~\ref{appendix_proof_equivalence}.

\begin{restatable}{corollary}{lincombcorollary}
\label{LC_corollary}
Theorem~\ref{equal_variances_theorem_thingy} also holds when $\vec{\mu} = \MatrixForm{A}\vec{\nu}$. 
\end{restatable}

\begin{proof}
This follows since we can make the transformation in Equation~\eqref{exp_family_transformed} with $\widetilde{{\eta}} = \MatrixForm{A}\vec{\eta}$ and $\widetilde{y} = \MatrixForm{A}^{-1}\vec{y}$ for $\MatrixForm{A}$ an invertible matrix of appropriate size, before applying Theorem~\ref{equal_variances_theorem_thingy}. A more detailed proof is on Page~\pageref{proof_lincombcor} in the appendix.
\end{proof}
\begin{figure*}
\begin{center}
\begin{tikzpicture}
    \node[draw, rectangle, minimum width=1cm, minimum height=1cm, black] (databox) at (0,0) {$\begin{array}{c}
    \text{\bf Estimate:} \sum_{i=1}^t \alpha_i \nu_i = \sum_{i=1}^t \alpha_i \magenta{y_i} \\
    \begin{array}{r l}
    \thetaunknown & := \{\nu_1,\hdots,\nu_t\} \\
    \thetaknown & := \{\nu_{t+1},\hdots,\nu_p\}
    \end{array}
    \end{array}$};
    
    \node[draw, rectangle, minimum width=1cm, minimum height=1cm, black] (MLEbox) at (0,2.25) {$
    \begin{array}{r c}
    {\text{\bf Score Function:}} & {n\left(\frac{\partial \red{\vec{\eta}(}\thetaunknown\red{)}^T}{\partial \thetaunknown} \right)(\magenta{\vec{y}} - \blue{\vec{\mu}})} \\
    {\text{\bf Asymptotic Variance:}} & {\frac{1}{n} \vec{\alpha}^T \underset{1:t, 1:t}{\left[\frac{\partial \red{\eta}}{\partial \blue{\mu}}\right]}^{-1}\vec{\alpha}} \\
    \end{array}$};
    
    \node[draw, rectangle, minimum width=1cm, minimum height=1cm, black] (CVbox) at (0,-2.25) {$
    \begin{array}{r c}
    {\text{\bf CVE:}} & { \sum_{i=1}^t \alpha_i\magenta{y_i} + \sum_{j=t+1}^pc_j(\magenta{y_j} - \blue{\mu_j})} \\
    {\text{\bf Variance:}} & {\frac{1}{n}\vec{\alpha}^T\left(\underset{1:t, 1:t}{\left[\frac{\partial \blue{\mu}}{\partial \red{\eta}}\right]} - \underset{1:t, (t+1):p}{\left[\frac{\partial \blue{\mu}}{\partial \red{\eta}}\right]}~~\underset{(t+1):p, (t+1):p}{\left[\frac{\partial \blue{\mu}}{\partial \red{\eta}}\right]^{-1}}~~ \underset{{(t+1):p,1:t}}{\left[\frac{\partial \blue{\mu}}{\partial \red{\eta}}\right]}\right)\vec{\alpha}} \\
    \end{array}$};    

     \draw[->, line width=1.5pt] (databox) -- (MLEbox) node[midway, left] {MLE};
     \draw[->, line width=1.5pt] (databox) -- (CVbox) node[midway, left] {CVE};

       \draw[->,line width=1.5pt] (MLEbox.south east) .. controls +(1,-1) and +(1,1) .. node[midway, above] {$\begin{array}{c}
       \phantom{a} \\
       \phantom{a} \\
       \phantom{a}
       \end{array}$\phantom{a}} node[midway, left] {$\begin{array}{c}
       \text{Max $t$} \\
       \text{parameters} \\
       \end{array}$\hspace*{-0.3cm}}  (CVbox.north east);

        \draw[->,line width=1.5pt] (CVbox.north west) .. controls +(-1,1) and +(-1,-1) .. node[midway, right] { $\begin{array}{c}
       \text{Min $p-t$} \\
       \text{parameters} 
       \end{array}$} node[midway, above] {$\begin{array}{c}
       \phantom{a} \\
       \phantom{a} \\
       \phantom{a}
       \end{array}$}  (MLEbox.south west);
       
\end{tikzpicture}
\end{center}
\caption{The exponential family  $\frac{\displaystyle\exp\left\{n  \red{\vec{\eta}(} \thetaunknown,\thetaknown   \red{)^T}\magenta{\vec{y}}  \right\}}{\displaystyle \exp\left\{ n\blue{\psi(\red{\vec{\eta}(}} \thetaunknown,\thetaknown   \red{)} \blue{)}     \right\}}g(\xset)$ for $\MatrixForm{V}^{\mathrm MLE} = \MatrixForm{V}^{\mathrm CVE}$. \label{pict_fig2}}
\end{figure*}

\begin{algorithm}[h]
\SetAlgoLined
\SetKwInOut{Preconditions}{Preconditions}
\SetKwInOut{Require}{Require}
\SetKwInOut{Input}{Input}
\SetKwInOut{Result}{Result}
\Preconditions{Want to estimate some linear combination of $\thetaunknown =\{\nu_1,\hdots,\nu_t\}$ and know $\thetaknown = \{\nu_{t+1},\hdots,\nu_p\}$}
\Require{$t$ control variate expressions of the form $ f_i(\nu_1,\hdots,\nu_t) = y_i + \sum_{j=t+1}^{p} \hat{c}_{ij} (\nu_1,\hdots,\nu_t) (y_j-\mu_j)$ where $\hat{c}_{ij}(\nu_1,\hdots,\nu_t)$ can be evaluated}
\Result{A fixed point of the CVE, which is the MLE of $\thetaunknown$ under Corollary~\ref{corollary_fixed_points_stationary_points}}
Get initial estimates of $\hat{\nu}_1^{(1)} = y_1,\hdots,\hat{\nu}_t^{(1)} = y_t$ and initialize $n = 1$\;
\Repeat{all $|\hat{\nu}_s^{(n)}-\hat{\nu}_s^{(n-1)}|\leq\epsilon$, for $1\leq s\leq t$}
{$\hat{\nu}_1^{(n+1)} = f_1(\hat{\nu}_1^{(n)},\hat{\nu}_2^{(n)},\hdots,\hat{\nu}_t^{(n)})$\;
$\hat{\nu}_2^{(n+1)} = f_2(\hat{\nu}_1^{(n+1)},\hat{\nu}_2^{(n)},\hdots,\hat{\nu}_t^{(n)})$\;
$\hdots$\;
$\hat{\nu}_t^{(n+1)} = f_t(\hat{\nu}_1^{(n+1)},\hat{\nu}_2^{(n+1)},\hdots,\hat{\nu}_t^{(n)})$\;
$n = n+1$\;
}
Return $(\hat{\nu}_1^{(n)},\hdots,\hat{\nu}_t^{(n)})$\;
\caption{Fixed Point Algorithm via Control Variates ({\tt CV-FP}) \label{FP_algo_format}}
\end{algorithm}

Theorem~\ref{equal_variances_theorem_thingy} shows that the MLE and the optimal CVE have the same asymptotic variance reduction. We show in Appendix~\ref{secA3} that the likelihood score equations for $\thetaunknown$ can be written in the same form as the optimal CV coefficients. In particular, for each $1 \leq i \leq t$, the MLE satisfies a fixed point equation
$\nu_i = f_i(\nu_1,\hdots,\nu_t) = y_i + \sum_{j=t+1}^{p}\hat{c}_{ij}({\nu_1,\hdots,\nu_t})(y_j - \mu_j)$
where $\hat{c}_{ij}$s are the optimal CV coefficients and $\mu_j=\nu_j$ is known for $j>t$. Hence, we get a system of fixed point equations whose solutions are the stationary points of the MLE, which we describe in Algorithm~\ref{FP_algo_format}.

\section{DISCUSSION}
\label{sect_discussion}
For $\thetaunknown := \{\nu_i\}_{i=1}^t$, $\thetaknown := \{\nu_{j}\}_{j=t+1}^p$, the method of MLE for estimating $\thetaunknown$ maximizes the score function with respect to $\thetaunknown$  over $t$ terms. The asymptotic variance of the MLE of $\thetaunknown$ is found via the inverse of the Fisher Information. Statistical theory states that an unbiased MLE asymptotically achieves the Cram{\'e}r Rao lower bound, implying no other estimator can do better \citep{lehmann1999elements}. However, finding the optimal $\thetaunknown$ usually involves numerical root-finding methods. On the other hand, the method of CV for estimating $\thetaunknown$ minimizes the variance of the CVE, and is minimized over $(p-t)$ terms. We make three observations, with Figure~\ref{pict_fig2} showing a summary of the equivalence.

{\bf Observation 1:} There is a duality between the variance estimates given by the MLE (maximizing $t$ terms in $\frac{\partial \eta}{\partial \mu}$, by {\it knowing the likelihood function}) and by the CVE (minimizing $p-t$ terms in $\frac{\partial \mu}{\partial \eta}$, {\it by computing relevant second moments}). In cases where we know fewer parameters $p-t$, then using CVE is desirable. Otherwise, using the MLE is desirable.

{\bf Observation 2:} Solving for the MLE requires the score function with respect to $\vec{\nu}$ without exponential families. Under the exponential family framework, we must know the map from $\vec{\eta}$ to $\vec{\nu}$ to compute partial derivatives of the form $\frac{\partial \eta_i}{\partial \nu_j}$ in order to find the MLE as in Equation~\eqref{curved_ll2}. For the CVE, we require partial derivatives of the form $\frac{\partial \mu_i}{\partial \eta_j}$ in order to find the optimal control variate coefficients $\hat{c}_i$ leading to the CV estimate. Using one method over the other depends on what information we have.

{\bf Observation 3:} In most cases, the covariances in the optimal CV corrections $\hat{c}_i$ is a function of $\thetaunknown$. In Equation~\eqref{to_use_in_proof}, solving for $\hat{c}_i$ requires knowing the respective $\Cov{y_i}{y_s} \equiv \Expt{y_iy_s} - \Expt{y_i}\Expt{y_s}, t+1 \leq s \leq p$, yet the terms $y_i$ are what we wanted to estimate and are unknown. This invalidates the theoretical variance calculations of the CVE if substitution of the terms $y_i$ are used. However, Algorithm~\ref{FP_algo_format} implies that the theoretical variance {\bf is} the MLE theoretical variance.

Various papers treat the CV correction $\hat{c}_i$ differently: e.g. \cite{adams2018estimating} computes the empirical covariance and variance in a CVE for Hutchinson's Estimator, while \cite{kang2021correlations} uses initial estimates of $\thetaunknown$ to substitute into respective $\hat{c}_i$s, as well as a constructing a multivariate Normal distribution where terms involving $\thetaunknown$ cancel out in the CV correction for random projections. Our results imply that treating the CVE as a fixed point algorithm in our experiments result in a better and {\it theoretically correct} estimator, in the sense that the variance of the estimator after convergence is equal to the variance of the MLE at that specific number of observations (or sketch size) $k$. We further suggest when conditions in Theorem~\ref{equal_variances_theorem_thingy} or Corollary~\ref{LC_corollary} hold, ({\tt CV-FP}) is empirically faster and numerically stable compared to root finding algorithms such as Newton Raphson and the Secant method, {\it and mitigates the reproducibility issues mentioned in our introduction}.

While Newton Raphson has quadratic convergence, the derivative has to be computed at each update without guarantee of convergence at bad starting points even if the initial estimate of $y_i$ was used. The Secant method has superlinear convergence, though it requires two initial starting observations, which have to be near the root, although they can be some $y_i \pm \epsilon$, or the upper and lower bounds of what $y_i$ could be. Empirically, our experiments in Section~\ref{sec_expt} for the bivariate Normal distribution show a better convergence to the root with {\tt CV-FP} compared to Newton Raphson for the MLE, and gives a reasonable explanation for the differing performance of CVEs compared to MLEs.

\section{EXPERIMENTS}
\label{sec_expt}

We demonstrate the practical applications of this equivalence on two algorithms: feature hashing and random projection with respect to {\it reproducibility}. We briefly discuss the equivalence between CVE and MLE for feature hashing in the main paper and defer to Appendix~\ref{secA4} for a detailed discussion on both algorithms.

Feature hashing is a technique to quickly estimate inner products between vectors $\vec{x}_i,\vec{x}_j$, with applications in similarity search, duplicate detection \citep{nauman2022introduction}, etc. It compresses the vectors to low-dimensional sketches and estimates the inner product from them. Suppose $\vec{v}_i$, $\vec{v}_{j} \in \R^{k}$ are the feature hashing sketch of the vectors $\vec{x}_i, \vec{x}_j \in \R^p$ where $k \ll d$. Then $\langle \vec{v}_i, \vec{v}_j  \rangle$ gives an unbiased estimate of  $\langle \vec{x}_i, \vec{x}_j \rangle$ \citep{weinberger2009feature}. We apply Algorithm~\ref{FP_algo_format} ({\tt CV-FP}) to \cite{verma2022variance} to estimate $\langle \vec{x}_i, \vec{x}_j\rangle$, getting the update step of 
\begin{align}
f_{n+1} &
  = \sum_{s=1}^k v_{is}v_{js}- \frac{ f_n  \|\vec{x}_j\|^2 \left( \sum_{s=1}^k v_{is}^2 - \|\vec{x}_i\|^2  \right)}{ f_n^2 + \|\vec{x}_i\|^2 \|\vec{x}_j\|^2}  
    \notag \\
     & \phantom{xxx} - \frac{f_n\left(\|\vec{x}_i\|^2 \left( \sum_{s=1}^k v_{js}^2 - \|\vec{x}_j\|^2  \right) \right)}{ f_n^2 + \|\vec{x}_i\|^2 \|\vec{x}_j\|^2}  \label{eq:eqn_iter_cv_fh}
\end{align}
where we set $f_1 = \sum_{s=1}^k v_{is}v_{js}$. Further details are in Appendix~\ref{secA4}.

\textbf{Hardware Description:} Experiments were done on a machine having the following configuration: CPU: Intel(R) Core(TM) i7-8750H CPU @ 2.21GHz x 6; Memory: 16 GB; OS: Windows 10. 

We run simulations on randomly generated vector pairs $\vec{x}_1,\vec{x}_2 \in \mathbb R^{10,000}$ with ratios $r \in \{0.1,0.5,1,2,10\}$ where $\|\vec{x}_1\|^2 = r \|\vec{x}_2\|^2$, and angles $\theta \in \{\frac{\pi}{12}, \frac{\pi}{4}, \frac{\pi}{2}, \frac{3\pi}{4}, \frac{11\pi}{12} \}$. Each specific ratio and angle for a vector pair takes $\approx 2$ hours to run, and data generated can be up to 150MB. 

We generate feature hashing sketches, denoted as $\vec{v}_i$ and $\vec{v}_j$, for vector pairs $\vec{x}_{i}$ and $\vec{x}_{j}$ respectively for sketch size ($k$) where  $1\leq k \leq 100$. We compute $\sum_{s=1}^k v_{is}v_{js}$ and refer to it as the baseline estimate. Subsequently, we compute estimates using: a) MLE \citep{verma2022variance} via Newton Raphson ({\tt MLE-NR}), b) MLE via the Secant method ({\tt MLE-Secant}), c) CV using $\sum_{s=1}^k v_{is}v_{js} \approx \langle \vec{x}_i,\vec{x}_j\rangle$ in $\hat{c}_i$ ({\tt CV-Init}), d) CV where the empirical covariance and variance are computed for $\hat{c}_i$ ({\tt CV-Emp}), and e) Algorithm~\ref{FP_algo_format}~ ({\tt CV-FP}). We repeat this for $10,000$ iterations, and record the number of updates a) b) and e) take until convergence.  

\begin{figure*}
    \centering
    \hspace{0.1cm}
    \begin{minipage}[t]{.3\textwidth}
        \centering
    \hspace{-0.55cm}\includegraphics[width=1.0\linewidth]{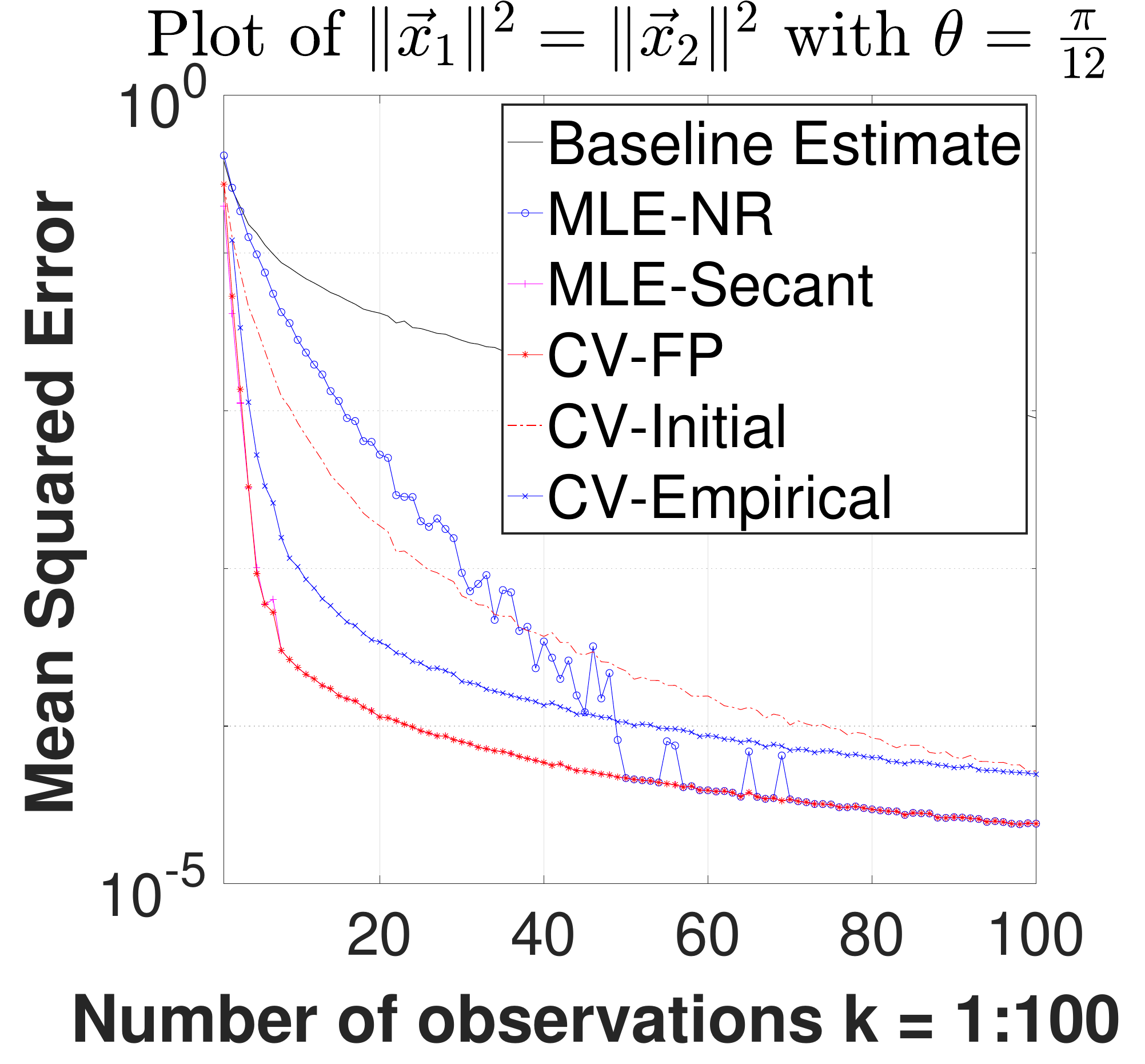}
    \vspace*{-0.2cm}
        \caption{\small MSE plot when \phantom{whenwhen} \mbox{$\|\vec{x}_1\|^2 =\|\vec{x}_2\|^2$ with $\theta = \frac{\pi}{12}$}.}
        \label{single_MSE_FH_main}
    \end{minipage}%
    \hspace{0.1cm}
    \begin{minipage}[t]{0.3\textwidth}
       \centering
        \includegraphics[width=1.0\linewidth]{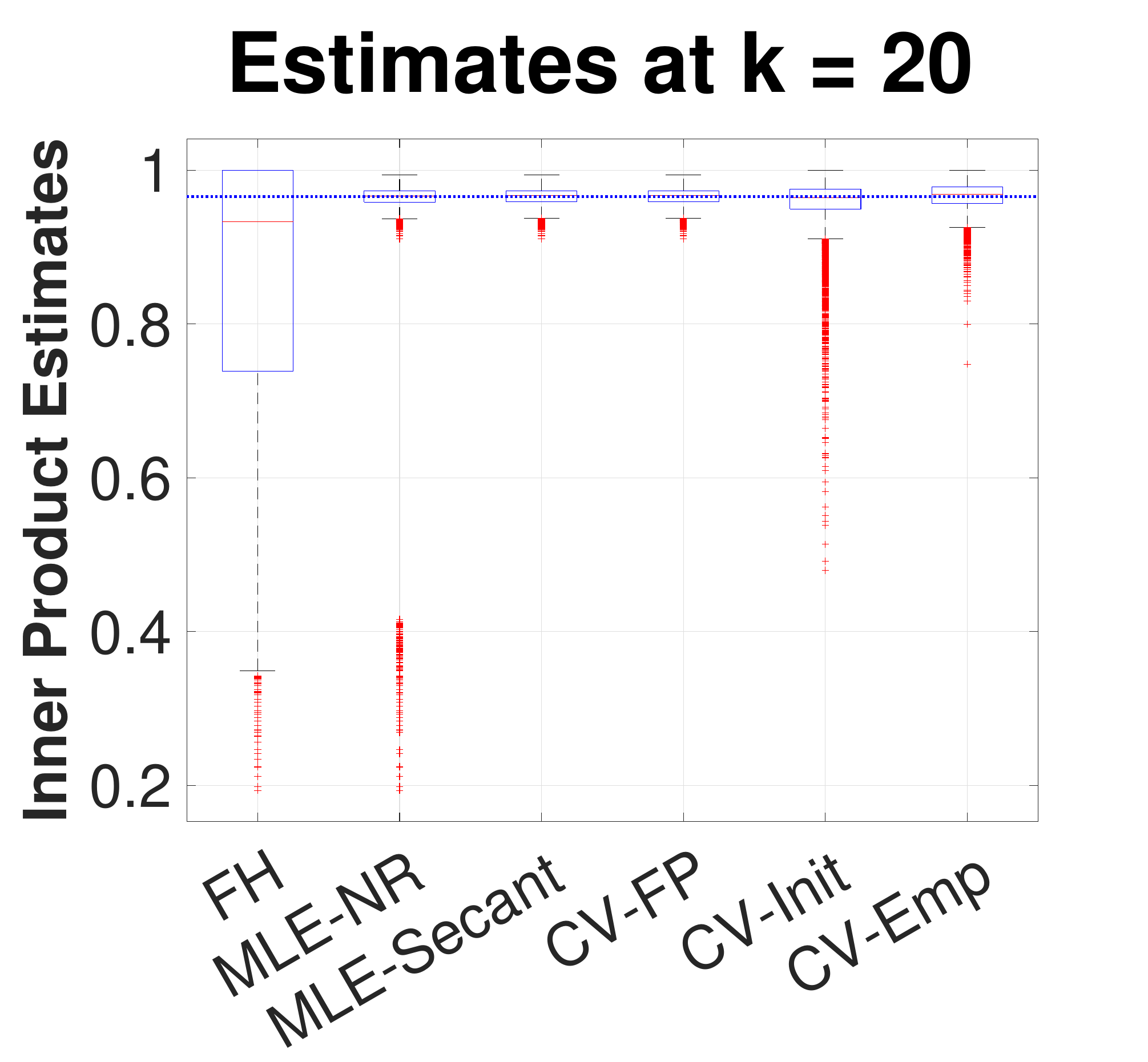}
        \vspace*{-0.6cm}
        \caption{\small Boxplots of estimated inner products. Blue horizontal line denotes true inner product. }
        \label{boxplot_outliers}
    \end{minipage}
    \hspace{0.1cm}
    \begin{minipage}[t]{0.3\textwidth}
        \centering
        \includegraphics[width=1.0\linewidth]{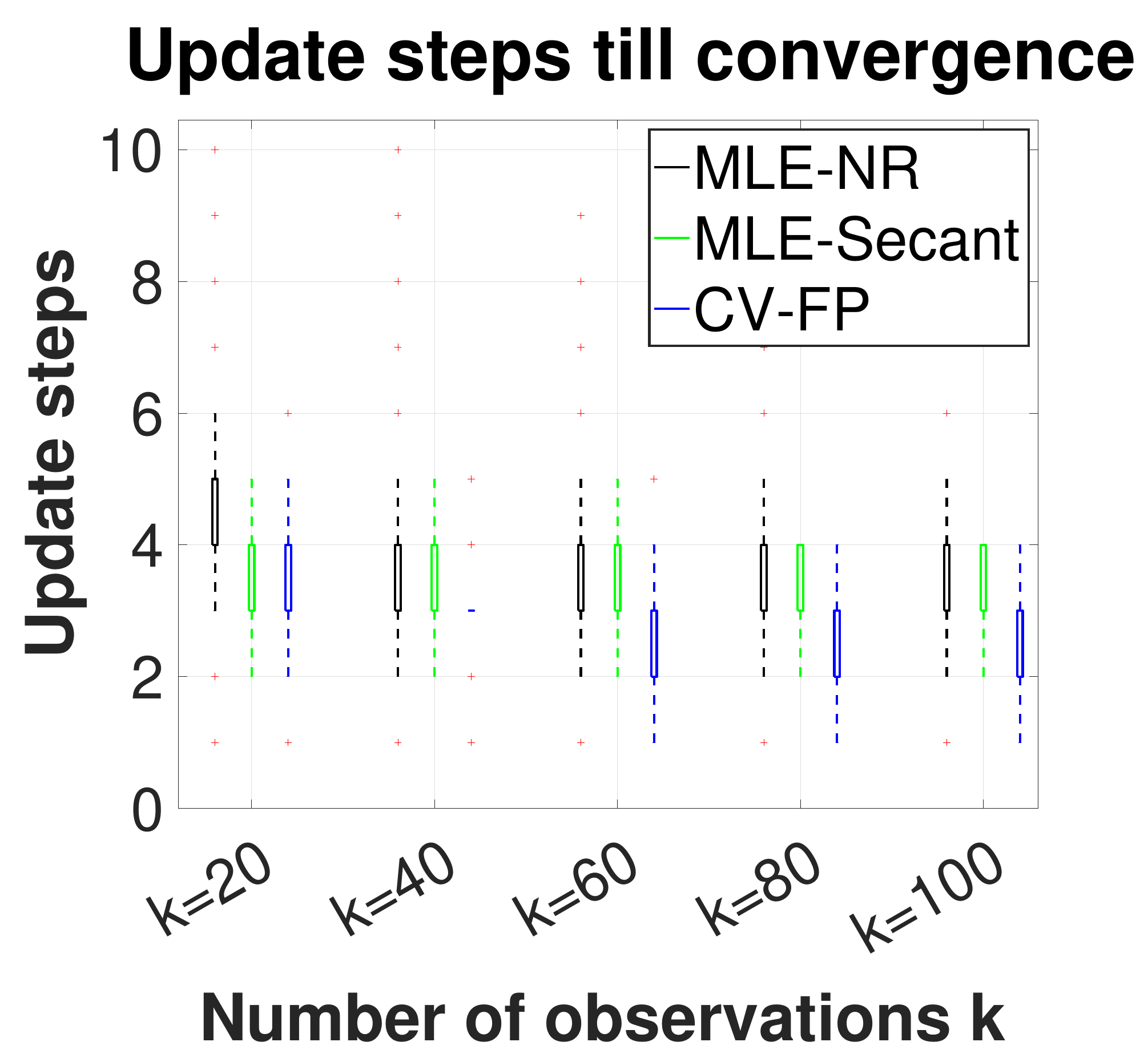}
        \vspace*{-0.6cm}
        \caption{\small Boxplots of update steps till convergence at $k = \{20,40,60,80,100\}$}
        \label{boxplot_update_steps_FH_main}
    \end{minipage}
\end{figure*}

Figure~\ref{single_MSE_FH_main} summarizes the mean square error (MSE) across $10,000$ iterations for various sketch sizes ($k$), considering $\|\vec{x}_1\|^2 = \|\vec{x}_2\|^2$ and $\theta = \frac{\pi}{12}$. A lower MSE indicates better performance. We display $\theta = \frac{\pi}{12}$ since similar vectors are usually of interest with more plots in Appendix~\ref{secA4}. From Figure~\ref{single_MSE_FH_main}, it is clear that the MSE of all five methods is lower than that of the baseline feature hashing estimate. {\tt  CV-FP} and {\tt MLE-Secant} show similar best performances compared to others. Unsurprisingly, out of the control variate implementions, {\tt CV-Init} is the worse, followed by {\tt CV-Emp}, with {\tt CV-FP} being the best, since the theoretical variance calculations are invalidated with empirical values, and only {\tt CV-FP} achieves the MLE variance reduction.

At smaller values of $k$, {\tt MLE-NR} has the worst performance, and after a threshold ($k \approx70$), its performance converges to the performance of {\tt MLE-Secant} and {\tt CV-FP}. Instead of error bars, we give boxplots in Figure~\ref{boxplot_outliers} for the respective estimates for these algorithms at sketch size $k = 20$. We can see the poor performance of {\tt MLE-NR} is mostly due to a higher number of outliers, since the interquartile range is comparable to other methods. Figure~\ref{boxplot_update_steps_FH_main} gives boxplots of update steps till convergence for {\tt MLE-NR}, {\tt MLE-Secant} and {\tt CV-FP}. {\tt CV-FP} generally takes fewer number of steps compared to the other two baselines, indicating the faster empirical convergence of  {\tt CV-FP}. Additional plots across all angles and norms, discussion of converging to the wrong estimate and convergence rates of a), b), and e) are in Appendix~\ref{CONVERGENCEDISCUSSION} and~\ref{secA4}.

{\bf Implication of our experiments on reproducibility:} Most papers involving MLEs and CVEs focus on theoretical work (proving variance bounds) and experiments (comparing against benchmarked algorithms), but do not explicitly mention how respective estimates from the CVE and MLE are found, as the assumption is that a user can implement  them on their own. We suspect this is why there are mixed empirical results when CVE and/or MLE are used.

{\it Our plots show there is a great difference in the methods used: from numerical error (Figure~\ref{single_MSE_FH_main}) to speed of convergence (Figure~\ref{boxplot_update_steps_FH_main}), if a suboptimal implementation of {\tt Newton Raphson} is used for the MLE, or computing $\hat{c}$ via {\tt CV-Init} (or {\tt CV-Emp}) via the CVE}. This gives  contradictory results based on the type of experiments. 

For example, identifying vectors with high similarity is a goal of similarity search, where the angle $\theta$ between the two vectors is small. A difference in root finding methods used can lead to different performance, e.g. {\tt Newton Raphson} has poor performance with small $k$ (Figure~\ref{boxplot_outliers}), but that does not mean the MLE is bad. Equivalently, researchers may use the median (which does not take into account outliers) to show better results (boxplots in Figure~\ref{boxplot_outliers} show {\tt MLE-NR} has similar (good) performance with {\tt MLE-Secant} and {\tt CV-FP} ignoring outliers); but is not realistic in a practical scenario where outliers are not known in advance.  Moreover, if sole error bars are used (based on computing standard deviations of the estimates), this may lead to a conclusion that an empirical MLE has ``higher error", even though the outliers increase standard deviations (Figure~\ref{boxplot_outliers}) and can be misleading. Figure~\ref{boxplot_outliers} shows there are more outliers in one direction for Newton Raphson. Using error bars implicitly assumes ``equal variation above and below the mean", and masks what is happening in practice.

Algorithm~\ref{FP_algo_format} ({\tt CV-FP}) can mitigate these issues.

\section{FINDING THE MLE VIA CVE WEIGHTS}
\label{sec_new_HUTCH}

Two heuristics arising from the construction of {\tt CV-FP} can be used to find the MLE, even if the conditions in Theorem~\ref{equal_variances_theorem_thingy} or Corollary~\ref{LC_corollary} are not satisfied, by numerically solving the fixed point iteration, or solving it analytically.

While we describe an application with respect to sketching algorithms, we believe our technique can be applied in different fields (e.g. Bayesian updates via finite mixture models) where once we find the CVE weights, we can repeatedly iterate over them, or find a closed form solution which gives the MLE.

We give an example for Hutchinson's estimator \citep{hutchinson1989stochastic}, which is used to estimate the trace of a matrix $\MatrixForm{M}$; but has found use in many applications, e.g. estimating the log determinant of a matrix \citep{cortinovis2021randomized} (using the fact that $\log(\det(\MatrixForm{M})) = \trace{\log(\MatrixForm{M})}$), counting number of triangles in graphs (using the fact that number of triangles in graph $G=\frac{1}{6} \trace{\MatrixForm{M}^3}$, where $\MatrixForm{M}$ is the adjacency matrix of $G$) \citep{avron2010counting}. Improved versions of Hutchinson's Estimator, such as Hutch++ \citep{meyer2021hutch++} also involve estimating the trace of a (different) matrix, and reducing the variance of the trace estimation in these cases is desirable. 

Let $\vec{r}_i \in \mathbb R^d$ be distributed $\mathcal N(\vec{0}_d,\MatrixForm{I}_d)$. Let $y_i = \vec{r}_i^T\MatrixForm{M}\vec{r}_i \in \mathbb R$ for $1 \leq i \leq k$. For a symmetric $\MatrixForm{M}_{d\times d}$ matrix, the estimator $ Y = \sum_{i=1}^k \frac{y_i}{k} = \sum_{i=1}^k \frac{\vec{r}_i^T\MatrixForm{M}\vec{r}_i}{k}$ estimates the trace of $\MatrixForm{M}$, with $\Var{Y} = \frac{2\trace{\MatrixForm{M}^2}}{k} = \frac{2\|\MatrixForm{M}\|_F^2}{k}$, where $\|\MatrixForm{M}\|_F^2$ is the Frobenius norm, defined by $\|\MatrixForm{M}\|_F^2 = \sum_{i,j}m_{ij}^2$. \cite{adams2018estimating} gives the following CVE for $\trace{\MatrixForm{M}}$ when there is a known diagonal matrix $\MatrixForm{B}$,
\begin{align*}
Z = \sum_{i=1}^k \frac{\vec{r}_i^T\MatrixForm{M}\vec{r}_i}{k} + c\left( \sum_{i=1}^k \frac{\vec{r}_i^T\MatrixForm{B}\vec{r}_i}{k} - \trace{\MatrixForm{B}} \right),
\end{align*}
and Lemma 4.1 in \cite{adams2018estimating} gives the optimal $\hat{c}$ to be
\begin{align*}
\hat{c} &= -\frac{
\Cov{\vec{r}^T\MatrixForm{M}\vec{r}}{\vec{r}^T\MatrixForm{B}\vec{r}}}{\Var{\vec{r}^T\MatrixForm{B}\vec{r}}} = -\frac{\trace{\MatrixForm{M}\MatrixForm{B}}}{\trace{\MatrixForm{B}^2}},
\end{align*}
with a variance reduction of $\frac{
2\trace{\MatrixForm{M}\MatrixForm{B}}^2}{k\trace{\MatrixForm{B}^2}}$.
The empirical covariance
$\Cov{\vec{r}^T\MatrixForm{M}\vec{r}}
{\vec{r}^T\MatrixForm{B}\vec{r}}$
and variance
$\Var{\vec{r}^T\MatrixForm{B}\vec{r}}$
may be computed from the observed data, since
$\trace{\MatrixForm{M}\MatrixForm{B}}$ is not known (so if empirical covariances and variances are used, the variance reduction is not theoretically correct).

{\it Given the work in \cite{adams2018estimating}, a natural question would be to ask if a MLE could be found, which might yield a lower variance reduction. However, $Y$ is distributed as a linear combination of $\chi^2$ variables, and finding an MLE is non-trivial.}

Our first heuristic is to {\it find linearly independent sufficient statistics} to construct the CVE  and find optimal CV corrections $\hat{c}_i$. We do so for each diagonal term $\MatrixForm{M}_{tt}$ in $\trace{\MatrixForm{M}}$. The new CVE is given by $\Expt{ \vec{r}^T\MatrixForm{M}\vec{r} + \sum_{i=1}^d c_i\left(  \sum_{s=1}^d r_{s} \left(\MatrixForm{B}\vec{r}\right)_s - b_{ii}\right)}$. Detailed derivations for this heuristic is described in the proof of Theorem~\ref{diag_CVE_thm} in Appendix~\ref{hutch_expt_appendix}.

The second heuristic is to {\it treat the CVE as a fixed point iteration (and perhaps find a closed form solution)}. In the case of Hutchinson's estimator, there exists a closed form solution, leading to the estimator $\widehat{\trace{\MatrixForm{M}}} = \sum_{s=1}^d \frac{b_{ss}r_{s}(\MatrixForm{M}\vec{r})_s}{r_s(\MatrixForm{B}\vec{r})_s}$ with overall approximate variance reduction given by $ \frac{2}{k} \sum_{i=1}^d m_{ii}^2$. We remark that even though the conditions for Theorem~\ref{equal_variances_theorem_thingy} and Corollary~\ref{LC_corollary} are not satisfied, we can do standard variance analysis on any new estimator we find. Details are on Pages~\pageref{heur2_appendix} and \pageref{heur2_appendix2} in Appendix \ref{hutch_expt_appendix}.

The estimator we find under our two heuristics is Bekas' diagonal estimator \citep{bekas2007estimator} (when Normal random variables are used). Hence the answer to the question proposed regarding the MLE would be: {\it The computations for the MLE could be bypassed via the CVE (and in fact, has the same form of \cite{bekas2007estimator}), and the MLE would be a biased estimator.} We remark that the original paper of \cite{bekas2007estimator} does not give any variance analysis of their estimator (or further intuition on constructing the estimator), though \cite{baston2022stochastic} builds on their work to construct probability bounds, so our results in Appendix~\ref{hutch_expt_appendix}  can be thought of as a mini-technical report complementing both papers.

To summarize, even if conditions do not hold, numerically solving the fixed point iteration, or finding its closed form analytically can give new estimators, where its variance (and other properties) can be computed.  With respect to our field in sketching algorithms, our heuristic can lead to re-examining past work on estimators to give alternate proofs for variance reduction (and using these variances in probability bounds), and future work on understanding new estimators along similar lines, e.g. applying our work to find better estimators for estimating Euclidean distances under composition of functions \citep{leroux2024euclidean}, or for the Jaccard similarity under circulant permutations \citep{pmlr-v162-li22m}. In both these cases, computing expectations and covariances (hence CVE weights) are much easier than analytically deriving the MLE involved.

\section{CONCLUSION, BROADER IMPACT, AND FUTURE WORK}
\label{conclusionsection}
We have shown a connection in exponential families between $\MatrixForm{V}_{ij} = \frac{\partial \mu_i}{\partial \eta_j}$ and $\MatrixForm{V}^{-1}_{ij} = \frac{\partial \eta_i}{\partial \mu_j}$ for CVE and MLE respectively, with {\tt CV-FP} converging to the MLE, leading to two important results: {\it reproducibility} between algorithms using MLE / CVE, and {\it finding MLEs when the CV weights are known}. We briefly mention the limitations of our work, two broad impacts our work might have, and two directions for future research.

{\bf Limitations:} Our work is limited to exponential families where each $\mu_i$ is a linear combination of $\nu_i$ (and vice versa), and our experiments only look at the bivariate Normal distribution. However, we expect the performance of {\tt CV-FP} to generalize across distributions from exponential families under these conditions. 

{\bf Unify Existing MLE/CVE}: Showing that $\MatrixForm{V}^{\mathrm{MLE}} = \MatrixForm{V}^{\mathrm{CVE}}$ gives an algorithm ({\tt CV-FP}) which is numerically stable and faster via experiments for the Normal distribution. {\tt CV-FP} has the potential to lead to better reproducibility of estimates from existing MLE/CVE, and ensure that evaluation of future estimators against benchmark MLE/CVE is fair when comparing against accuracy and speed.

{\bf Intuition For Existing / New Estimators:} The heuristic of {\it finding linearly independent sufficient statistics} to construct a CVE, and {\it treating the CVE as a fixed point iteration} gives an estimator, even if conditions for Theorem~\ref{equal_variances_theorem_thingy} are not met. The variance reduction from the CVE also gives some intuition of how the initial data can affect variance reduction. 

{\bf Future Work (Extending {\tt CV-FP}):} $\MatrixForm{V}^{\mathrm{MLE}} = \MatrixForm{V}^{\mathrm{CVE}}$ holds when $\nu$s are linear combinations of $\mu$s. There is analogous future work to show the relationship between $\MatrixForm{V}^{\mathrm{MLE}}$ and $ \MatrixForm{V}^{\mathrm{CVE}}$ and convergence of {\tt CV-FP} when this condition does not hold. We gave an example of the $\chi^2$ distribution in Section~\ref{sec_new_HUTCH} where we derived Bekas' estimator a different way, which hints that $\MatrixForm{V}^{\mathrm{MLE}} = \MatrixForm{V}^{\mathrm{CVE}}$  when $\nu$s are non linear combinations of $\mu$s.

{\bf Future Work (Convergence of {\tt CV-FP}):} Simulations in Appendix~\ref{remarks_convergence_rates_rp} show {\tt CV-FP} could either be linear or superlinear, with a small constant $C$. Future research could prove this convergence and give a bound on $C$ for various probability distributions.

{\bf Acknowledgments:} The authors would like to thank the many reviewers who contributed helpful suggestions, the Bucknell Physics and Astronomy Summer Research program, and Pamela Gorkin, Peter McNamara, Sanjay Dharmavaram, and Sara Stoudt for insightful conversations. 
\bibliographystyle{apalike}
\bibliography{biblio}
\newpage

\section*{Checklist}

\begin{enumerate}

  \item For all models and algorithms presented, check if you include:
  \begin{enumerate}
    \item A clear description of the mathematical setting, assumptions, algorithm, and/or model. 

    Yes. We give a clear description of exponential families, MLEs, and CVEs in Section~\ref{Sec: Intro}, Section~\ref{Sec: main results}, Section~\ref{section_equiv}, and Appendix~\ref{secA3}. We also know some readers may need a refresher on exponential families are, so we also recap in Appendix~\ref{secA1}.
    \item An analysis of the properties and complexity (time, space, sample size) of any algorithm. 

    Yes. We give an analysis of convergence in Appendix~\ref{secA3}, and experiments (including convergence properties) in Appendix~\ref{secA4}
    \item (Optional) Anonymized source code, with specification of all dependencies, including external libraries. Yes, see supplementary material.
  \end{enumerate}

  \item For any theoretical claim, check if you include:
  \begin{enumerate}
    \item Statements of the full set of assumptions of all theoretical results. Yes, we give the full set of assumptions for Theorem~\ref{equal_variances_theorem_thingy}.
    \item Complete proofs of all theoretical results. Yes. All proofs are in Appendix~\ref{secA3}.
    \item Clear explanations of any assumptions. Yes, we give clear explanations in Section~\ref{Sec: main results}, Section~\ref{section_equiv}, and Appendix~\ref{secA3}. 
  \end{enumerate}

  \item For all figures and tables that present empirical results, check if you include:
  \begin{enumerate}
    \item The code, data, and instructions needed to reproduce the main experimental results (either in the supplemental material or as a URL). Yes, the code is in the supplementary material. In fact, the instructions are on Page~\pageref{secA4}.
    \item All the training details (e.g., data splits, hyperparameters, how they were chosen). Yes, they are described in Section~\ref{sec_expt} and Appendix~\ref{secA4}.
    \item A clear definition of the specific measure or statistics and error bars (e.g., with respect to the random seed after running experiments multiple times). Yes, they are described in Section~\ref{sec_expt} and Appendix~\ref{secA4}.
    \item A description of the computing infrastructure used. (e.g., type of GPUs, internal cluster, or cloud provider). Yes, they are described in Section~\ref{sec_expt} and Appendix~\ref{secA4}.
  \end{enumerate}

  \item If you are using existing assets (e.g., code, data, models) or curating/releasing new assets, check if you include:
  \begin{enumerate}
    \item Citations of the creator If your work uses existing assets. Yes, we cite authors of algorithms we used.
    \item The license information of the assets, if applicable. Not applicable.
    \item New assets either in the supplemental material or as a URL, if applicable. Not applicable.
    \item Information about consent from data providers/curators. Not applicable.
    \item Discussion of sensible content if applicable, e.g., personally identifiable information or offensive content. Not applicable.
  \end{enumerate}

  \item If you used crowdsourcing or conducted research with human subjects, check if you include:
  \begin{enumerate}
    \item The full text of instructions given to participants and screenshots. Not applicable.
    \item Descriptions of potential participant risks, with links to Institutional Review Board (IRB) approvals if applicable. Not applicable.
    \item The estimated hourly wage paid to participants and the total amount spent on participant compensation. Not applicable.
  \end{enumerate}

\end{enumerate}

\clearpage
\appendix
\thispagestyle{empty}

\onecolumn
\aistatstitle{Supplementary Materials}

\section{EXPONENTIAL FAMILY MODEL WITH BIVARIATE NORMAL}\label{secA1}

We give an example of the notation  in Sections \ref{subsection_efamily_review}-\ref{section_CV} using the zero-mean bivariate Normal. We use this distribution as several authors \citep{li2006improving,kang2021correlations,pratap2021variance,verma2022variance} use either MLE or CVE in the context of feature hashing or similarity search, when {\bf the underlying distribution is the zero-mean bivariate (or multivariate) Normal}.

The probability density function of the bivariate Normal with zero mean for $n$ observations is
\begin{equation}
\label{pdf for BVN}
    p(\xset | \Sigma)= \frac{1}{(2\pi)^n|\Sigma|^{n/2}} \exp\left\{\sum_{i=1}^n -\frac{1}{2} \begin{pmatrix}
        x_{1_i} \\ x_{2_i}
    \end{pmatrix}^T \Sigma^{-1} \begin{pmatrix}
        x_{1_i} \\ x_{2_i}
    \end{pmatrix}\right\}.
\end{equation}

The parameters of the distribution are given by $\Sigma = \left(\begin{array}{c c}
\sigma_{11} & \sigma_{12} \\
\sigma_{21} & \sigma_{22}
\end{array}\right)$, and we write $\vec{\nu} \equiv  (\nu_1,\nu_2,\nu_3) = ( \sigma_{11},\sigma_{22}, \sigma_{12})$ noting that $\sigma_{12}=\sigma_{21}$. Equation~\eqref{pdf for BVN} can be expressed in exponential family form as
\begin{align*}
p(\xset | \Sigma) & = \frac{\exp \left\{-\frac{\sum_{i=1}^n \left( x_{1_i}^2\sigma_{22}+x_{2_i}^2\sigma_{11} - 2x_{1_i}x_{2_i}\sigma_{12}\right)}{2(\sigma_{11}\sigma_{22}-\sigma_{12}^2)} \right\}}{(2\pi)^n|\Sigma|^{n/2}} \\
 & = \frac{\exp \left\{
    -\frac{\sum_{i=1}^n x_{1_i}^2\sigma_{22}+x_{2_i}^2\sigma_{11} - 2x_{1_i}x_{2_i}\sigma_{12}}{2(\sigma_{11}\sigma_{22}-\sigma_{12}^2)}
    \right\}}{\exp \left\{\frac{n}{2} \log (\sigma_{11}\sigma_{22}-\sigma_{12}^2) \right\}} \frac{1}{(2\pi)^\frac{n}{2}}\\
& = \frac{ \exp \left\{ 
    \colorEta{
    \left(\begin{array}{c c c c c}
        \frac{-\sigma_{22}}{2(\sigma_{11}\sigma_{22}-\sigma_{12}^2)}\\
        \frac{-\sigma_{11}}{2(\sigma_{11}\sigma_{22}-\sigma_{12}^2)}\\
        \frac{\sigma_{12}}{(\sigma_{11}\sigma_{22}-\sigma_{12}^2)}
    \end{array}\right)^T
    }
    \colorSuffStat{
    \left(\begin{matrix}
        \sum_{i=1}^n x_{1_i}^2\\
        \sum_{i=1}^n x_{2_i}^2\\
        \sum_{i=1}^n x_{1_i}x_{2_i}\\
    \end{matrix}\right)
    }
    \right\}}
    { \exp\left\{ 
    \colorPsi{
    \frac{n}{2} \left( \log (\sigma_{11}\sigma_{22}-\sigma_{12}^2) 
    \right)
    }
    \right\} } \frac{1}{(2\pi)^{\frac{n}{2}}}\\
& = \frac{ \exp \left\{ n
    \colorEta{
    \left(\begin{array}{c c c c c}
        \frac{-\sigma_{22}}{2(\sigma_{11}\sigma_{22}-\sigma_{12}^2)}\\
        \frac{-\sigma_{11}}{2(\sigma_{11}\sigma_{22}-\sigma_{12}^2)}\\
        \frac{\sigma_{12}}{(\sigma_{11}\sigma_{22}-\sigma_{12}^2)}
    \end{array}\right)^T
    }
    \colorSuffStat{
    \left(\begin{matrix}
        \frac{\sum_{i=1}^n x_{1_i}^2}{n}\\
        \frac{\sum_{i=1}^n x_{2_i}^2}{n}\\
        \frac{\sum_{i=1}^n x_{1_i}x_{2_i}}{n}\\
    \end{matrix}\right)
    }
    \right\}}
    { \exp\left\{ 
    n\colorPsi{
     \left(
    \frac{1}{2} \log (\sigma_{11}\sigma_{22}-\sigma_{12}^2) 
    \right)
    }
    \right\} } \frac{1}{(2\pi)^{\frac{n}{2}}} \\
 &= \frac{\exp\left\{n \colorEta{\vec{\eta}^T}\colorSuffStat{\vec{y}} \right\}}{\exp\left\{n \colorPsi{\psi(\vec{\eta})}\right\}}g\left(\{\vec{x}_i\}_{i=1}^n\right)
\end{align*}
where
\begin{align*}
    \colorEta{\Vec{\eta}^T} &=\colorEta{\left(\begin{array}{c c c}
        \displaystyle\frac{-\sigma_{22}}{2(\sigma_{11}\sigma_{22}-\sigma_{12}^2)},
        &\displaystyle\frac{-\sigma_{11}}{2(\sigma_{11}\sigma_{22}-\sigma_{12}^2)},
        &\displaystyle\frac{\sigma_{12}}{(\sigma_{11}\sigma_{22}-\sigma_{12}^2)}
    \end{array}\right)} \\
    \colorSuffStat{\vec{y}^T} &= \colorSuffStat{
    \left(\begin{matrix}
        \displaystyle\frac{\sum_{i=1}^n x_{1_i}^2}{n}~,~
        &\displaystyle\frac{\sum_{i=1}^n x_{2_i}^2}{n}~,~
        &\displaystyle\frac{\sum_{i=1}^n x_{1_i}x_{2_i}}{n}
    \end{matrix}\right)
    } \\
\vec{\nu}^T & = \left(\sigma_{11}~~,~~\sigma_{22}~~,~~\sigma_{12}   \right) \\
 & = \left(-\frac{2 \eta_2}{4\eta_1\eta_2 - \eta_3^2},~
  -\frac{2 \eta_1}{4\eta_1\eta_2 - \eta_3^2},~\frac{ \eta_3}{4\eta_1\eta_2 - \eta_3^2}
                     \right) \\
\colorPsi{\psi(\vec{\eta})} & = \colorPsi{-\frac{1}{2}\log\left({4 \eta_1 \eta_2 - \eta_3^2} \right) }.
\end{align*}
In practice, we do not need to express $\colorPsi{\psi(\colorEta{\vec{\eta})}}$ in terms of $\colorEta{\vec{\eta}}$ as $\vec{\mu}$ and $\MatrixForm{V}$ can be computed via the multivariate chain rule using partial derivatives since
\begin{align*}
\frac{\partial \colorPsi{\psi(\colorEta{\vec{\eta}})}}{\partial \colorEta{\eta_i}} = \sum_{j=1}^p \frac{\partial \colorPsi{\psi(\colorEta{\vec{\eta}})}}{\partial {\nu_j}} \frac{\partial \nu_j}{\partial \colorEta{\eta_i}}~\forall~1 \leq i \leq p.
\end{align*}

The mean vector $\vec{\mu}$  is computed via Theorem~\ref{mean_and_var_exponential_family_thm} with 
\begin{align*}
\frac{\partial \colorPsi{\psi(\colorEta{\vec{\eta}})}}{\partial \colorEta{\eta_1}} & = -\frac{2 \eta_2}{4\eta_1\eta_2 - \eta_3^2} =  \sigma_{11} \\
\frac{\partial \colorPsi{\psi(\colorEta{\vec{\eta}})}}{\partial \colorEta{\eta_2}} & = -\frac{2 \eta_1}{4\eta_1\eta_2 - \eta_3^2} =  \sigma_{22} \\
\frac{\partial \colorPsi{\psi(\colorEta{\vec{\eta}})}}{\partial \colorEta{\eta_3}} & = \phantom{-}\frac{ \eta_3}{4\eta_1\eta_2 - \eta_3^2} =  \sigma_{12} 
\end{align*}
and they match the expectations of $\colorSuffStat{\vec{y}}$.  $V_n$ is analogously computed via Theorem~\ref{mean_and_var_exponential_family_thm} where
\begin{align*}
\MatrixForm{V}_n = \frac{1}{n}\left(\begin{array}{c c c c c}
\frac{\partial^2 \colorPsi{\psi(\colorEta{\vec{\eta}})}}{\partial \colorEta{\eta_1}^2} & \frac{\partial^2 \colorPsi{\psi(\colorEta{\vec{\eta}})}}{\partial \colorEta{\eta_1} \partial \colorEta{\eta_2}} & \frac{\partial^2 \colorPsi{\psi(\colorEta{\vec{\eta}})}}{\partial \colorEta{\eta_1} \partial \colorEta{\eta_3}}                \\
\frac{\partial^2 \colorPsi{\psi(\colorEta{\vec{\eta}})}}{\partial \colorEta{\eta_2} \partial \colorEta{\eta_1}} & \frac{\partial^2 \colorPsi{\psi(\colorEta{\vec{\eta}})}}{\partial \colorEta{\eta_2}^2}  & \frac{\partial^2 \colorPsi{\psi(\colorEta{\vec{\eta}})}}{\partial \colorEta{\eta_2} \partial \colorEta{\eta_3}}           \\
\frac{\partial^2 \colorPsi{\psi(\colorEta{\vec{\eta}})}}{\partial \colorEta{\eta_3} \partial \colorEta{\eta_1}} & \frac{\partial^2 \colorPsi{\psi(\colorEta{\vec{\eta}})}}{\partial \colorEta{\eta_3} \partial \colorEta{\eta_2}}  & \frac{\partial^2 \colorPsi{\psi(\colorEta{\vec{\eta}})}}{\partial \colorEta{\eta_3}^2}
\end{array}\right) = \frac{1}{n}\left(\begin{array}{c c c c c}
2\sigma_{11}^2 & 2 \sigma_{12}^2 & 2\sigma_{11}\sigma_{12}               \\
2 \sigma_{12}^2 & 2 \sigma_{22}^2  & 2\sigma_{22}\sigma_{12}            \\
2\sigma_{11}\sigma_{12} & 2\sigma_{22}\sigma_{12}  & \sigma_{11}\sigma_{22} +\sigma_{12}^2
\end{array}\right),
\end{align*}

which matches the variance-covariances between $\colorSuffStat{\vec{y}}$.

The MLEs of $\colorEta{\vec{\eta}}$ are given by $\colorSuffStat{\vec{y}}$ from the score function. To find the MLE of $\vec{\nu}$, we can use the invariance properties of MLE with the estimates $\colorEta{\hat{\eta}}$. In the case of the bivariate Normal with zero mean, the MLE of $\colorEta{\vec{\eta}}$ is exactly the same as the MLE of $\vec{\nu}$ (up to some permutation of its entries).

Suppose we know $\sigma_{11}, \sigma_{22}$, and want an MLE for $\sigma_{12}$. We compute the score function given by Equation~\eqref{curved_ll2} to get
\begin{align}
\frac{\partial l(\xset|\vec{\eta}(\sigma_{12}))}{\partial \sigma_{12}} & = n \left(\frac{\partial \vec{\eta}(\sigma_{12})}{\partial \sigma_{12}}\right)^T \left(\colorSuffStat{\vec{y}} - \vec{\mu}  \right) \notag \\
 & = n \left(\begin{array}{c}
 \frac{-\sigma_{22}\sigma_{12}}{(\sigma_{11}\sigma_{22} - \sigma_{12}^2)^2} \\
 \frac{-\sigma_{11}\sigma_{12}}{(\sigma_{11}\sigma_{22} - \sigma_{12}^2)^2}\\
 \frac{\sigma_{11}\sigma_{12} + \sigma_{12}^2}{(\sigma_{11}\sigma_{22} - \sigma_{12}^2)^2}
 \end{array}\right)^T\left(\colorSuffStat{\left(\begin{array}{c}
        \frac{\sum_{i=1}^n x_{1_i}^2}{n}\\
        \frac{\sum_{i=1}^n x_{2_i}^2}{n}\\
        \frac{\sum_{i=1}^n x_{1_i}x_{2_i}}{n}\\
 \end{array}\right)} - \left(\begin{array}{c}
 \sigma_{11} \\
 \sigma_{22} \\
 \sigma_{12}
 \end{array}\right)\right). \label{bivariate_normal_eg_score_mle_fn}
\end{align}

Setting the score function in Equation~\eqref{bivariate_normal_eg_score_mle_fn} to zero 
leads to
\begin{align}
-\sigma_{22}\sigma_{12}\left(\frac{\sum_{i=1}^n x_{1_i}^2}{n} - \sigma_{11}    \right) - \sigma_{11}\sigma_{12}\left(      \frac{\sum_{i=1}^n x_{2_i}^2}{n} - \sigma_{22} \right) + (\sigma_{11}\sigma_{12} + \sigma_{12}^2) \left(\frac{\sum_{i=1}^n x_{1_i}x_{2_i}}{n} - \sigma_{12} \right) = 0 \label{MLE_cubic_eqn_normal}
\end{align}

where $\frac{n}{\sigma_{11}\sigma_{22}-\sigma_{12}^2}$ is factored out, since $\mathrm{det}(\Sigma) = \sigma_{11}\sigma_{22}-\sigma_{12}^2 > 0$, due to positive-definiteness of the covariance matrix.

The LHS of Equation~\ref{MLE_cubic_eqn_normal} is a cubic in $\sigma_{12}$, and can be solved via root-finding methods to get $\hat{\sigma}_{12}$.

Using the differential relationship in Theorem~\ref{diff_rs_thm}, we have $\left(\frac{\partial \mu_1}{\partial \sigma_{12}}, \hdots, \frac{\partial \mu_3}{\partial \sigma_{12}}\right)^T = \frac{1}{n}(0,0,1)^T = \MatrixForm{V}_n \frac{\partial \vec{\eta}(\sigma_{12})}{\partial \sigma_{12}}$. From Equation~\eqref{curved_ll2}, the Fisher Information with respect to $\sigma_{12}$ is
\begin{align*}
i_{{\sigma}_{12}}\ & = \Expt{n^2 \frac{\partial \vec{\eta}(\sigma_{12})}{\partial \sigma_{12}} \left( \colorSuffStat{\vec{y}} - \frac{\partial \psi(\vec{\eta})}{\partial \vec{\eta}}  \right)\left( \colorSuffStat{\vec{y}} - \frac{\partial \psi(\vec{\eta})}{\partial \vec{\eta}}  \right)^T\frac{\partial \vec{\eta}(\sigma_{12})}{\partial \sigma_{12}}} \\
& = n^2 \frac{\partial \vec{\eta}(\sigma_{12})}{\partial \sigma_{12}} \MatrixForm{V} \frac{\partial \vec{\eta}(\sigma_{12})}{\partial \sigma_{12}} =n \frac{\partial \eta_3}{\partial \sigma_{12}}.
\end{align*}

This implies the distribution of the MLE for $\hat{\sigma}_{12}$ converges to $\mathcal N\left(\sigma_{12}, \frac{1}{n  \frac{\partial \eta_3}{\partial \sigma_{12}}}\right)$, i.e. the asymptotic variance of $\hat{\sigma}_{12}$ is $\frac{(\sigma_{11}\sigma_{12}-\sigma_{12}^2)^2}{n(\sigma_{11}\sigma_{12} + \sigma_{12}^2)}$. Now suppose we know $\sigma_{11}, \sigma_{22}$, and want a CVE for $\sigma_{12}$. The CVE given by the sufficient statistics is
\begin{align}
\hat{\sigma}_{12} & = \Expt{\colorSuffStat{\frac{\sum_{i=1}^n x_{1i}x_{2i}}{n}} + c_1\left(\colorSuffStat{\frac{\sum_{i=1}^n x_{1i}^2}{n}}  - \sigma_{11}\right) + c_2\left(\colorSuffStat{\frac{\sum_{i=1}^n x_{2i}^2}{n}}   - \sigma_{22}\right)} \notag \\ 
& = \Expt{y_3 + \sum_{i=1}^2 c_i(\colorSuffStat{y}_i - \mu_i)}. \label{cv_bivariate_normal_eg_eqn1}
\end{align}
The variance of $\hat{\sigma}_{12}$ in Equation~\eqref{cv_bivariate_normal_eg_eqn1} is given by
\begin{align}
\Var{\colorSuffStat{y_3}} + \sum_{i=1}^2 c_i^2 \Var{\colorSuffStat{y_i}} +2\sum_{i=1}^2c_i\Cov{\colorSuffStat{{y}_i}}{\colorSuffStat{{y}_3}} + 2 \sum_{i>j }^2 c_ic_j \Cov{\colorSuffStat{y_i}}{\colorSuffStat{y_j}}, \notag
\end{align}
and the optimal value of $\hat{c}_i$ given by $
\widetilde{\MatrixForm{V}} \left(\begin{array}{c}
c_1 \\
c_2 \\
\end{array}\right) = -\vec{d} $ with $\widetilde{\MatrixForm{V}} = \left(\begin{array}{c c c c}
\Var{\colorSuffStat{y_1}} & \Cov{\colorSuffStat{y_1}}{\colorSuffStat{y_2}}  \\
\Cov{\colorSuffStat{y_2}}{\colorSuffStat{y_1}} & \Var{\colorSuffStat{y_2}}
\end{array}\right)$ and $\vec{d} = \left(\begin{array}{c}
\Cov{\colorSuffStat{y_1}}{\colorSuffStat{y}_3} \\
\Cov{\colorSuffStat{y_2}}{\colorSuffStat{y}_3} 
\end{array}\right)$. Solving for $\hat{c}_1$ and $\hat{c}_2$ gives
\begin{align*}
\hat{c}_1 = -\frac{\sigma_{12} \sigma_{22}}{\sigma_{12}^2 + \sigma_{11}\sigma_{22}}~~~~~\hat{c}_2 = -\frac{\sigma_{12} \sigma_{11}}{\sigma_{12}^2 + \sigma_{11}\sigma_{22}}.
\end{align*}
and thus the CVE is given by
{\small
\begin{align}
\hat{\sigma}_{12} = \Expt{\colorSuffStat{\frac{\sum_{i=1}^n x_{1i}x_{2i}}{n}} - \frac{\sigma_{12} \sigma_{22}}{\sigma_{12}^2 + \sigma_{11}\sigma_{22}}\left( \colorSuffStat{\frac{\sum_{i=1}^n x_{1i}^2}{n}}  - \sigma_{11} \right) - \frac{\sigma_{12} \sigma_{11}}{\sigma_{12}^2 + \sigma_{11}\sigma_{22}}\left( \colorSuffStat{\frac{\sum_{i=1}^n x_{2i}^2}{n}}  - \sigma_{22} \right)} \label{CV_eqn_bivar_norm}
\end{align}
}

The variance-covariance matrix $\MatrixForm{V}_n$ can be partitioned as
\begin{align}
\MatrixForm{V}_n & = \left(\begin{array}{c | c}
\widetilde{\MatrixForm{V}} & \vec{d}^T \\ \hline
\vec{d} & \Var{\colorSuffStat{y_3}}
\end{array}\right), \notag
\end{align}
hence the variance given by the CVE is
\begin{align}
\Var{\hat{\sigma}_{12}} & = \Var{\colorSuffStat{y_3}} - \vec{d}^T\widetilde{\MatrixForm{V}}^{-1}\vec{d} \notag \\
 & = \MatrixForm{V}_{3,3} - \MatrixForm{V}_{3,1:2}(\MatrixForm{V}_{1:2,1:2})^{-1}\MatrixForm{V}_{1:2,3} \notag \\
 & = \frac{(\sigma_{11}\sigma_{22}-\sigma_{12}^2)^2}{n(\sigma_{11}\sigma_{22} + \sigma_{12}^2)} \label{control_variate_bivariate_variance_eg}
\end{align}

\label{observations_in_appendix}
We make the following observations. 
\begin{enumerate}
\item The asymptotic variance given by the MLE of $\hat{\sigma}_{12}$ is $\frac{(\sigma_{11}\sigma_{12}-\sigma_{12}^2)^2}{n(\sigma_{11}\sigma_{12} + \sigma_{12}^2)}$. However,  Equation~\eqref{control_variate_bivariate_variance_eg} also gives an identical variance.

\item The CVE for $\hat{\sigma}_{12}$ in Equation~\eqref{CV_eqn_bivar_norm} contains $\sigma_{12}$ in the optimal $\hat{c}_1$, $\hat{c}_2$, which we do not have, since the goal was to estimate $\sigma_{12}$. Substituting the empirical values of $\hat{\sigma}_{12}$ from the data would invalidate the variance calculations for the CVE for Equation~\eqref{control_variate_bivariate_variance_eg}. 

\item This naturally leads to a question: {\it How can the asymptotic variance for the MLE be ``correct" when $n$ is large, but the same variance for the CVE be ``wrong" (invalid) since $\hat{c}_1, \hat{c}_2$ uses the empirical $\hat{\sigma}_{12}$?}

\item We observe  that if the expectation was removed from Equation~\eqref{CV_eqn_bivar_norm}, and we let
{\small
\begin{align}
{\sigma}_{12} = {\colorSuffStat{\frac{\sum_{i=1}^n x_{1i}x_{2i}}{n}} - \frac{\sigma_{12} \sigma_{22}}{\sigma_{12}^2 + \sigma_{11}\sigma_{22}}\left( \colorSuffStat{\frac{\sum_{i=1}^n x_{1i}^2}{n}}  - \sigma_{11} \right) - \frac{\sigma_{12} \sigma_{11}}{\sigma_{12}^2 + \sigma_{11}\sigma_{22}}\left( \colorSuffStat{\frac{\sum_{i=1}^n x_{2i}^2}{n}}  - \sigma_{22} \right)}, \label{CV_no_expt}
\end{align}
}

collecting terms on one side of Equation~\ref{CV_no_expt} yields the same expression as Equation~\ref{MLE_cubic_eqn_normal}, up to scalar  multiples. 

\item Observation 1 seems to suggest that we can somehow make use of the CVE to get an estimate of $\hat{\sigma}_{12}$, even if $\hat{c}_1, \hat{c}_2$ contain $\sigma_{12}$ since the MLE and CVE have the same (asymptotic) variance for $\hat{\sigma}_{12}$. Observation 4 somehow suggests a fixed point iteration (or similar) given the same terms in Equation~\ref{MLE_cubic_eqn_normal} and Equation~\ref{CV_no_expt}, which motivates our  Algorithm~\ref{FP_algo_format}.

\end{enumerate}

In the main text of the paper, we will give conditions for the equivalence of variances for the estimators under MLE and CVE. We show that the structure  of the CVE and MLE leads to a fixed point algorithm which converges to the required estimates.

ionion
\newpage

\section{PROOFS}\label{secA3}

\subsection{PRELIMINARY LEMMA AND THEOREM}

We give a lemma and a theorem which we use in the main text. 

\begin{lemma} \label{lemma_invertible_jacobian}
Suppose we have an invertible map between $\vec{\eta} \in \mathbb R^p$ and $\vec{\mu} \in \mathbb R^p$, and the Jacobian is given by $\MatrixForm{V} = \left(\begin{array}{c c c}
\frac{\partial \mu_1}{\partial \eta_1} & \hdots & \frac{\partial \mu_1}{\partial \eta_p} \\
\vdots & \ddots & \vdots \\
\frac{\partial \mu_p}{\partial \eta_1} & \hdots & \frac{\partial \mu_p}{\partial \eta_p} \\
\end{array}
\right)$. Then $\MatrixForm{V}^{-1} = \left(\begin{array}{c c c}
\frac{\partial \eta_1}{\partial \mu_1} & \hdots & \frac{\partial \eta_1}{\partial \mu_p} \\
\vdots & \ddots & \vdots \\
\frac{\partial \eta_p}{\partial \mu_1} & \hdots & \frac{\partial \eta_p}{\partial \mu_p} \\
\end{array}
\right)$.
\end{lemma}
\begin{proof}
Consider $\MatrixForm{G} = \MatrixForm{V}\MatrixForm{V}^{-1}$. Then $\MatrixForm{G}_{ii} = \sum_{s=1}^p \frac{\partial \mu_i}{\partial \eta_s} \frac{\partial \eta_s}{\partial \mu_i} = \frac{\partial \mu_i}{\partial \mu_i} = 1$ and $\MatrixForm{G}_{ij} = \sum_{s=1}^p \frac{\partial \mu_i}{\partial \eta_s} \frac{\partial \eta_s}{\partial \mu_j} = \frac{\partial \mu_i}{\partial \mu_j} = 0$ for $i \neq j$ by the multivariate chain rule, implying $\MatrixForm{G} = \MatrixForm{I}$.
\end{proof}

\begin{theorem} (Schur Complement) \label{theorem_schur} \\
Suppose we have an invertible block matrix $\MatrixForm{M} = \left(\begin{array}{c |c}
\MatrixForm{A} & \MatrixForm{B} \\ \hline
\MatrixForm{C} & \MatrixForm{D}
\end{array}\right)$. If $\MatrixForm{A}$ and $\MatrixForm{D}$ are invertible, then \\
\phantom{abcdeg}$\MatrixForm{M}^{-1} = \left(\begin{array}{c|c}
(\MatrixForm{A} - \MatrixForm{B}\MatrixForm{D}^{-1}\MatrixForm{C})^{-1} & -(\MatrixForm{A}-\MatrixForm{B}\MatrixForm{D}^{-1}\MatrixForm{C})^{-1}\MatrixForm{B}\MatrixForm{D}^{-1} \\ \hline
-\MatrixForm{D}^{-1}\MatrixForm{C}(\MatrixForm{A}-\MatrixForm{B}\MatrixForm{D}^{-1}\MatrixForm{C})^{-1} & \MatrixForm{D}^{-1} + \MatrixForm{D}^{-1}\MatrixForm{C}(\MatrixForm{A}-\MatrixForm{B}\MatrixForm{D}^{-1}\MatrixForm{C})^{-1}\MatrixForm{B}\MatrixForm{D}^{-1}
\end{array}\right)$.
\end{theorem}

\begin{proof}
A proof can be found in \cite{golub2013matrix}.
\end{proof}

\subsection{PROOF OF VARIANCE EQUIVALENCE BETWEEN MLE AND CVE}
\label{appendix_proof_equivalence}

We now give a detailed proof of the variance equivalence.

\meanvarthm*

\begin{proof} \label{proof_meanvarthm}
Let $\mathcal{S}$ denote the support of $\xset$, and write
$y_i = y_i(\xset)$ below for ease of notation. Since
\begin{align*}
p(\xset~|~\vec{\eta}) = \frac{\exp\left\{ n\vec{\eta}^T\vec{y}(\xset) \right\}}{\exp\left\{ n\psi(\vec{\eta})\right\}} g(\xset),
\end{align*}
the normalizing condition gives
\begin{align*}
\int_{\mathcal{S}}\frac{\exp\left\{ n\vec{\eta}^T\vec{y}(\xset)\right\}}{\exp\left\{ n\psi(\vec{\eta})\right\}} g(\xset)~\mathrm{d}\vec{x}_1\cdots\mathrm{d}\vec{x}_n=1.
\end{align*}
Equivalently,
\begin{align}
\label{diff_both_sides_mean}
\int_{\mathcal{S}} \exp\left\{ n\vec{\eta}^T\vec{y}(\xset)\right\} g(\xset) ~\mathrm{d}\vec{x}_1\cdots\mathrm{d}\vec{x}_n = \exp\left\{ n\psi(\vec{\eta})\right\}.
\end{align}

Assume that $\vec{\eta}$ is in the interior of the natural parameter space. Then the derivative can be taken under the integral as the dominated convergence properties hold for exponential families. A proof of this can be found in Section 2.3 of \cite{keener2010theoretical}. Differentiating Equation~\eqref{diff_both_sides_mean} with respect to $\eta_i$ gives
\begin{align}
\int_{\mathcal{S}} n y_i \exp\left\{ n\vec{\eta}^T\vec{y}(\xset) \right\} g(\xset) ~\mathrm{d}\vec{x}_1\cdots\mathrm{d}\vec{x}_n &= n\exp\left\{ n\psi(\vec{\eta}) \right\} \frac{\partial\psi(\vec{\eta})}{\partial\eta_i} \label{diff_next_for_var} \\
\Rightarrow~~ \frac{\partial\psi(\vec{\eta})}{\partial\eta_i} &= \int_{\mathcal{S}} y_i \frac{ \exp\left\{ n\vec{\eta}^T\vec{y}(\xset) \right\}}{ \exp\left\{ n\psi(\vec{\eta}) \right\}} g(\xset) ~\mathrm{d}\vec{x}_1\cdots\mathrm{d}\vec{x}_n \notag\\
& = \Expt{y_i}. \notag
\end{align}
This gives the mean of $y_i$. Differentiating Equation~\eqref{diff_next_for_var} again with respect to $\eta_i$ gives
\begin{align}
\int_{\mathcal{S}} n^2y_i^2 \exp\left\{ n\vec{\eta}^T\vec{y}(\xset) \right\} g(\xset)~\mathrm{d}\vec{x}_1\cdots\mathrm{d}\vec{x}_n &= n\left(\exp\left\{ n\psi(\vec{\eta}) \right\} \frac{\partial^2\psi(\vec{\eta})}{\partial\eta_i^2} + n\exp\left\{ n\psi(\vec{\eta}) \right\}\left(\frac{\partial\psi(\vec{\eta})}{\partial\eta_i}\right)^2\right) \notag\\
\Rightarrow~~ n\Expt{y_i^2} &= \frac{\partial^2\psi(\vec{\eta})}{\partial\eta_i^2} + n\left(\Expt{y_i}\right)^2 \notag\\
\Rightarrow~~ \frac{\partial^2\psi(\vec{\eta})}{\partial\eta_i^2} &= n\left( \Expt{y_i^2} - \Expt{y_i}^2\right) = n\Var{y_i}. \notag
\end{align}

Similarly, differentiating Equation~\eqref{diff_next_for_var} with respect to $\eta_j$, where $j\neq i$, gives
\begin{align}
\int_{\mathcal{S}} n^2y_iy_j \exp\left\{ n\vec{\eta}^T\vec{y}(\xset) \right\} g(\xset) ~\mathrm{d}\vec{x}_1\cdots\mathrm{d}\vec{x}_n &= n\left( \exp\left\{ n\psi(\vec{\eta}) \right\} \frac{\partial^2\psi(\vec{\eta})} {\partial\eta_i\partial\eta_j} + n\exp\left\{ n\psi(\vec{\eta}) \right\} \frac{\partial\psi(\vec{\eta})}{\partial\eta_i} \frac{\partial\psi(\vec{\eta})}{\partial\eta_j} \right) \notag \\
\Rightarrow~~ n\Expt{y_iy_j} &= \frac{\partial^2\psi(\vec{\eta})}{\partial\eta_i\partial\eta_j} + n\Expt{y_i}\Expt{y_j} \notag\\ 
\Rightarrow~~ \frac{\partial^2\psi(\vec{\eta})} {\partial\eta_i\partial\eta_j} &= n\left( \Expt{y_iy_j} - \Expt{y_i}\Expt{y_j}\right) = n\Cov{y_i}{y_j}, \notag
\end{align}
which completes the proof.
\end{proof}

\maintheoremMLECVE*

\label{startMLECVEproof}

\begin{proof}
Let $\LCY = \sum_{i=1}^t \alpha_i y_i$  for $\alpha_i \in \mathbb R$. Clearly, $\Expt{\LCY} = \Expt{\sum_{i=1}^t \alpha_i y_i} = \sum_{i=1}^t \alpha_i\nu_i$, and $\Var{\LCY} = \sum_{i=1}^t \alpha_i^2\Var{y_i} + 2\sum_{i,j}^t \alpha_i\alpha_j\Cov{y_i}{y_j} = \vec{\alpha}^T (\MatrixForm{V}_n)_{1:t,1:t} \vec{\alpha}$, where $\vec{\alpha}^T = (\alpha_1,\hdots,\alpha_t)$ and $\MatrixForm{V}_n$ comes from Theorem~\ref{mean_and_var_exponential_family_thm}.

Our proof strategy is to obtain expressions for the
{\bf variances} of the estimate of $\LCY$ under the MLE and CVE, and show that they are equal.

The MLE for $\thetaunknown$ is asymptotically distributed as $\mathcal N\left(\thetaunknown,i_{\thetaunknown}^{-1}\right)$, so the asymptotic variance of $\LCY$ is $\vec{\alpha}^T i_{\thetaunknown}^{-1}\vec{\alpha}$.

We now compute $i_{\thetaunknown}$ by substituting the score function from Equation~\eqref{curved_ll2}:
\begin{align}
i_{\thetaunknown} &= \Expt{ n^2 \left( \frac{\partial\restrictedETA}{\partial\thetaunknown} \right)^T \left(\vec{y} - \frac{\partial\psi(\vec{\eta})}{\partial\vec{\eta}} \right) \left(\vec{y} - \frac{\partial\psi(\vec{\eta})}{\partial\vec{\eta}} \right)^T \frac{\partial\restrictedETA}{\partial\thetaunknown}} \notag\\
 & = n^2 \left( \frac{\partial\restrictedETA}{\partial\thetaunknown} \right)^T \MatrixForm{V}_n \frac{\partial\restrictedETA}{\partial\thetaunknown}, \label{second_line_FI_proof} 
\end{align}
where Equation~\eqref{second_line_FI_proof} follows from
Theorem~\ref{mean_and_var_exponential_family_thm}. Using the differential relationship in Theorem~\ref{diff_rs_thm} and the fact that $\mu_i=\nu_i$, we have
\begin{align*}
\MatrixForm{V}_n \frac{\partial\restrictedETA}{\partial\thetaunknown} &= \frac{1}{n} \left(
\begin{array}{c c c c}
\frac{\partial\mu_1}{\partial\nu_1} & \hdots & \frac{\partial\mu_1}{\partial\nu_t} \\
\vdots & \ddots & \vdots \\
\frac{\partial\mu_p}{\partial\nu_1} & \hdots & \frac{\partial\mu_p}{\partial\nu_t}
\end{array}\right) = \frac{1}{n}\left( \begin{array}{c}
\MatrixForm{I}_{t\times t} \\ \hline
\mathbf{0}_{(p-t)\times t}
\end{array}\right),
\end{align*}
where $\MatrixForm{I}_{t\times t}$ is the identity matrix. Therefore,
\begin{align}
i_{\thetaunknown} &= n \left( \frac{\partial\restrictedETA}{\partial\thetaunknown} \right)^T \left(
\begin{array}{c}
\MatrixForm{I}_{t} \\ \hline
\mathbf{0}_{(p-t)\times t}
\end{array} \right) = n \left(
\begin{array}{c c c}
\frac{\partial\eta_1}{\partial\mu_1} & \hdots & \frac{\partial\eta_1}{\partial\mu_t} \\
\vdots & \ddots & \vdots \\
\frac{\partial\eta_t}{\partial\mu_1} & \hdots & \frac{\partial\eta_t}{\partial\mu_t}
\end{array} \right).
\label{fisher_information_in_proof_of_nu}
\end{align}

By Equation~\eqref{fisher_information_in_proof_of_nu}, the asymptotic {\bf variance} given by the MLE is
\begin{align}
\MatrixForm{V}^{\mathrm{MLE}} = \vec{\alpha}^T i_{\thetaunknown}^{-1}\vec{\alpha} = \vec{\alpha}^T \left[ \frac{1}{n} \left(
\begin{array}{c c c}
\frac{\partial\eta_1}{\partial\mu_1} & \hdots & \frac{\partial\eta_1}{\partial\mu_t} \\
\vdots & \ddots & \vdots \\
\frac{\partial\eta_t}{\partial\mu_1} & \hdots & \frac{\partial\eta_t}{\partial\mu_t} 
\end{array}\right)^{-1} \right] \vec{\alpha}.
\label{variance_MLE_all}
\end{align}

We now consider the CVE
\begin{align*}
\LCY + \sum_{j=t+1}^p \hat{c}_j(y_j-\nu_j),
\end{align*}
where $\hat{c}_j$ are the optimal CV corrections. Partition $\MatrixForm{V}$, with entries
$\MatrixForm{v}_{ij} = \frac{\partial\mu_i}{\partial\eta_j}$, as
\begin{align}
\MatrixForm{V} = \left(
\begin{array}{c | c}
\MatrixForm{A}_{t\times t} &
\MatrixForm{B}_{t\times(p-t)} \\ \hline
\MatrixForm{B}^T_{(p-t)\times t} & \MatrixForm{D}_{(p-t)\times(p-t)}
\end{array} \right).
\label{case3_partition_V_big}
\end{align}

Because $\MatrixForm{V}$ is positive definite by assumption, its principal submatrix $\MatrixForm{D} =
\MatrixForm{V}_{(t+1):p,(t+1):p}$ is also positive definite and hence invertible.

The covariance matrix of the control variates $
(y_{t+1},\hdots,y_p)^T$ is $ \frac{1}{n}\MatrixForm{D}$. Moreover,
\begin{align*}
\Cov{(y_{t+1},\hdots,y_p)^T}{\LCY} = \frac{1}{n} \MatrixForm{B}^T\vec{\alpha}  = \frac{1}{n}\left(\sum_{i=1}^t \alpha_i \frac{\partial\mu_i}{\partial\eta_{t+1}}, \hdots, \sum_{i=1}^t \alpha_i \frac{\partial\mu_i}{\partial\eta_p} \right)^T.
\end{align*}

Therefore, the optimal CV coefficients satisfy
\begin{align}
\frac{1}{n}\MatrixForm{D}\vec{c} = - \frac{1}{n} \MatrixForm{B}^T\vec{\alpha}  = - \frac{1}{n} \left( \sum_{i=1}^t \alpha_i \frac{\partial\mu_i}{\partial\eta_{t+1}}, \hdots, \sum_{i=1}^t \alpha_i \frac{\partial\mu_i}{\partial\eta_p} \right)^T,
\label{to_use_in_proof}
\end{align}
or equivalently, $\vec{c} = - \MatrixForm{D}^{-1} \MatrixForm{B}^T \vec{\alpha}$. Substituting this expression of $\vec{c}$ in  Equation~\eqref{optimal_cv_var_reduct}, the {\bf variance}
of the CVE is
\begin{align}
\MatrixForm{V}^{\mathrm{CVE}} = \vec{\alpha}^T (\MatrixForm{V}_n)_{1:t,1:t} \vec{\alpha} - \frac{1}{n} \vec{\alpha}^T \MatrixForm{B} \MatrixForm{D}^{-1} \MatrixForm{B}^T \vec{\alpha} = \vec{\alpha}^T \left( \frac{1}{n}\MatrixForm{A} - \frac{1}{n} \MatrixForm{B} \MatrixForm{D}^{-1} \MatrixForm{B}^T \right) \vec{\alpha}. \label{variance_CV_all}
\end{align}

To show that the asymptotic {\bf variance} given by the MLE in Equation~\eqref{variance_MLE_all} equals the {\bf variance} of the CVE in Equation~\eqref{variance_CV_all}, it remains to show that
\begin{align*}
\left(\begin{array}{c c c}
\frac{\partial\eta_1}{\partial\mu_1} & \hdots & \frac{\partial\eta_1}{\partial\mu_t} \\
\vdots & \ddots & \vdots \\
\frac{\partial\eta_t}{\partial\mu_1} & \hdots & \frac{\partial\eta_t}{\partial\mu_t} \end{array} \right)^{-1} = \MatrixForm{A} - \MatrixForm{B} \MatrixForm{D}^{-1} \MatrixForm{B}^T.
\end{align*}

We partition $\MatrixForm{V}^{-1}$ as
\begin{align*}
\MatrixForm{V}^{-1} = \left( \begin{array}{c | c}
\widetilde{\MatrixForm{A}}_{t\times t} & \widetilde{\MatrixForm{B}}_{t\times(p-t)} \\ \hline
\widetilde{\MatrixForm{B}}^T_{(p-t)\times t} & \widetilde{\MatrixForm{D}}_{(p-t)\times(p-t)}
\end{array} \right),
\end{align*}
where $(\MatrixForm{V}^{-1})_{ij} = \frac{\partial\eta_i}{\partial\mu_j}$ by Lemma~\ref{lemma_invertible_jacobian}. By the Schur Complement (Theorem~\ref{theorem_schur}), the upper-left block of $\MatrixForm{V}^{-1}$ is $ \widetilde{\MatrixForm{A}} = \left( \MatrixForm{A} - \MatrixForm{B} \MatrixForm{D}^{-1} \MatrixForm{B}^T \right)^{-1}$.

Since $\MatrixForm{V}$ and $\MatrixForm{D}$ are positive definite, the Schur complement $\MatrixForm{A} - \MatrixForm{B} \MatrixForm{D}^{-1} \MatrixForm{B}^T$ is positive definite and therefore invertible. Hence,
\begin{align*}
\left(\begin{array}{c c c}
\frac{\partial\eta_1}{\partial\mu_1} & \hdots & \frac{\partial\eta_1}{\partial\mu_t} \\
\vdots & \ddots & \vdots \\
\frac{\partial\eta_t}{\partial\mu_1} & \hdots & \frac{\partial\eta_t}{\partial\mu_t}
\end{array} \right)^{-1} =  \widetilde{\MatrixForm{A}}^{-1} = \MatrixForm{A} - \MatrixForm{B} \MatrixForm{D}^{-1} \MatrixForm{B}^T.
\end{align*}

Substituting this identity into Equation~\eqref{variance_MLE_all} gives $\MatrixForm{V}^{\mathrm{MLE}} = \MatrixForm{V}^{\mathrm{CVE}}$, completing the proof.
\end{proof}

\lincombcorollary*

\begin{proof} \label{proof_lincombcor}
Let 
\begin{align}
p(\xset~|~\vec{\nu})  \equiv  \frac{\exp\left\{n \vec{\eta}^T\vec{y}\right\}}{\exp\left\{n\psi(\vec{\eta}) \right\}} g\left(\xset\right) \label{generic_exp_family}
\end{align}
be the probability density expressed in exponential family form. 
Without loss of generality, let $\thetaunknown = \{\nu_i\}_{i=1}^t$, and $\thetaknown = \{\nu_{j}\}_{j=t+1}^p$, and let $\vec{\mu} = \MatrixForm{A}\vec{\nu}$ where $\MatrixForm{A}$ is invertible. 

$\MatrixForm{A}$ being invertible implies $\nu_i$ can be written as a linear combination of $\mu_i$, with coefficients given by $(\MatrixForm{A}^{-1}\vec{\mu})_i$. But since $\mu_i = \Expt{y_i}$, then 
\begin{align*}
\nu_i = \sum_{j=1}^p (\MatrixForm{A}^{-1})_{ij}\mu_j = \sum_{j=1}^p (\MatrixForm{A}^{-1})_{ij}\Expt{y_j} = \Expt{\sum_{j=1}^p (\MatrixForm{A}^{-1})_{ij}y_j}.
\end{align*}
Given that the sufficient statistics $\{y_i\}_{i=1}^p$ are linearly independent, then $\{(\MatrixForm{A}^{-1}\vec{y})_i\}_{i=1}^p$ have to be linearly independent as well. 

Now let $\widetilde{y} = \MatrixForm{A}^{-1}\vec{y}$, $\widetilde{\eta} =  \MatrixForm{A}^{T}\vec{\eta}$ and $\widetilde{\psi}(\widetilde{\eta}) = \psi\left(\MatrixForm{A}^{-T}\widetilde{\eta}\right)$. Since $\vec{\mu} = \MatrixForm{A}\vec{\nu}$,
\begin{align*}
\Expt{\widetilde{\vec{y}}} = \MatrixForm{A}^{-1}\Expt{\vec{y}} = \MatrixForm{A}^{-1}\vec{\mu} = \vec{\nu}.
\end{align*}
Furthermore, since $\MatrixForm{A}$ is invertible and $\{y_i\}_{i=1}^p$ are linearly independent, the components of $\widetilde{\vec{y}}$ are linearly independent. Therefore, Equation~\eqref{generic_exp_family} can be re-expressed as
\begin{align}
p(\xset~|~\vec{\nu}) = \frac{\exp\left\{n\vec{\eta}^T\vec{y}\right\}}{\exp\left\{n\psi(\vec{\eta})\right\}} g\left(\xset\right) = \frac{ \exp\left\{ n\left(\MatrixForm{A}^{T}\vec{\eta}\right)^T \left(\MatrixForm{A}^{-1}\vec{y}\right)\right\}}{\exp\left\{n\psi(\vec{\eta})\right\}} g\left(\xset\right) =
\frac{\exp\left\{n\widetilde{\eta}^T\widetilde{\vec{y}}\right\}}{\exp\left\{
n\widetilde{\psi}\left(\widetilde{\eta}\right)\right\}} g\left(\xset\right).
\label{linear_combi_eqn_corollary_proof}
\end{align}
The control variates are $\widetilde{y}_j$, $(t+1) \leq j \leq p$, and therefore the corresponding CVE is
\begin{align*}
\Expt{\left(\sum_{i=1}^t \alpha_i\widetilde{y}_i\right) +\sum_{j=t+1}^p c_j\left( \widetilde{y}_j-\nu_j \right)}.
\end{align*}

Applying Theorem~\ref{equal_variances_theorem_thingy} to the exponential family in Equation~\eqref{linear_combi_eqn_corollary_proof} completes the proof.
\end{proof}

\subsection{PROOF BEHIND ALGORITHM 1}
\label{ALGODISCUSSION}

Here, we show that the optimal CVE gives a fixed point iteration that converges to the MLE, and give the conditions for convergence.

In the proofs that follow, the notation $\MatrixForm{A}$, $\MatrixForm{B}$, and
$\MatrixForm{D}$ denote the blocks of $\MatrixForm{V}$ in
Equation~\eqref{case3_partition_V_big}, evaluated at
$\vec{\eta}(\thetaunknown)$.

We first give a lemma that gives the optimal CV coefficient. 
\begin{lemma}
\label{lemma_optimal_vector_cv_coefficients}

Under the assumptions of Theorem~\ref{equal_variances_theorem_thingy}, consider estimators of
$\thetaunknown$ of the form $\vec{y}_{1:t} + \MatrixForm{C} \left(\vec{y}_{(t+1):p}-\thetaknown
\right)$,
where $\MatrixForm{C}\in\mathbb{R}^{t\times(p-t)}$. Let $\MatrixForm{A}$,
$\MatrixForm{B}$, and $\MatrixForm{D}$ be given by the partition in Equation~\eqref{case3_partition_V_big}. The unique choice of $\MatrixForm{C}$ which minimizes $
\Cov{\vec{y}_{1:t}+\MatrixForm{C}\left(\vec{y}_{(t+1):p} - \thetaknown\right)}{\vec{y}_{1:t}+\MatrixForm{C}\left(\vec{y}_{(t+1):p}-\thetaknown\right)}$ in the positive semidefinite ordering is $
\MatrixForm{C} = -\MatrixForm{B}\MatrixForm{D}^{-1}. 
$
Hence, for $1\leq i\leq t$ and $t+1\leq j\leq p$, the optimal CV coefficient in the $i^{\mathrm{th}}$ fixed point equation is $
\hat{c}_{ij} = -\left(\MatrixForm{B}\MatrixForm{D}^{-1}\right)_{i,j-t}$. 
\end{lemma}

\begin{proof}
Since $\thetaknown$ is fixed and not random,
\begin{align*}
\Cov{\vec{y}_{1:t} + \MatrixForm{C}\left(\vec{y}_{(t+1):p}-\thetaknown\right)}{\vec{y}_{1:t}+\MatrixForm{C}\left(\vec{y}_{(t+1):p}-\thetaknown\right)} = \Cov{\vec{y}_{1:t}+\MatrixForm{C}\vec{y}_{(t+1):p}}{\vec{y}_{1:t}+\MatrixForm{C}\vec{y}_{(t+1):p}}.
\end{align*}

By the partition in Equation~\eqref{case3_partition_V_big},
\begin{align*}
\Cov{\vec{y}_{1:t}}{\vec{y}_{1:t}} =\frac{1}{n}\MatrixForm{A},~~
\Cov{\vec{y}_{1:t}}{\vec{y}_{(t+1):p}} = \frac{1}{n}\MatrixForm{B},~~
\Cov{\vec{y}_{(t+1):p}}{\vec{y}_{(t+1):p}}= \frac{1}{n}\MatrixForm{D}.
\end{align*}
Therefore,
\begin{align*}
\Cov{\vec{y}_{1:t}+\MatrixForm{C}\vec{y}_{(t+1):p}}{\vec{y}_{1:t}+\MatrixForm{C}\vec{y}_{(t+1):p}}   = \frac{1}{n}\left(\MatrixForm{A}+\MatrixForm{C}\MatrixForm{B}^{T}+\MatrixForm{B}\MatrixForm{C}^{T}+\MatrixForm{C}\MatrixForm{D}\MatrixForm{C}^{T}\right).
\end{align*}

Completing the matrix square gives
\begin{align*}
\MatrixForm{A} + \MatrixForm{C}\MatrixForm{B}^{T} + \MatrixForm{B}\MatrixForm{C}^{T} + \MatrixForm{C}\MatrixForm{D}\MatrixForm{C}^{T}  =\MatrixForm{A}- \MatrixForm{B}\MatrixForm{D}^{-1}\MatrixForm{B}^{T} +\left(\MatrixForm{C}+\MatrixForm{B}\MatrixForm{D}^{-1}\right)\MatrixForm{D}\left(\MatrixForm{C}+\MatrixForm{B}\MatrixForm{D}^{-1}\right)^{T}.
\end{align*}

Hence,
\begin{align}
& \phantom{=} \Cov{\vec{y}_{1:t}+\MatrixForm{C}\left(\vec{y}_{(t+1):p}-\thetaknown\right)}{\vec{y}_{1:t}+\MatrixForm{C}\left(\vec{y}_{(t+1):p}-\thetaknown\right)}\notag  \\
& \hspace*{3cm}=\frac{1}{n}\left(\MatrixForm{A}-\MatrixForm{B}\MatrixForm{D}^{-1}\MatrixForm{B}^{T}\right) +\frac{1}{n}\left(\MatrixForm{C}+\MatrixForm{B}\MatrixForm{D}^{-1}\right)\MatrixForm{D}\left(\MatrixForm{C}+\MatrixForm{B}\MatrixForm{D}^{-1}\right)^{T}. \label{covariance_complete_square_cv_matrix}
\end{align}

Since $\MatrixForm{D}$ is positive definite, the second term in Equation~\eqref{covariance_complete_square_cv_matrix} is positive semidefinite and equals the zero matrix if and only if $\MatrixForm{C}+\MatrixForm{B}\MatrixForm{D}^{-1}=\mathbf{0}$. Therefore, the unique choice of $\MatrixForm{C}$ which minimizes the covariance matrix in the positive semidefinite ordering is $\MatrixForm{C}=-\MatrixForm{B}\MatrixForm{D}^{-1}$. For $1\leq i\leq t$ and $t+1\leq j\leq p$, the coefficient multiplying $y_j-\mu_j$ in the $i^{\text{th}}$ CV expression is the $(i,j-t)^{\text{th}}$ entry of $\MatrixForm{C}$. Thus, $\hat{c}_{ij}=-\left(\MatrixForm{B}\MatrixForm{D}^{-1}\right)_{i,j-t}$ and we are done.

\end{proof}

While Theorem~\ref{equal_variances_theorem_thingy} shows that the MLE and optimal CVE have the same variance reduction, Theorem~\ref{theorem_cv_fixed_point_equations} below shows that the optimal CV corrections themselves turn the likelihood score equations into fixed point equations, so solving the CV system is equivalent to finding a likelihood stationary point.

\begin{theorem}
\label{theorem_cv_fixed_point_equations}
Under the assumptions of Theorem~\ref{equal_variances_theorem_thingy}, let $\MatrixForm{S} = \MatrixForm{A} - \MatrixForm{B}\MatrixForm{D}^{-1}\MatrixForm{B}^{T}$. Then
\begin{align*}
\frac{\partial l(X~|~\vec{\eta}(\thetaunknown))}{\partial\thetaunknown} = n\MatrixForm{S}^{-1} \left[ \vec{y}_{1:t} - \thetaunknown - \MatrixForm{B}\MatrixForm{D}^{-1} \left(\vec{y}_{(t+1):p} - \thetaknown\right)\right].
\end{align*}
Consequently, $\frac{\partial l(X~|~\vec{\eta}(\thetaunknown))}{\partial\thetaunknown} = \vec{0}$
if and only if $
\thetaunknown=\vec{y}_{1:t}-\MatrixForm{B}\MatrixForm{D}^{-1}\left(\vec{y}_{(t+1):p}-\thetaknown\right)$.
Equivalently, for every $1\leq i\leq t$, $\nu_i = y_i + \sum_{j=t+1}^{p}\hat{c}_{ij}(\nu_1,\hdots,\nu_t)(y_j-\mu_j)$
where $
\hat{c}_{ij}(\nu_1,\hdots,\nu_t)=-\left(\MatrixForm{B}\MatrixForm{D}^{-1}\right)_{i,j-t}$.
\end{theorem}

\begin{proof}
Since $\mu_i = \nu_i$ and $\thetaknown$ is fixed, $\frac{\partial \vec{\mu}}{\partial\thetaunknown} = \left(
\begin{array}{c}
\MatrixForm{I}_{t\times t}\\ \hline
\mathbf{0}_{(p-t)\times t}
\end{array}\right)$.
By Theorem~\ref{diff_rs_thm}, $\MatrixForm{V}\frac{\partial\vec{\eta}(\thetaunknown)}{\partial\thetaunknown}=\frac{\partial\vec{\mu}}{\partial\thetaunknown}$. Therefore,$\frac{\partial\vec{\eta}(\thetaunknown)}{\partial\thetaunknown}=\MatrixForm{V}^{-1}\left(
\begin{array}{c}
\MatrixForm{I}_{t\times t}\\ \hline
\mathbf{0}_{(p-t)\times t}
\end{array}\right)$. Let $\MatrixForm{S} =\MatrixForm{A} - \MatrixForm{B}\MatrixForm{D}^{-1}\MatrixForm{B}^{T}$. By Theorem~\ref{theorem_schur},
\begin{align*}
\MatrixForm{V}^{-1} = \left(\begin{array}{c|c}
\MatrixForm{S}^{-1} & -\MatrixForm{S}^{-1}\MatrixForm{B}\MatrixForm{D}^{-1} \\ \hline
-\MatrixForm{D}^{-1}\MatrixForm{B}^{T}\MatrixForm{S}^{-1} & \MatrixForm{D}^{-1} + \MatrixForm{D}^{-1}\MatrixForm{B}^{T}\MatrixForm{S}^{-1} \MatrixForm{B}\MatrixForm{D}^{-1}
\end{array} \right).
\end{align*}

Thus,
\begin{align*}
\frac{\partial\vec{\eta}(\thetaunknown)}{\partial\thetaunknown}=\left(\begin{array}{c}
\MatrixForm{S}^{-1}\\ \hline
-\MatrixForm{D}^{-1}\MatrixForm{B}^{T}\MatrixForm{S}^{-1}
\end{array}\right).
\end{align*}

Using Equation~\eqref{curved_ll2},
\begin{align}
\frac{\partial l(X\mid\vec{\eta}(\thetaunknown))}{\partial\thetaunknown} &= n\left(\frac{\partial\vec{\eta}(\thetaunknown)}{\partial\thetaunknown}\right)^T\left(\vec y-\vec\mu\right) \notag \\
& = n \left(\begin{array}{c}
\MatrixForm{S}^{-1}\\ \hline
-\MatrixForm{D}^{-1}\MatrixForm{B}^{T}\MatrixForm{S}^{-1}
\end{array}\right)^T\left(\begin{array}{c}
\vec y_{1:t}-\thetaunknown\\ \hline
\vec y_{(t+1):p}-\thetaknown
\end{array}\right) \notag \\
& = n\MatrixForm{S}^{-1}\left(\vec y_{1:t}-\thetaunknown\right)- n\MatrixForm{S}^{-1}\MatrixForm{B}\MatrixForm{D}^{-1}\left( \vec y_{(t+1):p}-\thetaknown\right)\notag \\
& = n\MatrixForm{S}^{-1} \left[ \vec y_{1:t} - \thetaunknown - \MatrixForm{B}\MatrixForm{D}^{-1} \left( \vec y_{(t+1):p}-\thetaknown \right) \right] \label{S_invertible??}.
\end{align}

Since $\MatrixForm{S}$ is invertible, Equation~\eqref{S_invertible??} implies that $\frac{\partial l(X~|~\vec{\eta}(\thetaunknown))}{\partial\thetaunknown} = \vec{0}$ if and only if $
\vec y_{1:t} - \thetaunknown - \MatrixForm{B}\MatrixForm{D}^{-1} \left( \vec y_{(t+1):p}-\thetaknown \right) = \mathbf{0}$. 

Making $\thetaunknown$ the subject gives
$\thetaunknown = \vec y_{1:t} - \MatrixForm{B}\MatrixForm{D}^{-1} \left( \vec y_{(t+1):p}-\thetaknown \right)$. By Lemma~\ref{lemma_optimal_vector_cv_coefficients}, $
\hat{c}_{ij} = - \left(\MatrixForm{B}\MatrixForm{D}^{-1} \right)_{i,j-t}$. Finally, taking the $i^{\text{th}}$ coordinate of the preceding vector equation gives
\begin{align*}
\nu_i = y_i + \sum_{j=t+1}^{p} \hat{c}_{ij}(\nu_1,\hdots,\nu_t)(y_j-\mu_j).
\end{align*}
\end{proof}

Corollary~\ref{corollary_fixed_points_stationary_points} now states that solving the fixed point equations yields a likelihood stationary point. If the likelihood has a unique stationary point that is the MLE, this stationary point is the MLE.

\begin{corollary}
\label{corollary_fixed_points_stationary_points}
Suppose that $\thetaunknown^{\star}$ satisfies
\begin{align}
\nu_i^{\star} = y_i + \sum_{j=t+1}^{p} \hat{c}_{ij}(\nu_1^{\star},\hdots,\nu_t^{\star})(y_j-\mu_j),~~1 \leq i \leq t.
\label{coordinate_fixed_point_definition}
\end{align}
Then 
\begin{align*}
\left. \frac{\partial l(X~|~\vec{\eta}(\thetaunknown))}{\partial\thetaunknown}\right|_{\thetaunknown=\thetaunknown^{\star}} =\vec{0}.
\end{align*}
Conversely, every interior solution of
\begin{align*}
\frac{\partial l(X\mid\vec{\eta}(\thetaunknown))}{\partial\thetaunknown}=\vec{0}
\end{align*}
satisfies Equation~\eqref{coordinate_fixed_point_definition}. Therefore, the fixed points of the CV expressions are exactly the interior likelihood stationary points. If Algorithm~\ref{FP_algo_format} converges to a fixed point and the likelihood has a unique stationary point that is the MLE, then its limit is the MLE of $\thetaunknown$.
\end{corollary}

\begin{proof}
Suppose that $\thetaunknown^{\star}$ satisfies
Equation~\eqref{coordinate_fixed_point_definition}. By the definition of $\hat{c}_{ij}$ in Theorem~\ref{theorem_cv_fixed_point_equations}, stacking
the $t$ coordinate equations gives
\begin{align*}
\thetaunknown^{\star} = \vec{y}_{1:t} - \MatrixForm{B} \MatrixForm{D}^{-1}\left(\vec{y}_{(t+1):p} - \thetaknown \right),
\end{align*}
where $\MatrixForm{B}$ and $\MatrixForm{D}$ are evaluated at $\vec{\eta}(\thetaunknown^{\star})$. By Theorem~\ref{theorem_cv_fixed_point_equations},
\begin{align*}
\left. \frac{\partial l(X~|~\vec{\eta}(\thetaunknown))}{\partial\thetaunknown}\right|_{\thetaunknown=\thetaunknown^{\star}} = \vec{0}.
\end{align*}

Thus, every fixed point of the CV expressions is a likelihood stationary point. Conversely, suppose that $\thetaunknown^{\star}$ is an interior solution of
\begin{align*}
\left. \frac{\partial l(X~|~\vec{\eta}(\thetaunknown))}
{\partial\thetaunknown}\right|_{\thetaunknown=\thetaunknown^{\star}} = \vec{0}.
\end{align*}
By Theorem~\ref{theorem_cv_fixed_point_equations}, $
\thetaunknown^{\star} = \vec{y}_{1:t} - \MatrixForm{B} \MatrixForm{D}^{-1} \left( \vec{y}_{(t+1):p} - \thetaknown \right)$. Taking the $i^{\text{th}}$ coordinate and using $ \hat{c}_{ij} = - \left(\MatrixForm{B}\MatrixForm{D}^{-1} \right)_{i,j-t} $ gives
\begin{align*}
\nu_i^{\star} = y_i + \sum_{j=t+1}^{p} \hat{c}_{ij}(\nu_1^{\star},\hdots,\nu_t^{\star})(y_j-\mu_j), ~~~ 1\leq i\leq t.
\end{align*}
Therefore, every interior likelihood stationary point is a fixed point of the CV expressions.

If Algorithm~\ref{FP_algo_format} converges to a fixed point and the likelihood has a unique stationary point that is the MLE, then its limit is the MLE of $\thetaunknown$.
\end{proof}

Finally, Theorem~\ref{theorem_local_convergence_cv_fixed_point} gives a condition for Algorithm~\ref{FP_algo_format} to converge.

\begin{theorem}
\label{theorem_local_convergence_cv_fixed_point}
Suppose that $\thetaunknown^{\star}$ satisfies Equation~\eqref{coordinate_fixed_point_definition}. For $1\leq i\leq t$, define
\begin{align}
g_i(\nu_1,\hdots,\nu_t) = y_i + \sum_{j=t+1}^{p}\hat{c}_{ij}(\nu_1,\hdots,\nu_t)(y_j-\mu_j).
\label{definition_coordinate_update}
\end{align}
Suppose that $g_1,\hdots,g_t$ are continuously differentiable in a neighborhood of $\thetaunknown^{\star}$. Let $\MatrixForm{G}$ denote one complete sequential update of
Algorithm~\ref{FP_algo_format}. That is, for
$\thetaunknown=(\nu_1,\hdots,\nu_t)^T$,
\begin{align}
\MatrixForm{G}(\thetaunknown) = \left(
\begin{array}{c}
g_1(\nu_1,\hdots,\nu_t) \\
g_2(g_1(\nu_1,\hdots,\nu_t),\nu_2,\hdots,\nu_t) \\
\vdots \\
g_t(\nu_1^{+},\hdots,\nu_{t-1}^{+},\nu_t)
\end{array}\right),
\label{definition_iteration_map}
\end{align}
where $\nu_1^{+},\hdots,\nu_{t-1}^{+}$ denote the updated values obtained earlier in the same sequential iteration. Let
\begin{align}
\MatrixForm{J}_{\MatrixForm{G}}(\thetaunknown^{\star}) =
\left(\frac{\partial G_i}{\partial \nu_j}(\thetaunknown^{\star})\right)_{1\leq i,j\leq t}
\label{definition_jacobian_map}
\end{align}
denote the Jacobian matrix of $\MatrixForm{G}$ evaluated at $\thetaunknown^{\star}$. If
\begin{align}
\rho\left(\MatrixForm{J}_{\MatrixForm{G}}(\thetaunknown^{\star})\right) < 1\label{local_cv_fp_convergence_condition}
\end{align}
where $\rho(\cdot)$ denotes the spectral radius, then there exists a neighborhood $\mathcal{N}$ of $\thetaunknown^{\star}$ such that every initial estimate $\widehat{\thetaunknown}^{(1)}\in\mathcal{N}$ produces iterates satisfying
\begin{align}
\widehat{\thetaunknown}^{(r)} \rightarrow \thetaunknown^{\star}~~~\text{as }r\rightarrow\infty.
\end{align}
\end{theorem}

\begin{proof}
Since $\thetaunknown^{\star}$ satisfies Equation~\eqref{coordinate_fixed_point_definition}, $ \nu_i^{\star} = g_i(\nu_1^{\star},\hdots,\nu_t^{\star}), 1\leq i\leq t$. Therefore, one complete sequential update beginning at $\thetaunknown^{\star}$ leaves every coordinate unchanged. Hence, $\MatrixForm{G}(\thetaunknown^{\star}) = \thetaunknown^{\star}$. 

Now suppose that $\rho\left(\MatrixForm{J}_{\MatrixForm{G}}(\thetaunknown^{\star})\right)< 1$. Choose $q$ such that $\rho\left(\MatrixForm{J}_{\MatrixForm{G}}(\thetaunknown^{\star})
\right) < q < 1$. Then there exists a vector norm $\|\cdot\|_{\star}$ such that $\left\| \MatrixForm{J}_{\MatrixForm{G}}(\thetaunknown^{\star}) \right\|_{\star} < q < 1$.

Since $g_1,\hdots,g_t$ are continuously differentiable in a neighborhood of $\thetaunknown^{\star}$, $\MatrixForm{G}$ is continuously differentiable in a neighborhood of $\thetaunknown^{\star}$. Therefore, there exists $r>0$ such that $\mathcal{N} = \left\{ \thetaunknown: \left\| \thetaunknown-\thetaunknown^{\star} \right\|_{\star} \leq r \right\}$ is contained in that neighborhood and $\sup_{\thetaunknown\in\mathcal{N}} \left\| \MatrixForm{J}_{\MatrixForm{G}}(\thetaunknown)\right\|_{\star} \leq q < 1$.

For any $\thetaunknown,\widetilde{\thetaunknown}\in\mathcal{N}$, the multivariate mean value inequality \cite{ortega1970iterative} gives
\begin{align*}
\left\|\MatrixForm{G}(\thetaunknown) - \MatrixForm{G}(\widetilde{\thetaunknown}) \right\|_{\star} \leq q \left\| \thetaunknown-\widetilde{\thetaunknown} \right\|_{\star}.
\end{align*}
Hence, for every $\thetaunknown\in\mathcal{N}$,
\begin{align*}
\left\|\MatrixForm{G}(\thetaunknown) - \thetaunknown^{\star} \right\|_{\star} = \left\| \MatrixForm{G}(\thetaunknown) - \MatrixForm{G}(\thetaunknown^{\star}) \right\|_{\star} \leq q \left\| \thetaunknown \thetaunknown^{\star} \right\|_{\star} \leq qr < r.
\end{align*}
Therefore, $\MatrixForm{G}(\mathcal{N}) \subseteq \mathcal{N}$, and $\MatrixForm{G}$ is a contraction from $\mathcal{N}$ to itself. Since $\MatrixForm{G}(\thetaunknown^{\star}) = \thetaunknown^{\star}$, the contraction mapping theorem \cite{ortega1970iterative} implies that every initial estimate $\widehat{\thetaunknown}^{(1)} \in \mathcal{N}$ produces iterates satisfying $
\widehat{\thetaunknown}^{(r)}
\rightarrow
\thetaunknown^{\star}$ as $r\rightarrow\infty$.

\end{proof}

\subsection{DISCUSSION OF CONVERGENCE AND VARIANCE OF SMALL $K$}
\label{CONVERGENCEDISCUSSION}

We give plots (Figure~\ref{boxplot_convergence_FH} and Figure~\ref{boxplot_convergence_RP}) estimating the rate of convergence in Appendix~\ref{secA4} for feature hashing and random projections, and getting a bound on the rate of convergence would be a future open question for different types of exponential families.

In the case of the bivariate Normal (feature hashing and random projections), the form of the MLE is a cubic, and hence there is a possibility of the MLE (or Algorithm~\ref{FP_algo_format}) converging to the wrong root. However, Lemma 4 in \citep{li2006improving} gives an upper bound the probability of the cubic having multiple real roots (which decreases exponentially fast as $k$ increases). Our experiments in Appendix~\ref{secA4} (in particular Table~\ref{outlier_FH_table} and Table~\ref{outlier_RP_table}) show that even with $k = 10$, the proportion of cubics with three real roots is negligible, and with $k = 20$, there are no cubics with three real roots.

On the other hand, if applying Algorithm~\ref{FP_algo_format} gives a closed form solution such as an unbiased estimator for the trace as discussed in Section~\ref{sect_discussion} and Appendix~\ref{hutch_expt_appendix} for Hutchinson's Trace Estimator, then we can just do standard variance analysis on the derived estimator as shown in Appendix~\ref{secA5}.

\newpage
\section{SUPPLEMENT TO EXPERIMENTS}\label{secA4}

\subsection{LAYOUT OF OUR SUPPLEMENTARY MATERIAL AND CODE FILES}

We first describe the structure of code in our folder of supplementary material. The folder: {\tt plots in paper} gives the plots that we use in paper in eps and pdf format, for ease of comparison.

The other two folders, {\tt feature hashing} and {\tt random projection} contain the code to generate all the plots in our paper. 

There are three Matlab code files which we describe in detail below. The three files 
\begin{itemize}
\item {\tt newton\_raphson\_cubic\_vectorized.m}
\item {\tt secant\_cubic\_vectorized.m}
\item {\tt cv\_vectorized.m}
\end{itemize}
compute the estimates of $\langle \vec{x}_1,\vec{x}_2\rangle$ under the respective algorithms, with 10 iterations/update steps till convergence at $\epsilon < 0.0001$. We pick 10 iterations/update steps till convergence since most estimates took at most 7 iterations/update steps till convergence, with remaining estimates failing to converge (even with 100 iterations/update steps) ; i.e. $\|f_{n+1} - f_n\|$ is not (almost always) monotonically decreasing.

We use the same vectorized procedure across all three Matlab files ; where only the update step changes. Algorithm~\ref{FP_algo_format} via the CVE is given by
\begin{align*}
f_{n+1} & = \frac{\sum_{s=1}^k v_{1s}v_{2s}}{k} - \frac{ f_n \left( \|\vec{x}_2\|^2 \left( \frac{\sum_{s=1}^k v_{1s}^2}{k} - \|\vec{x}_1\|^2  \right) +\|\vec{x}_1\|^2 \left( \frac{\sum_{s=1}^k v_{2s}^2}{k} - \|\vec{x}_2\|^2  \right) \right)}{ f_n^2 + \|\vec{x}_1\|^2 \|\vec{x}_2\|^2}. \notag
\end{align*}
All three files do the following two checks after the 10 iterations/update steps to maintain consistency and to mimic what would have been done in an actual (non-experimental) situation.
\begin{enumerate}
\item If there was no convergence after 10 iterations/update steps, we set the estimate to be the initial estimate given by feature hashing / random projection.
\item If the converged estimate was greater than $\sqrt{\|\vec{x}_1\|^2 \|\vec{x}_2\|^2}$, we set it to be $\sqrt{\|\vec{x}_1\|^2 \|\vec{x}_2\|^2}$. Equivalently, if the converged estimate was lesser than $-\sqrt{\|\vec{x}_1\|^2 \|\vec{x}_2\|^2}$, we set it to be $-\sqrt{\|\vec{x}_1\|^2 \|\vec{x}_2\|^2}$. The reason for this comes from the cosine rule since
\begin{align*}
\cos(\theta) = \frac{\langle \vec{x}_1,\vec{x}_2\rangle}{\sqrt{\|\vec{x}_1\|^2 \|\vec{x}_2\|^2}}
\end{align*}
and as we would not know $\theta$ {\it a priori} (since we are estimating the inner product), 
\begin{align*}
\begin{array}{r c c c c c}
 & -1 &\leq &\cos(\theta) &\leq &1 \\
\Rightarrow & -\sqrt{\|\vec{x}_1\|^2 \|\vec{x}_2\|^2} &\leq& \langle\vec{x}_1,\vec{x}_2\rangle &\leq &\sqrt{\|\vec{x}_1\|^2 \|\vec{x}_2\|^2}.
\end{array}
\end{align*}

\end{enumerate}
The motivation for providing these three files is to show the difference (in implementation) for these three algorithms, since our experiments focus on comparing the performance of two common root-finding algorithms to recover the estimates of $\hat{\eta}$ under MLE, and {\tt CV-FP} given by Algorithm~\ref{FP_algo_format} in our paper.

{\it We stress that our paper demonstrates the performance of Algorithm~\ref{FP_algo_format}. Therefore, we are not aiming to get improved performance on benchmarked datasets, or outperform state-of-the-art algorithms. }

\subsection{FEATURE HASHING: FASTER CONVERGENCE AND BETTER STABILITY FOR THE MLE} \label{sec:fh_exp_appendix}

The feature hashing algorithm~\citep{weinberger2009feature} is a widely used method to quickly estimate inner products between high dimensional vectors $\vec{x}_i,\vec{x}_j$. It is a hashing-based dimension reduction algorithm that compresses high-dimensional data points into low-dimensional points so that the estimated pairwise inner product of the low-dimensional data points gives a close approximation of the corresponding pairwise inner product of the original high-dimensional data points.

Let $\MatrixForm{X}_{n \times p}$ be a data matrix, $h:[p] \mapsto [k]$ and $\phi:[p] \mapsto \{+1, -1\}$ be hash functions from a $2$-wise universal hash family. Suppose $\vec{x}_i$ is the $i$-th row of the data matrix $X$ and $\vec{v}_i$ be its $k$-dimensional sketch obtained using the feature hashing method. Then for any $s \in [k]$, 
\begin{align}
v_{is} = \sum_{t=1~s.t.~h(t)=s}^{p}\phi(t) \, x_{it}.  \label{eq:eqn_fh_append}
\end{align}
Let $\MatrixForm{R}_{n \times k}$ be a matrix defined as $\MatrixForm{R} = \MatrixForm{D}\MatrixForm{P}$, where $\MatrixForm{D}_{n \times n}$ is a diagonal matrix with i.i.d. $\pm1$ random entries and $\MatrixForm{P}_{n \times k}$ is a random matrix with exactly one non-zero entry at a random position per row, i.e. $\MatrixForm{P}_{i,:} \sim \mathrm{Uniform}\{\vec{e}_1, \ldots, \vec{e}_p\}$, where for $t \in [p]$, $\vec{e}_{t}$ represents a standard basis of the $p$-dimensional space. Then, the feature hashing operation on the data matrix $\MatrixForm{X}$ can also be visualized as $\MatrixForm{V}_{n \times k} = \MatrixForm{X}\MatrixForm{R}$. Suppose $\vec{v}_i, \vec{v}_{j} \in \R^k$ are the sketch vectors of the $i$-th and $j$-th row of the data matrix $\MatrixForm{X}$ (i.e. $\vec{x}_i$ and $\vec{x}_j$) obtained using feature hashing, respectively. Then $\langle \vec{v}_i, \vec{v}_j \rangle$ gives an unbiased estimate of $\langle \vec{x}_i, \vec{x}_j \rangle$.

\cite{verma2022variance} used Lyapunov's Central Limit Theorem to show that for any two rows $\vec{v}_{i}, \vec{v}_j$ in the matrix $\MatrixForm{V}$, $(\vec{v}_{i}, \vec{v}_{j}) = (v_{i1,}, \ldots, v_{ik},v_{j1},\ldots,v_{jk})$ asymptotically follows the multivariate Normal distribution with
\begin{align*}
    \left(\begin{array}{c}
        \vec{v}_{i}\\
        \vec{v}_j
    \end{array}\right) \sim \mathcal{N}\left(\left(\begin{array}{c}
        0\\
        \vdots\\
        0\\
        0\\
        \vdots\\
        0
    \end{array}\right),\frac{1}{k}\left(\begin{array}{cccccc}
    \|\vec{x}_i\|^2&\cdots& 0&\langle \vec{x}_i,\vec{x}_j\rangle&\cdots&0\\
    \vdots&\ddots&\vdots&\vdots&\ddots&\vdots\\
    0& \cdots &\|\vec{x}_i\|^2& 0 & \cdots &\langle \vec{x}_i,\vec{x}_j \rangle\\
    \langle \vec{x}_i,\vec{x}_j\rangle&\cdots&0& \|\vec{x}_j\|^2&\cdots& 0\\
    \vdots&\ddots&\vdots&\vdots&\ddots&\vdots\\
    0& \cdots &\langle \vec{x}_i,\vec{x}_j \rangle& 0 & \cdots &\|\vec{x}_j\|^2
    \end{array} \right)\right).
\end{align*}

\cite{verma2022variance} then used the asymptotic normality of $(\Vec{v}_{i}, \vec{v}_j)$ and the fact that if the original $\ell_2$ norm of the vectors $\vec{x}_1,\hdots,\vec{x}_n$ are known (as also considered in \cite{li2006improving} in case of random projection) to construct an MLE to estimate $\langle \vec{x}_i,\vec{x}_j\rangle$.  Their MLE estimate of the $\langle\vec{x}_i, \vec{x}_j \rangle$ corresponds to  finding the root of the following cubic equation:
 
\begin{multline}
    \lambda^3 - \left(\sum_{s=1}^k v_{is} v_{js} \right)\lambda^2 + \left( \|\vec{x}_i\|^2 \left(\sum_{s=1}^k v_{js}^2 \right) + \|\vec{x}_j\|^2 \left(\sum_{s=1}^k v_{is}^2 \right) - \|\vec{x}_i\|^2 \|\vec{x}_j\|^2\right) \lambda \\ - \|\vec{x}_i\|^2 \|\vec{x}_j\|^2 \left( \sum_{s=1}^k v_{is} v_{js}\right) =0.  \label{eq:eqn_mle_fh_appendix}
\end{multline}

However, instead of showing that $(\vec{v}_{i} ,\vec{v}_j)$ follows a multivariate Normal distribution, applying Lyapunov's Central Limit Theorem shows that the distribution of the tuple $(v_{is},v_{js})$, $1\leq s \leq k$ is bivariate Normal with
\begin{align*}
\left(\begin{array}{c}
v_{is}\\
v_{js}
\end{array}\right) \sim \mathcal N\left(\left(\begin{array}{c}
0 \\
0
\end{array}\right), \frac{1}{k}\left(\begin{array}{c c}
\|\vec{x}_i\|^2 & \langle \vec{x}_i,\vec{x}_j\rangle \\
\langle \vec{x}_i,\vec{x}_j\rangle & \|\vec{x}_j\|^2
\end{array}\right)\right),
\end{align*}
and hence an MLE is constructed to estimate $\langle\vec{v}_i, \vec{v}_j \rangle$. The probability density function of the distribution of $\{(v_{is}, v_{js})\}_{i=1}^k$ is given by
\begin{align*}
\frac{ \exp \left\{ k
    \left(\begin{array}{c c c c c}
        \frac{-\|\vec{x}_j\|^2}{2(\|\vec{x}_i\|^2\|\vec{x}_j\|^2-\langle \vec{x}_i,\vec{x}_j\rangle^2)}\\
        \frac{-\|\vec{x}_i\|^2}{2(\|\vec{x}_i\|^2\|\vec{x}_j\|^2-\langle \vec{x}_i,\vec{x}_j\rangle^2)}\\
        \frac{\langle \vec{x}_i,\vec{x}_j\rangle}{(\|\vec{x}_i\|^2\sigma_{22}-\langle \vec{x}_i,\vec{x}_j\rangle^2)}
    \end{array}\right)^T
    \left(\begin{matrix}
        \sum_{s=1}^k v_{is}^2\\
        \sum_{s=1}^k v_{js}^2\\
        \sum_{s=1}^k v_{is}v_{js}\\
    \end{matrix}\right)
    \right\}}
    { \exp\left\{ 
    k
     \left(
    \frac{1}{2} \log (\|\vec{x}_i\|^2\|\vec{x}_j\|^2-\langle \vec{x}_i,\vec{x}_j\rangle^2) 
    \right)
    \right\} } \frac{1}{(2\pi)^{\frac{k}{2}}}
\end{align*}
and the conditions in Theorem~\ref{equal_variances_theorem_thingy} hold. Hence, the asymptotic variance given by the MLE of $\langle \vec{x}_i,\vec{x}_j\rangle$ must be equivalent to the variance of the CVE.

The control variate estimator to estimate $\langle\vec{x}_i,\vec{x}_j\rangle$ is given by
\begin{align}
\langle \vec{x}_i,\vec{x}_j\rangle = \Expt{\sum_{s=1}^k v_{is}v_{js} + \hat{c}_1\left( \sum_{s=1}^k v_{is}^2 - \|\vec{x}_i\|^2 \right) + \hat{c}_2\left( \sum_{s=1}^k v_{js}^2 - \|\vec{x}_j\|^2 \right)}, \label{eq:eqn_cv_fh_append}
\end{align}
with
\begin{align*}
\hat{c}_1 = -\frac{\langle \vec{x}_i,\vec{x}_j\rangle \|\vec{x}_j\|^2}{\langle \vec{x}_i,\vec{x}_j\rangle^2 + \|\vec{x}_i\|^2 \|\vec{x}_j\|^2}~~~~~\hat{c}_2 = -\frac{\langle \vec{x}_i,\vec{x}_j\rangle \|\vec{x}_i\|^2}{\langle \vec{x}_i,\vec{x}_j\rangle^2 + \|\vec{x}_i\|^2 \|\vec{x}_j\|^2}.
\end{align*}
Since $\langle \vec{x}_i,\vec{x}_j\rangle$ is what we want to estimate, yet appears in $\hat{c}_1,\hat{c}_2$, Algorithm~\ref{FP_algo_format} gives an algorithm of the form
\begin{align}
f_{n+1} & = \sum_{s=1}^k v_{is}v_{js} - \frac{f_n \|\vec{x}_j\|^2}{f_n^2 + \|\vec{x}_i\|^2 \|\vec{x}_j\|^2}\left( \sum_{s=1}^k v_{is}^2  - \|\vec{x}_i\|^2 \right) - \frac{f_n \|\vec{x}_i\|^2}{f_n^2 + \|\vec{x}_i\|^2 \|\vec{x}_j\|^2}\left( \sum_{s=1}^k v_{js}^2 - \|\vec{x}_j\|^2 \right) \notag \\
 & = \sum_{s=1}^k v_{is}v_{js} - \frac{ f_n \left( \|\vec{x}_j\|^2 \left( \sum_{s=1}^k v_{is}^2 - \|\vec{x}_i\|^2  \right) +\|\vec{x}_i\|^2 \left( \sum_{s=1}^k v_{js}^2 - \|\vec{x}_j\|^2  \right) \right)}{ f_n^2 + \|\vec{x}_i\|^2 \|\vec{x}_j\|^2}  \label{eq:eqn_iter_cv_fh_appendix}
\end{align}
where we set $f_1 = \sum_{s=1}^k v_{is}v_{js}$.

Note that the control variate estimate given in Equation~\eqref{eq:eqn_cv_fh_append} is a generic version of the control variate estimate proposed by \cite{verma2022variance} and converges to their estimate when $\hat{c}_1 = \hat{c}_2$.

We use the following experimental setup to demonstrate the practical applications of the equivalence between CVE and MLE using the feature hashing algorithm.

\textbf{Hardware Description:} We performed the experiments on a machine having the following configuration: CPU: Intel(R) Core(TM) i7-8750H CPU @ 2.21GHz x 6; Memory: 16 GB; OS: Windows 10. 

\textbf{Dataset:} We randomly generate vector pairs $\vec{x}_1,\vec{x}_2 \in \mathbb R^{100,000}$ with varying squared norms and angles between them. Our experimental study involves the ratios $r \in \{0.1,0.5,1,2,10\}$ where $\|\vec{x}_1\|^2 = r \|\vec{x}_2\|^2$, and angles $\theta \in \{\frac{\pi}{12}, \frac{\pi}{4}, \frac{\pi}{2}, \frac{3\pi}{4}, \frac{11\pi}{12} \}$.\\

We generate feature hashing sketches, denoted as $\vec{v}_i$ and $\vec{v}_j$, for the vector pairs $\vec{x}_{i}$ and $\vec{x}_{j}$, respectively, for different sketch sizes ($k$) ranging from $1$ to $100$ using Equation~\eqref{eq:eqn_fh_append}. We then compute the estimate of the inner product as $\langle \vec{x}_i, \vec{x}_j\rangle \approx \sum_{s=1}^k v_{is}v_{js}$ and call it the baseline estimate. Further, we compute estimates of the inner product using: a) MLE \citep{verma2022variance} (Equation~\eqref{eq:eqn_mle_fh_appendix}) via Newton Raphson ({\tt MLE-NR}), b) MLE (Equation~\eqref{eq:eqn_mle_fh_appendix}) via the Secant method ({\tt MLE-Secant}), c) CV  (Equation~\eqref{eq:eqn_cv_fh_append} where the initial estimate $\sum_{s=1}^k v_{is}v_{js}$is used as part of $\hat{c}_i$ ({\tt CV-Init}), d) CV (Equation~\eqref{eq:eqn_cv_fh_append}) where the empirical covariances and variances are found for $\hat{c}_i$ ({\tt CV-Emp}), and e) our CVE   (Algorithm~\ref{FP_algo_format}) ({\tt CV-FP}). 
The MLE formulation implies that all five methods are finding the root of a cubic. We repeat this procedure using different hash functions (or projection matrix $\MatrixForm{R}$) $10,000$ times and compute the corresponding inner product estimates. We also record the number of updates each method took to converge to the root and the number of cubic equations with three real roots over $10,000$ iterations.

We use the following metrics to evaluate the performance of the aforementioned five methods:

(i) \textbf{Mean Square Error (MSE) w.r.t. $k$ over $10,000$ iterations across various combinations of $r$ and $\theta$:} A lower MSE indicates better performance of the method. We compute the MSE of the estimates provided by the baseline methods over the $10,000$ independent iterations and summarise them in Figure~\ref{5by5plot_FH}.

We can make two observations from the plots in Figure~\ref{5by5plot_FH}. First, when the angle $\theta$ between the two vectors is fixed, the squared norm of the vectors does not affect the convergence of all five methods, as each horizontal row displays a similar trend. Second, there is varying performance among the inner product estimates obtained via different root-finding algorithms when vectors are either close to each other ($\theta = \frac{\pi}{12})$ or far apart from each other ($\theta = \frac{11 \pi}{12}$).

\begin{figure}[h]
\begin{center}
       \includegraphics[width=16cm]{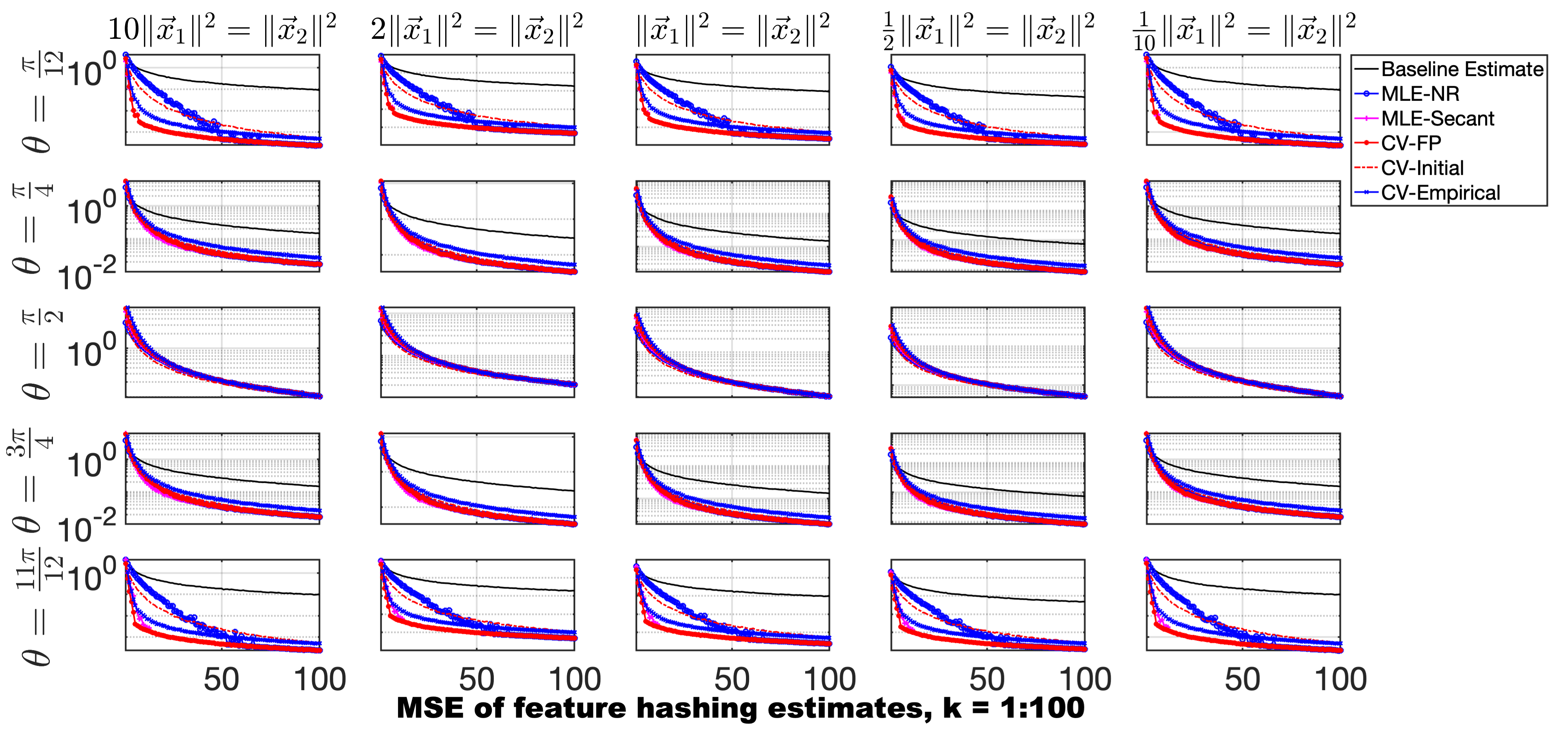}
    \caption{Plot of mean squared error (MSE) of respective estimates of inner product between $\vec{x}_1, \vec{x}_2$ as the squared vector norm and angle $\theta$ between the vectors vary for feature hashing. A lower value is better. \label{5by5plot_FH}}
\end{center}
\end{figure}

Further, the plots depicted in Figure~\ref{single_MSE_FH} examine the case where $\|\vec{x}_1\|^2 = \|\vec{x}_2\|^2$, and $\theta = \frac{\pi}{12}$ and summarizes the MSE of five baseline methods. From Figure~\ref{single_MSE_FH}, it is clear that {\tt MLE-NR} has inferior performance (excluding the baseline random projection estimate) until the number of observations $k$ surpasses a certain value ($\approx 70$); after that, it has similar performance to {\tt CV-FP} and {\tt MLE-Secant}. For the smaller value of $k$, the MSE of {\tt CV-Emp} is lower than {\tt CV-Init}. However, the MSE of both algorithms converges as $k$ increases. {\tt MLE-Secant} and {\tt CV-FP} have the same performance, which is superior to the other algorithms. 

\begin{figure}[h]
\begin{center}
\hspace*{-1cm} \includegraphics[width=15cm]{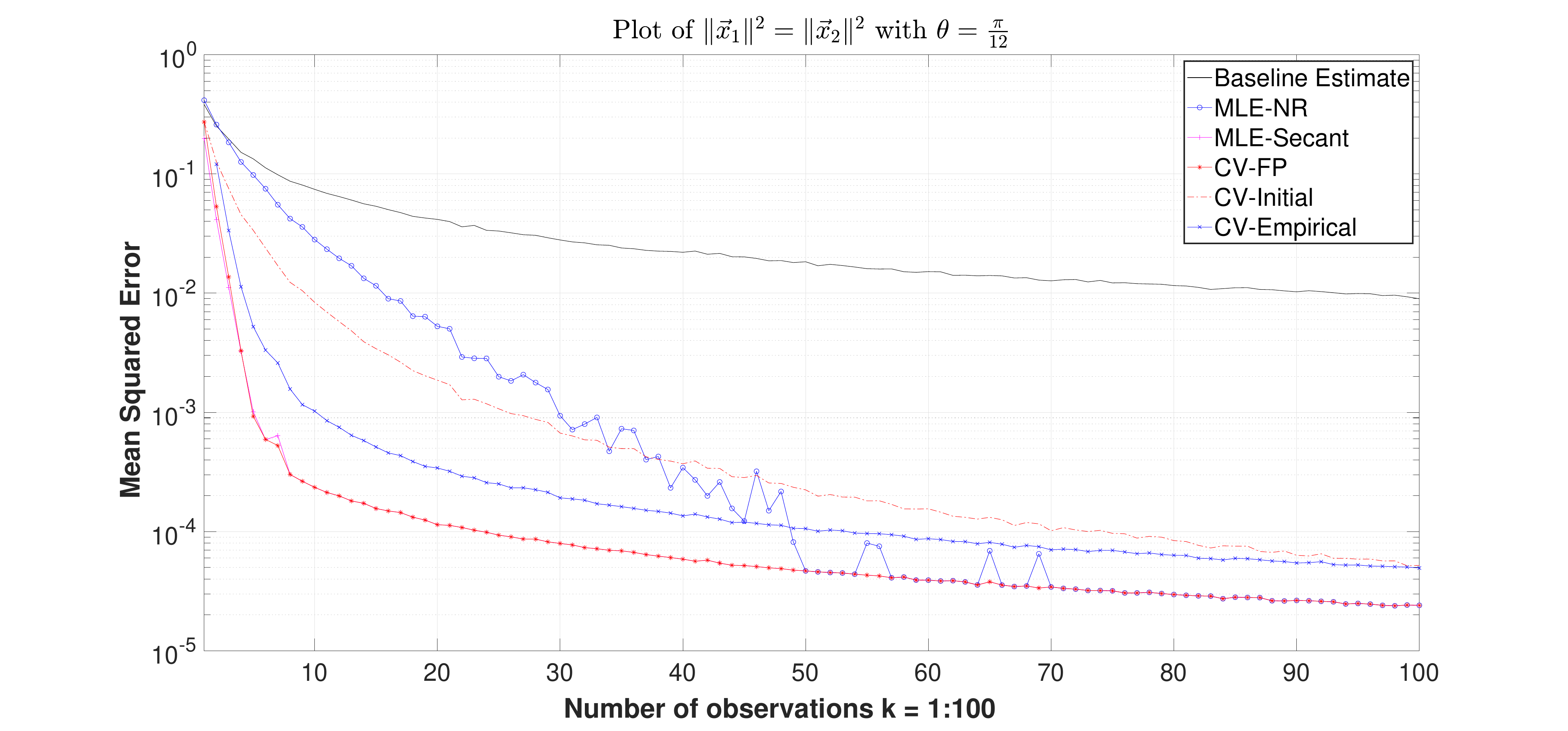}
    \caption{MSE plot when $\|\vec{x}_1\|^2 =\|\vec{x}_2\|^2$ with $\theta = \frac{\pi}{12}$ for feature hashing. A lower value is better. \label{single_MSE_FH}}
\end{center}
\end{figure}

(ii) \textbf{Proportion of outliers in $10,000$ estimates:} A smaller proportion of outliers suggests better performance. We generate the boxplots of the estimates of the inner product obtained using baseline methods over $10,000$  runs and present them in Figure~\ref{side_by_side_bp_FH} for $k=10$ and $k=20$. We summarise the corresponding proportion of outliers and a fraction of cubic equations with three real roots in Table~\ref{outlier_FH_table}.

\begin{figure}[h]
\begin{center}
\includegraphics[width=15cm]{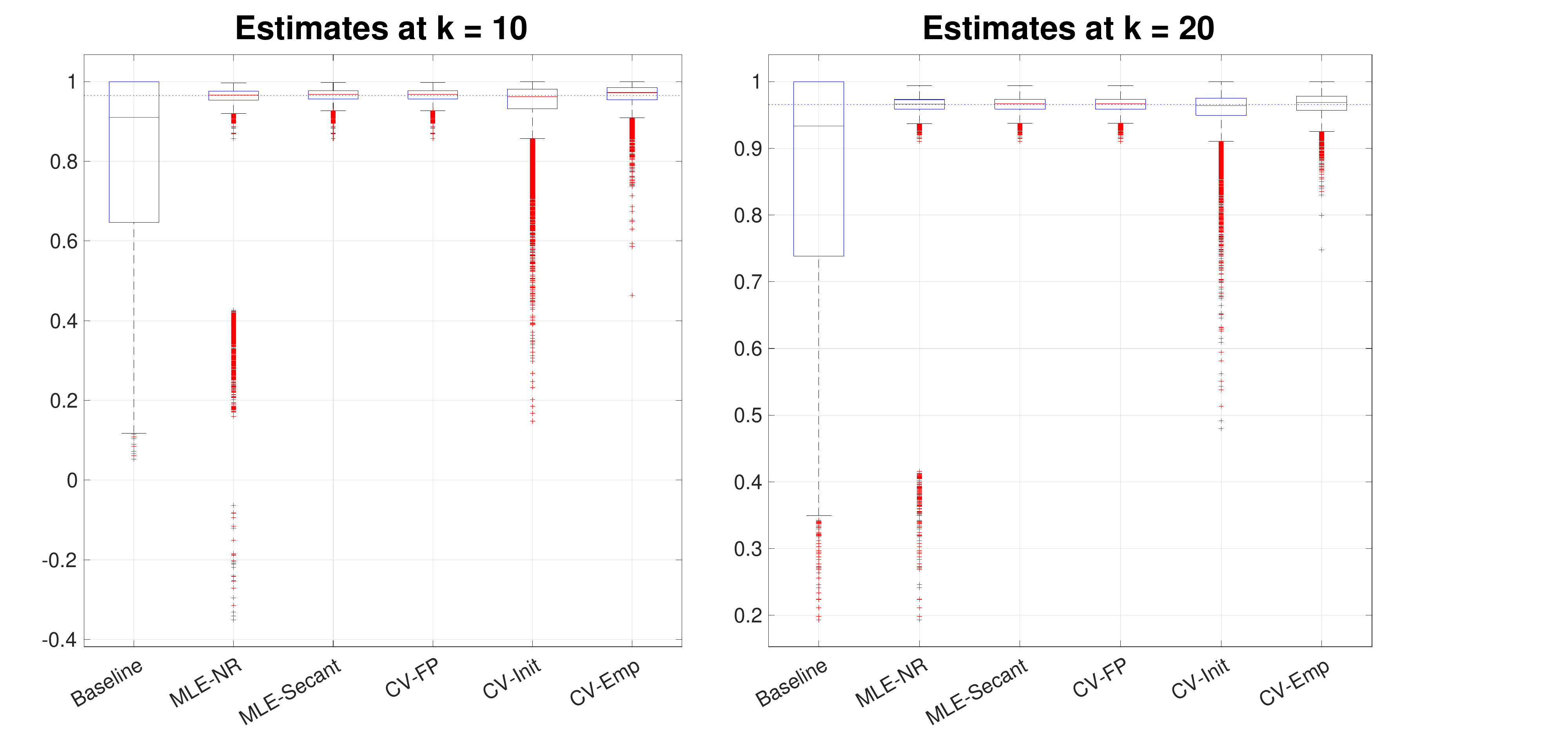}
    \caption{Boxplots of estimated inner products via the baseline feature hashing estimate ({\tt FH}), Newton Raphson ({\tt MLE-NR}), Secant method ({\tt MLE-Secant}), our algorithm ({\tt CV-FP}), CV using initial estimate ({\tt CV-Init}), and CV using empirical values of variances and covariances ({\tt CV-Emp}). The blue horizontal line denotes the true inner product between the vectors $\vec{x}_1,\vec{x}_2$. A narrower boxplot with few outliers is better. \label{side_by_side_bp_FH}}
\end{center}
\end{figure}

\begin{table}[h]
\begin{center}
\begin{tabular}{|lcc|}
\hline
\multicolumn{3}{|c|}{\textbf{Proportion of Outliers / Cubic with three real roots}}                        \\ \hline
\multicolumn{1}{|l|}{\textbf{Method}} & \multicolumn{1}{c|}{$\boldsymbol{k = 10}$} & $\boldsymbol{k = 20}$ \\ \hline
\multicolumn{1}{|l|}{MLE-NR}          & \multicolumn{1}{c|}{$0.1097$} & $0.0871$ \\ \hline
\multicolumn{1}{|l|}{MLE-Secant}      & \multicolumn{1}{c|}{$0.0670$} & $0.0758$ \\ \hline
\multicolumn{1}{|l|}{CV-FP}           & \multicolumn{1}{c|}{$0.0670$} & $0.0759$ \\ \hline
\multicolumn{1}{|l|}{CV-Init}         & \multicolumn{1}{c|}{$0.1156$} & $0.1959$ \\ \hline
\multicolumn{1}{|l|}{CV-Emp}          & \multicolumn{1}{c|}{$0.1020$} & $0.0736$ \\ \hline
\multicolumn{1}{|l|}{{\it Cubic with three real roots}}          & \multicolumn{1}{c|}{$0.0026$} & $0$ \\ \hline
\end{tabular}
\end{center}
\caption{Proportion of outliers in the first five lines of the table. Proportion of cubics with three real roots for the five methods at $k= 10$ and $k = 20$.  \label{outlier_FH_table}}
\end{table}

\clearpage

From Figure~\ref{side_by_side_bp_FH}, it is evident that the interquartile range of estimates using all five algorithms is considerably smaller compared to the baseline feature hashing estimate, and the median closely approximates the true inner product between the vectors $\langle \vec{x}_1,\vec{x}_2\rangle$. Table~\ref{outlier_FH_table} highlights that the proportion of outliers for both {\tt MLE-NR} and {\tt CV-Emp} decreases from $k=10$ to $k=20$, though {\tt MLE-NR} exhibits outliers further from the true value compared to {\tt CV-Emp} and {\tt CV-Init} (as depicted in Figure~\ref{side_by_side_bp_FH}). Furthermore, at $k=10$, there are $26$ out of $10,000$ cubic equations with three real roots, whereas at $k=20$, there are none, indicating that the poor performance of NR at small $k$ is not due to the fact that there are multiple roots. CV-FP and MLE-Secant have a narrower interquartile range and a lower proportion of outliers than the others, resulting in their superiority over other methods. 

(iii) \textbf{Number of updates required by each method until convergence to the estimate at respective sketch size ($k$):} A smaller number of updates implies that the method converges faster to the root. As we repeat the experiments $10,000$ times, each time we record the number of updates {\tt MLE-NR}, {\tt MLE-Secant}, and {\tt CV-FP} methods took to coverage to the root (i.e., inner product estimate) at different values of the sketch size ($k$). We generate the boxplots corresponding to the number of updates and summarise it in Figure~\ref{boxplot_update_steps_FH}.

\begin{figure}[h]
\begin{center}
 \includegraphics[width=15cm]{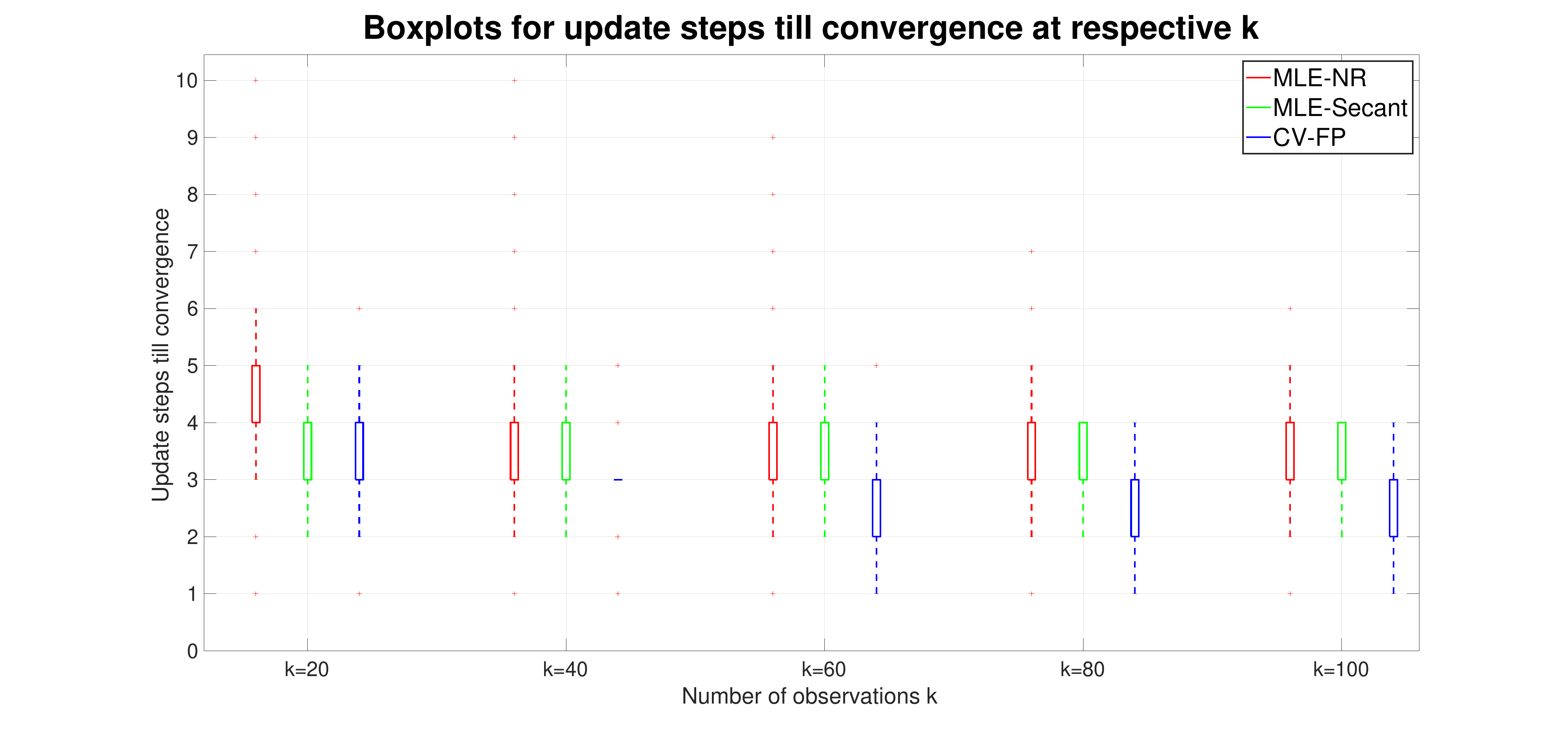}
\caption{Boxplots of update steps for {\tt MLE-NR}, {\tt MLE-Secant}, and {\tt CV-FP} at $k = \{20,40,60,80,100\}$ for feature hashing. A lower and narrower boxplot is better. \label{boxplot_update_steps_FH} }
\end{center}
\end{figure}
From Figure~\ref{boxplot_update_steps_FH}, it is clear that {\tt CV-FP} takes fewer update steps to converge, and this holds generally throughout the different vector squared norms and varying angles $\theta$. These experimental results indicate that {\tt CV-FP} empirically achieves faster convergence than other baseline methods.

We also give a table of time taken for convergence to show that fewer number of update steps also implies faster time in Table~\ref{table_time_taken}.

\begin{table}[h]
\begin{center}
\begin{tabular}{|c | c c| c c c|} \hline
\ & $\frac{\text{CV-FP}}{\text{Newton Raphson}}$ & $\frac{\text{CV-FP}}{\text{Secant method}}$ \\ \hline
10 & 0.905 $\pm$ 0.354 ~~~& 0.968 $\pm$ 0.414 \\
20 & 0.882 $\pm$ 0.340 ~~~& 0.841 $\pm$ 0.295 \\
30 & 0.910 $\pm$ 0.362 ~~~& 0.980 $\pm$ 0.364 \\
40 & 0.918 $\pm$ 0.423 ~~~& 0.899 $\pm$ 0.416 \\
50 & 0.877 $\pm$ 0.452 ~~~& 0.821 $\pm$ 0.432 \\
60 & 0.961 $\pm$ 0.420 ~~~& 0.875 $\pm$ 0.392 \\
70 & 0.945 $\pm$ 0.437 ~~~& 0.867 $\pm$ 0.416 \\
80 & 0.956 $\pm$ 0.474 ~~~& 0.826 $\pm$ 0.392 \\
90 & 0.890 $\pm$ 0.419 ~~~& 0.807 $\pm$ 0.340 \\
100 & 0.892 $\pm$ 0.494 ~~~& 0.813 $\pm$ 0.448 \\ \hline
\end{tabular}
\end{center}
\caption{Ratios of time taken with 2 standard deviations over 10000 iterations with $\|\vec{x}_1\|^2 = \|\vec{x}_2\|^2, \theta = \frac{\pi}{12}$, for $k \in \{10, \hdots, 100\}$. A ratio less than 1 means that the method on the numerator is faster than the method on the denominator.\label{table_time_taken}}
\end{table}

\subsection{RANDOM PROJECTION: FASTER CONVERGENCE AND BETTER STABILITY FOR THE MLE} \label{sec:rp_exp_appendix}

Suppose $\MatrixForm{X}_{n \times p}$ is a data matrix, and $\MatrixForm{R}_{p \times k}$ is matrix with entries $r_{ij} \sim \mathcal N(0,1)$. Consider the matrix $\MatrixForm{V}_{n \times k} = \MatrixForm{X}\MatrixForm{R}$. For any two rows $\vec{v}_i, \vec{v}_j$ in the matrix, the distribution of the tuple $(v_{is}, v_{js}), 1 \leq s \leq k$ is bivariate normal, with
\begin{align*}
\left(\begin{array}{c}
v_{is}\\
v_{js}
\end{array}\right) \sim \mathcal N\left(\left(\begin{array}{c}
0 \\
0
\end{array}\right) \left(\begin{array}{c c}
\|\vec{x}_i\|^2 & \langle \vec{x}_i,\vec{x}_j\rangle \\
\langle \vec{x}_i,\vec{x}_j\rangle & \|\vec{x}_j\|^2
\end{array}\right)\right).
\end{align*}

\cite{li2006improving} used the fact that if the original lengths of the vectors $\vec{x}_1,\hdots,\vec{x}_n$ were known (e.g. normalized to the length of 1), then an MLE could be constructed to estimate $\langle \vec{x}_i,\vec{x}_j\rangle$. The MLE estimate of the inner product is the real root of the following cubic equation~\cite{li2006improving}:
\begin{align}
    \lambda^3 - \left( \frac{\sum_{s=1}^k v_{is} v_{js}}{k}\right) \lambda^2 + \left( \|\vec{x}_{i}\|^2 \left(\frac{\sum_{s=1}^k v_{js}^2}{k} \right) +\|\vec{x}_{j}\|^2 \left(\frac{\sum_{s=1}^k v_{is}^2}{k} \right)  - \|\vec{x}_i\|^2 \|\vec{x}_{j}\|^2\right) \lambda \notag \\
    \phantom{=} - \left(\frac{\sum_{s=1}^k v_{is} v_{js}}{k}\right) \| \vec{x}_i \|^2 \|\vec{x}_j\|^2 = 0,  \label{eq:eqn_mle_rp_append}
\end{align}
and we note that Equation~\eqref{eq:eqn_mle_rp_append} is identical to Equation~\eqref{eq:eqn_mle_fh_appendix}. Hence we follow the same steps to get the update step similar to feature hashing.

We adopt a similar experimental setup used for the feature hashing experiments in Section~\ref{sec:fh_exp_appendix}.

We run simulations on vector pairs $\vec{x}_1,\vec{x}_2 \in \mathbb R^{100,000}$ with varying squared norms and angles between them over 10,000 iterations.  We look at the ratios $r \in \{0.1,0.5,1,2,10\}$ where $\|\vec{x}_1\|^2 = r \|\vec{x}_2\|^2$, and angles $\theta \in \{\frac{\pi}{12}, \frac{\pi}{4}, \frac{\pi}{2}, \frac{3\pi}{4}, \frac{11\pi}{12} \}$. For each iteration, we compute $V=XR$ and estimate the inner product as $\langle\vec{x}_i, \vec{x}_j \rangle \approx \frac{\sum{s=1}^k v_{is}v_{js}}{k}$ for $1\leq k \leq 100$, referring to this estimate as the baseline estimate. We compute the estimates of the inner product using: a) MLE \citep{li2006improving}  (Equation~\eqref{eq:eqn_mle_rp_append}) via Newton Raphson ({\tt MLE-NR}) b) MLE (Equation~\eqref{eq:eqn_mle_rp_append}) via the Secant method ({\tt MLE-Secant}), c) CV where the initial estimate $\frac{\sum_{s=1}^k v_{is}v_{js}}{k}$ is used as part of $\hat{c}_i$ ({\tt CV-Init}), d) CV where the empirical covariances and variances are found for $\hat{c}_i$ ({\tt CV-Emp}), and e) our CVE (Algorithm~\ref{FP_algo_format}) ({\tt CV-FP}).  
We record the inner product estimate provided by the above-mentioned methods and the number of updates each algorithm takes to converge the estimate in each iteration. Additionally, we note that cubic equations possess three real roots within the 10,000 iterations. We compare the performance of these methods using the following metrics:

(i) \textbf{Mean Square Error (MSE) w.r.t. $k$ over $10,000$ iterations across various combinations of $r$ and $\theta$:} We compute the MSE of the estimates generated by baseline methods across various values of sketch size ($k$) over $10,000$ iterations for different combinations of $r$ and $\theta$. The MSE results are summarised in Figure~\ref{5by5plot_RP}.
\begin{figure}[h]
\begin{center}
       \includegraphics[width=16cm]{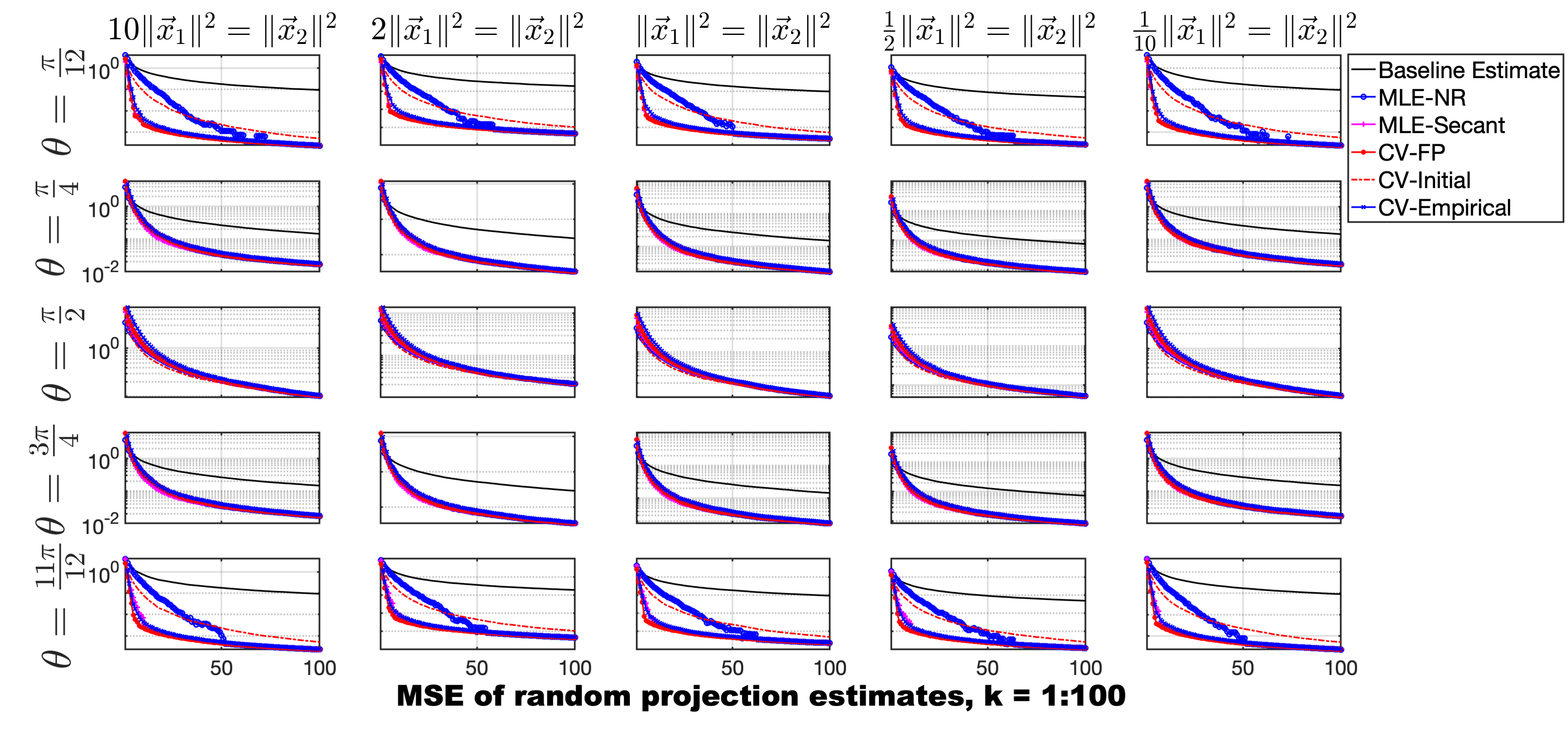}
    \caption{Plot of mean squared error (MSE) of respective estimates of inner product between $\vec{x}_1, \vec{x}_2$ as the squared vector norm and angle $\theta$ between the vectors vary for random projections. A lower value is better. \label{5by5plot_RP}}
\end{center}
\end{figure}

From Figure~\ref{5by5plot_RP}, we can make two observations. First, for a fixed angle $\theta$ between the two vectors, the squared norm of the vectors does not affect the convergence of the baseline methods, since each horizontal row shows roughly the same trend. Second, there is varying performance between the inner product estimates via different root-finding algorithms when vectors are close to each other ($\theta = \frac{\pi}{12})$ or far away from each other ($\theta = \frac{11}{12}$).

The subsequent plots in Figure~\ref{single_MSE_RP} depict the MSE of the baseline methods for $\| \vec{x}_1 \|^2 = \| \vec{x}_2 \|^2$ and $\theta = \frac{\pi}{12}$. It is evident from Figure~\ref{single_MSE_RP} that \texttt{MLE-NR} exhibits the poorest performance (apart from the baseline random projection estimate) until the number of observations $k$ surpasses a certain threshold (approximately $50$), after which it demonstrates similar performance to \texttt{CV-FP} and \texttt{MLE-Secant}. At lower values of $k$, \texttt{CV-Emp} has a slightly higher MSE compared to \texttt{MLE-Secant} and \texttt{CV-FP}, while at higher values of $k$ their MSE converges. The \texttt{MLE-Secant} and \texttt{CV-FP} exhibit nearly identical performance and outperform other algorithms across the entire range of sketch sizes ($k$).

\begin{figure}[h]
\begin{center}
\includegraphics[width=15cm]{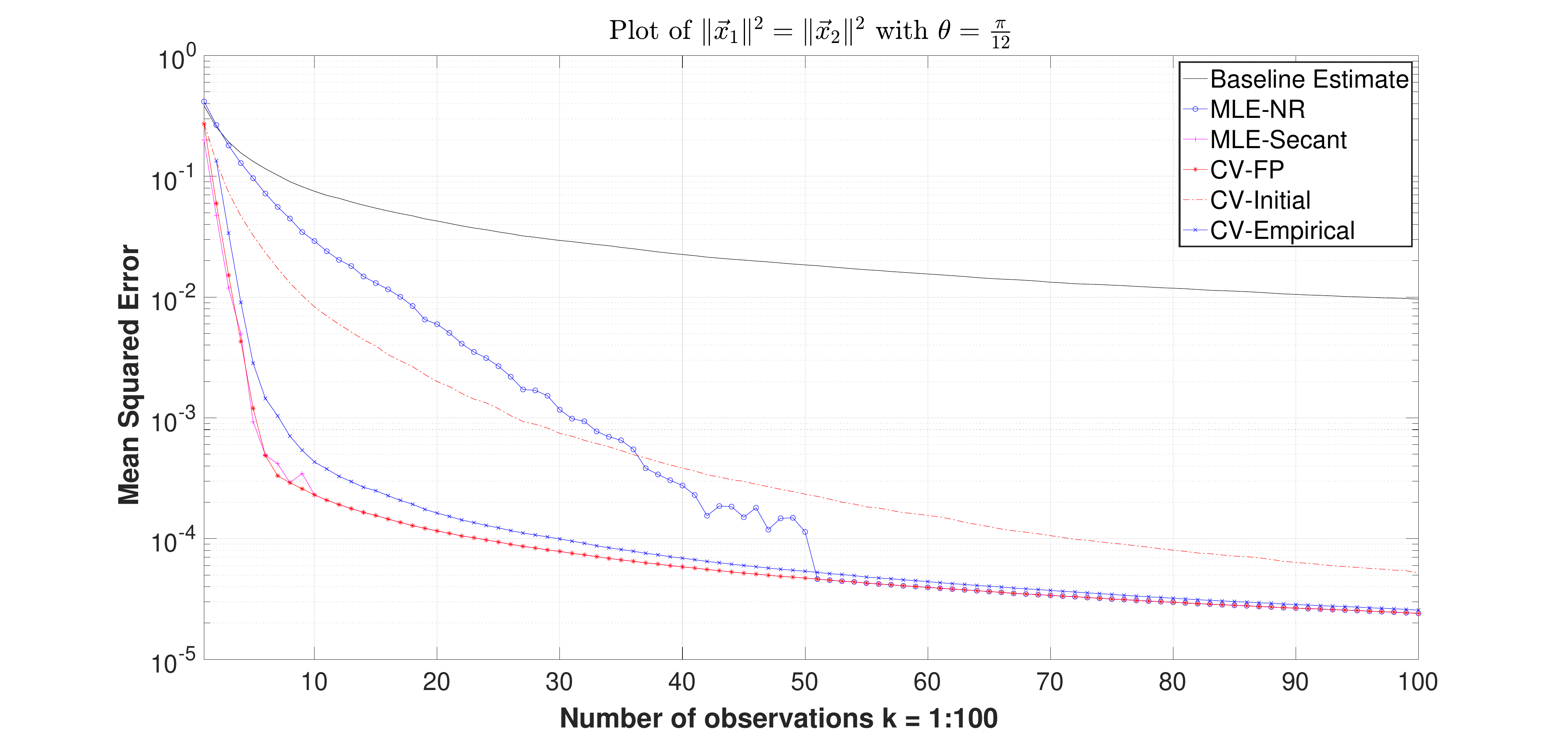}
    \caption{MSE plot when $\|\vec{x}_1\|^2 =\|\vec{x}_2\|^2$ with $\theta = \frac{\pi}{12}$ for random projections. A lower value is better. \label{single_MSE_RP}}
\end{center}
\end{figure}

(ii) \textbf{Proportion of outliers in $10,000$ estimates:} For each baseline method, we construct boxplots using the $10,000$ estimates obtained over the 10,000 iterations. The boxplots for $k=10$ and $k=20$ are depicted in Figure~\ref{side_by_side_bp_RP}. Furthermore, we provide a summary of the corresponding proportion of outliers in estimates and cubic equations with three real roots over 10,000 iterations in Table~\ref{outlier_RP_table}.
\begin{figure}[h]
\begin{center}
\includegraphics[width=15cm]{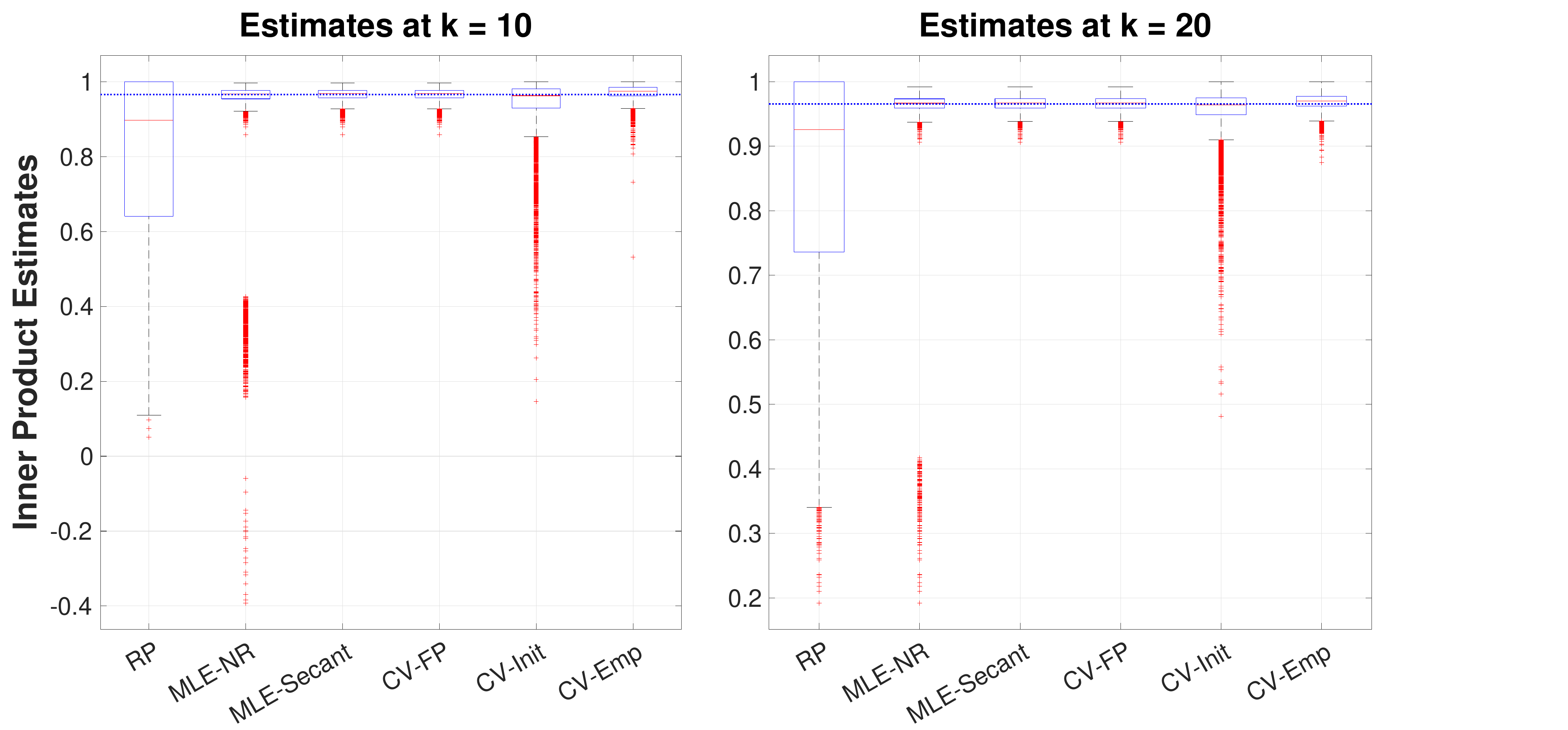}
    \caption{Boxplots of estimated inner products via the baseline random projection estimate ({\tt RP}), Newton Raphson ({\tt MLE-NR}), Secant method ({\tt MLE-Secant}), our algorithm ({\tt CV-FP}), CV using initial estimate ({\tt CV-Init}), and CV using empirical values of variances and covariances ({\tt CV-Emp}). The blue horizontal line denotes the true inner product between the vectors $\vec{x}_1,\vec{x}_2$. A narrower boxplot is better. \label{side_by_side_bp_RP}}
\end{center}
\end{figure}
\begin{table}[h]
\begin{center}
\begin{tabular}{|lcc|}
\hline
\multicolumn{3}{|c|}{\textbf{Proportion of Outliers / Cubic with three real roots}}                        \\ \hline
\multicolumn{1}{|l|}{\textbf{Method}} & \multicolumn{1}{c|}{k = 10} & k = 20 \\ \hline
\multicolumn{1}{|l|}{MLE-NR}          & \multicolumn{1}{c|}{0.1079} & 0.0861 \\ \hline
\multicolumn{1}{|l|}{MLE-Secant}      & \multicolumn{1}{c|}{0.0599} & 0.0737 \\ \hline
\multicolumn{1}{|l|}{CV-FP}           & \multicolumn{1}{c|}{0.0599} & 0.0737 \\ \hline
\multicolumn{1}{|l|}{CV-Init}         & \multicolumn{1}{c|}{0.1080} & 0.1936 \\ \hline
\multicolumn{1}{|l|}{CV-Emp}          & \multicolumn{1}{c|}{0.1057} & 0.0925 \\ \hline
\multicolumn{1}{|l|}{{\it Cubic with three real roots}}          & \multicolumn{1}{c|}{0.0021} & 0 \\ \hline
\end{tabular}
\end{center}
\caption{Proportion of outliers in the first five lines of the table. Proportion of cubics with three real roots for the five methods at $k= 10$ and $k = 20$. \label{outlier_RP_table}}
\end{table}

From Figure~\ref{side_by_side_bp_RP}, it is evident that the interquartile range of the estimates obtained using all five baseline methods is much smaller compared to the baseline random projection estimate, and the median is close to the true inner product between the two vectors $\langle \vec{x}_1,\vec{x}_2\rangle$. From Table~\ref{outlier_RP_table} we can see that the proportion of outliers for {\tt MLE-NR} and {\tt CV-Emp} decrease from $k= 10$ to $k = 20$, however {\tt MLE-NR} has outliers farther away from the true value as compared to {\tt CV-Emp} and {\tt CV-Init} (Figure~\ref{side_by_side_bp_RP}). Moreover, from Table~\ref{outlier_RP_table},  we can see that at $k = 10$, there are only 21 out of 10,000 cubics with three real roots, and at $k = 20$, there are 0 cubics with three real roots, implying that the poor performance of NR at small $k$ cannot solely be attributed to the presence of cubics with multiple real roots. Finally, it is apparent from both Figure~\ref{side_by_side_bp_RP} and Table~\ref{outlier_RP_table} that both \texttt{CV-FP} and \texttt{MLE-Secant} have a smaller interquartile range and a lower proportion of outliers compared to the other baseline methods, thus demonstrating their superiority over the other methods.

\clearpage

(iii) \textbf{Number of updates required by each method until convergence to the estimate at respective sketch size ($k$):} Lastly, we generate the boxplot corresponding to the number of updates each method took to converge to the inner product estimate in each of the $10,000$ iterations across various values of the sketch size ($k$). These results are summarized in Figure~\ref{boxplot_update_steps_RP}.
\begin{figure}[h]
\begin{center}
 \includegraphics[width=15cm]{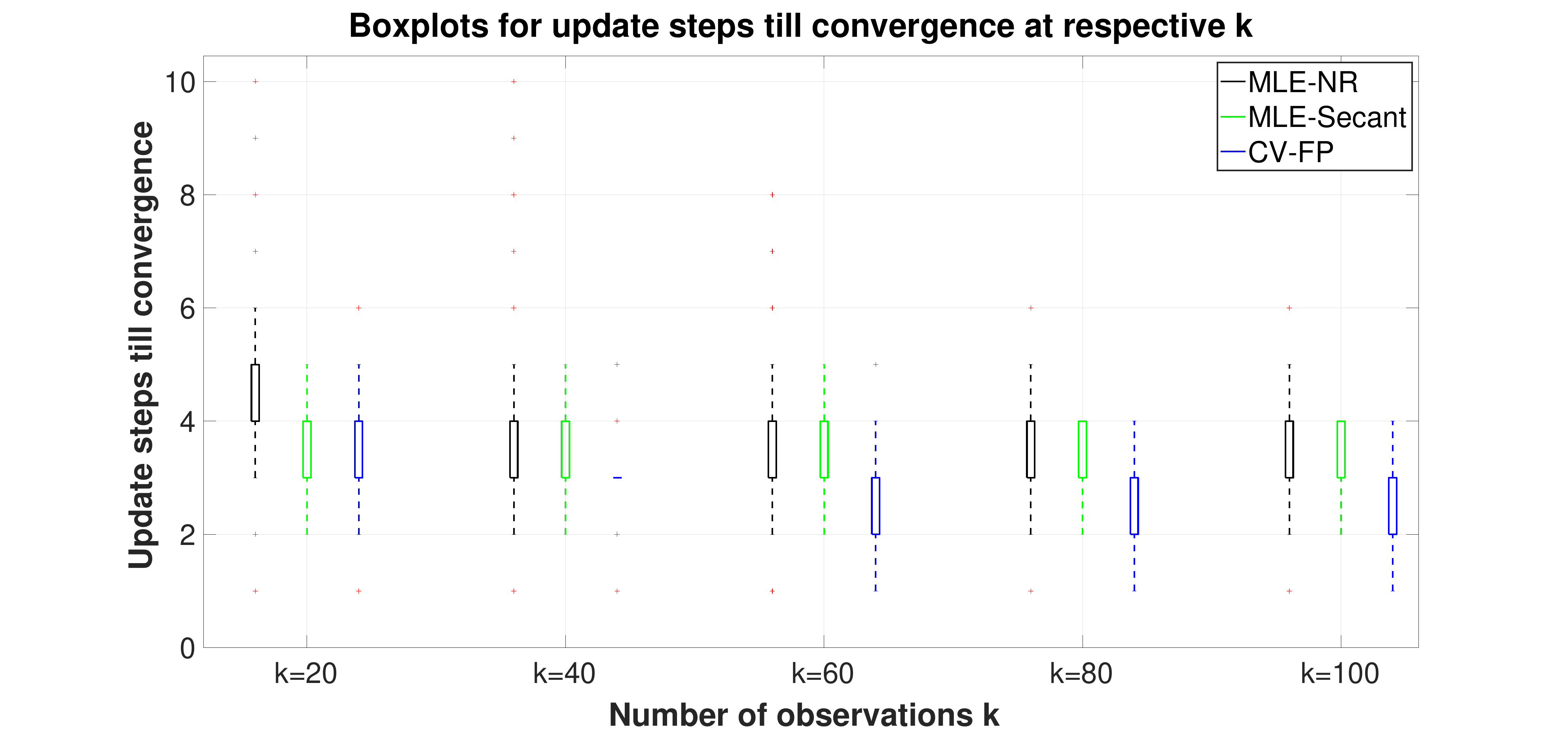}
\caption{Boxplots of update steps for {\tt MLE-NR}, {\tt MLE-Secant}, and {\tt CV-FP} at $k = \{20,40,60,80,100\}$ for random projections. A lower and narrower boxplot is better. \label{boxplot_update_steps_RP} }
\end{center}
\end{figure}

From Figure~\ref{boxplot_update_steps_RP}, it is evident that {\tt CV-FP} requires fewer updates to converge to the estimate compared to {\tt MLE-NR} and {\tt MLE-Secant}. This trend holds generally throughout the different vector squared norms and varying angles $\theta$, suggesting that empirically {\tt CV-FP} achieves faster convergence to the estimate compared to the other methods.

\subsection{IMPLICATION OF OUR EXPERIMENTS}
We now summarize the implications of our experiments.

MLEs and CVEs are used to improve estimates when information about the parameters are known. However, most papers involving MLEs and CVEs focus on theoretical work (proving variance bounds) and experiments (comparing against benchmarked algorithms), but do not explicitly mention how the respective estimates from the CVE and MLE should be found, as the assumption is that a user can implement them on their own.

However, our plots show {\bf there is a great difference} in the methods used, in terms of speed of convergence (Figure~\ref{boxplot_update_steps_FH}, Figure~\ref{boxplot_update_steps_RP}), and in terms of error (Figure~\ref{5by5plot_FH}, Figure~\ref{5by5plot_RP}), if a suboptimal implementation of {\tt Newton Raphson} is used for the MLE, or computing $\hat{c}$ via {\tt CV-Init} (or {\tt CV-Emp}) via the CVE. This leads to possibly contradictory results based on the type of experiments conducted. 

This leads to possibly contradictory results based on the type of experiments conducted.

For example, identifying vectors with high similarity is one goal of similarity search, where the angle $\theta$ between the two vectors are small. This is where a difference in methods used can lead to different performance (e.g. {\tt Newton Raphson} does not have good performance with small $k$ in Figure~\ref{boxplot_update_steps_FH}, Figure~\ref{boxplot_update_steps_RP}), but that does not mean an MLE is bad.  Equivalently, outliers may be removed during experiments to show better results (our boxplots in Figure~\ref{boxplot_update_steps_FH},Figure~\ref{boxplot_update_steps_RP} show that {\tt MLE-NR} has similar (good) performance with {\tt MLE-Secant} and {\tt CV-FP} ; but in a practical scenario the outliers are not known in advance.

Moreover, if sole error bars are used (based on computing the standard deviations of the estimates), this can lead to thinking that an MLE has ``higher error", even though the outliers increase the standard deviations (Figure~\ref{side_by_side_bp_FH}, Figure~\ref{side_by_side_bp_RP}), and can be misleading. From Figure~\ref{side_by_side_bp_FH}, Figure~\ref{side_by_side_bp_RP}, there are more outliers in one direction for Newton Raphson. Using error bars implicitly assumes there is ``equal variation above and below the mean", and may mask what is happening in practice.

Hence we suggest that Algorithm~\ref{FP_algo_format} ({\tt CV-FP}) be used, which we expect will mitigate any reproducibility issues.

\subsection{REMARKS ON CONVERGENCE RATE OF FIXED POINT ITERATION}
\label{remarks_convergence_rates_rp}
Despite our experiments showing that {\tt CV-FP} takes fewer update steps to converge compared to {\tt MLE-NR} and {\tt MLE-Secant}, {\tt CV-FP} does not have quadratic convergence, and might not even have superlinear convergence. 

\begin{definition}
Suppose there is a sequence of real numbers $\{x_i\}_{i=1}^n$ such that $\lim_{n \rightarrow \infty} = x$. Further suppose
\begin{align}
\lim_{n\rightarrow \infty} \frac{|x_{n+1}-x|}{|x_n - x|^\alpha} = C < \infty. \label{def_convergence}
\end{align}
If $\alpha = 1$, the sequence converges linearly to $x$. If $1 < \alpha < 2$, the sequence converges superlinearly to $x$. If $\alpha = 2$, the sequence converges quadratically to $x$.
\end{definition}

The Newton Raphson algorithm has quadratic convergence, where $\alpha = 2$, and the Secant method has superlinear convergence, with $\alpha = \phi = 1.618...$ \citep{atkinson2008introduction}. To (empirically) check the convergence of an algorithm, where $x_n$ are the updated values, and $x$ the true value, we plot $\log(|x_{n+1}-x|)$ versus $\log(x_n - x)$ and record the gradient of the best fit line for our 10,000 iterations at $k = 100$ across all our vector pairs and angles (25,000 observations) in the bivariate Normal case for both feature hashing and random projections.

Equation~\eqref{def_convergence} gives the motivation behind this, since
\begin{align*}
& |x_{n+1}-x| = C|x_n - x|^\alpha  \\
 \Rightarrow & \log(|x_{n+1}-x|) = \log(C) + \alpha \log(|x_n-x|).
\end{align*}

\begin{figure}[h]
\begin{center}
 \includegraphics[width=15cm]{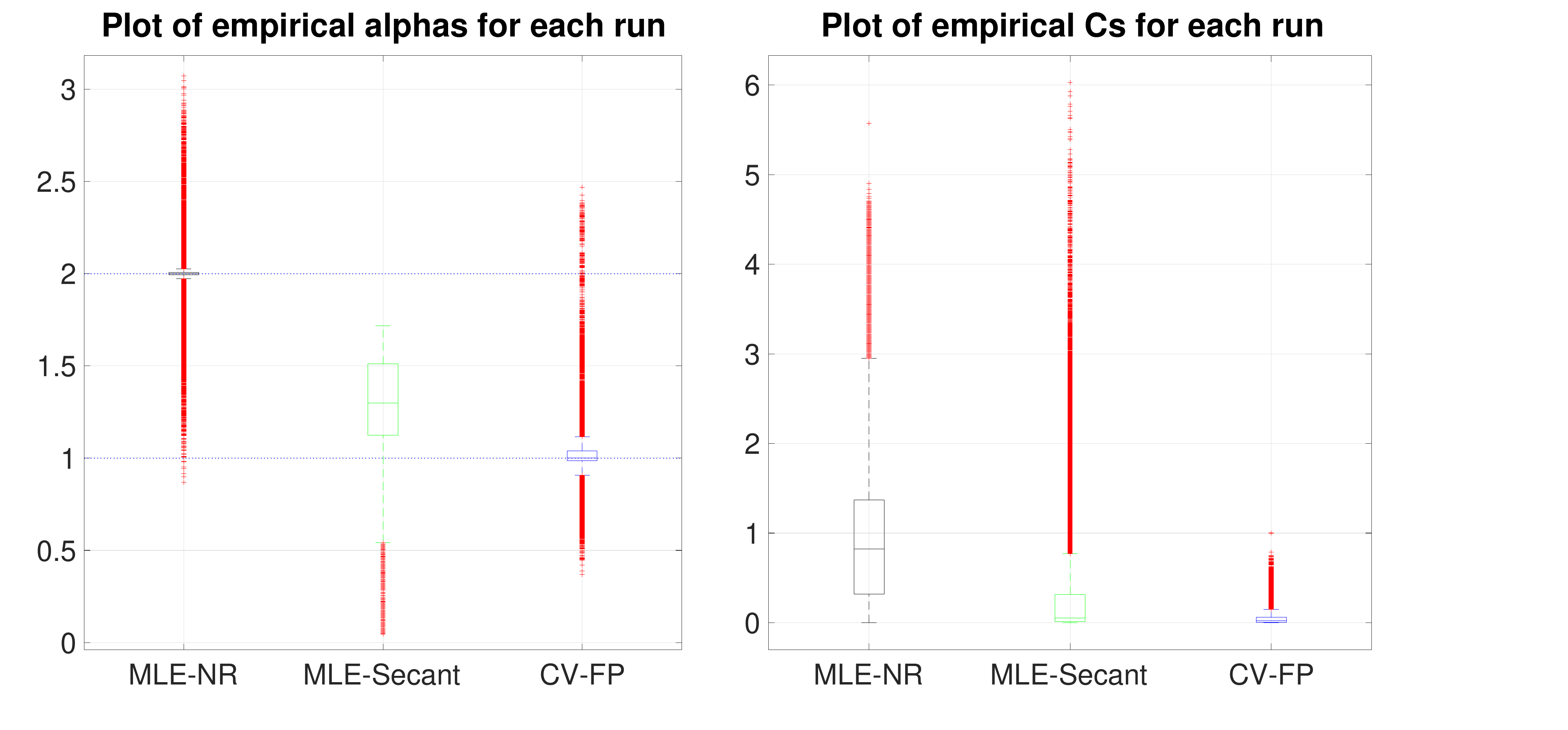}
\caption{Boxplots of empirical $\alpha$ and $C$ at $k = 100$ for {\tt MLE-NR}, {\tt MLE-Secant}, and {\tt CV-FP}  across all vector pairs at 25,000 observations for feature hashing. \label{boxplot_convergence_FH} }
\end{center}
\end{figure}

\begin{figure}[h]
\begin{center}
 \includegraphics[width=15cm]{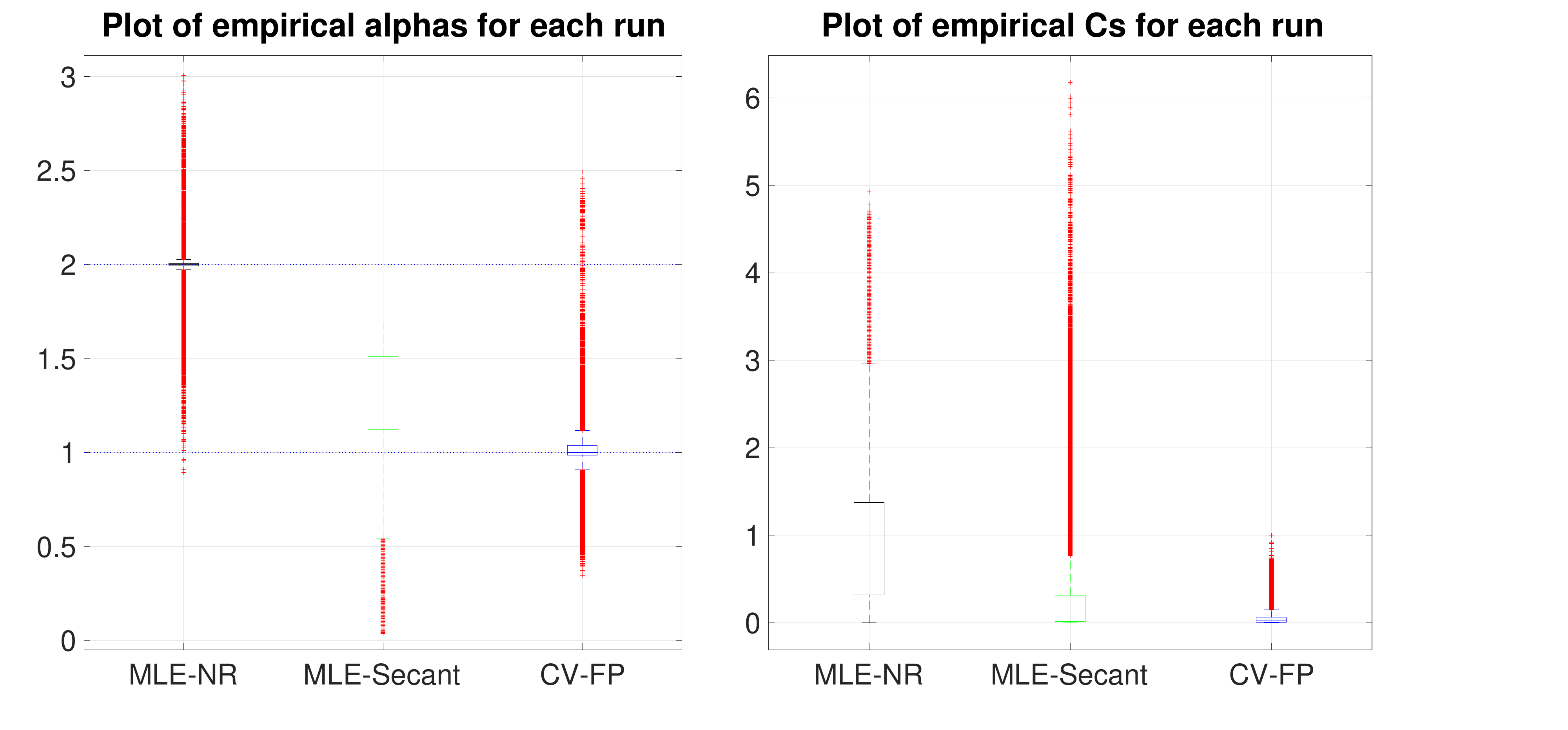}
\caption{Boxplots of empirical $\alpha$ and $C$ at $k = 100$ for {\tt MLE-NR}, {\tt MLE-Secant}, and {\tt CV-FP}  across all vector pairs at 25,000 observations for random projections. \label{boxplot_convergence_RP} }
\end{center}
\end{figure}

The boxplots in Figures \ref{boxplot_convergence_FH} and \ref{boxplot_convergence_RP} show the empirical $\alpha$ and $C$ for each run. The interquartile range of $\alpha$ for Newton Raphson is close to 2, but is less than $\phi$ for the Secant method. For {\tt CV-FP}, the interquartile range is slightly above 1, and hence the convergence of {\tt CV-FP} could either be linear or super-linear.

However, the empirical $C$ for {\tt CV-FP} is the lowest out of {\tt MLE-NR} and {\tt MLE-Secant}, which explains why there is faster convergence for {\tt CV-FP} in the number of update steps required. 

We leave this open for future work: to prove (or disprove) that the CVE has superlinear convergence, and that $C$ is lower than both Newton Raphson and the Secant method {\bf across all exponential families} satisfying the conditions in Theorem~\ref{equal_variances_theorem_thingy}.

\newpage
\section{FURTHER DISCUSSION ON HUTCHINSON'S ESTIMATOR}\label{secA5}

\label{hutch_expt_appendix}

This section of the appendix elaborates on our heuristic of how {\tt CV-FP} can be used to find different estimators, if the  conditions in Theorem~\ref{equal_variances_theorem_thingy} or Corollary~\ref{LC_corollary}, by either repeatedly iterating through the fixed point equations, or by finding a closed form solution analytically.

We demonstrate with Hutchinson's Estimator \citep{hutchinson1989stochastic} where we get an estimator, and analyze it directly.

Let $\vec{r}_i \in \mathbb R^d$ be distributed $\mathcal N(\vec{0}_d,\MatrixForm{I}_d)$. Let $y_i = \vec{r}_i^T\MatrixForm{M}\vec{r}_i \in \mathbb R$ for $1 \leq i \leq k$. For a symmetric $\MatrixForm{M}_{d\times d}$ matrix, the estimator $ Y = \sum_{i=1}^k \frac{y_i}{k} = \sum_{i=1}^k \frac{\vec{r}_i^T\MatrixForm{M}\vec{r}_i}{k}$ estimates the trace of $\MatrixForm{M}$, with $\Var{Y} = \frac{2\trace{\MatrixForm{M}^2}}{k} = \frac{2\|\MatrixForm{M}\|_F^2}{k}$, where $\|\MatrixForm{M}\|_F^2$ is the Frobenius norm, defined by $\|\MatrixForm{M}\|_F^2 = \sum_{i,j}m_{ij}^2$.

\cite{adams2018estimating} gives the following CVE for $\trace{\MatrixForm{M}}$ when there is a known diagonal matrix $\MatrixForm{B}$,

\begin{align}
Z = \sum_{i=1}^k \frac{\vec{r}_i^T\MatrixForm{M}\vec{r}_i}{k} + c \left( \sum_{i=1}^k \frac{\vec{r}_i^T\MatrixForm{B}\vec{r}_i}{k} - \trace{\MatrixForm{B}}\right) \label{easy_CV_estimator_hutch}
\end{align}

Lemma 4.1 in \cite{adams2018estimating} gives the optimal $\hat{c}$ to be
\begin{align*}
\hat{c} & = -\frac{\Cov{\vec{r}^T\MatrixForm{M}\vec{r}}{\vec{r}^T\MatrixForm{B}\vec{r}}}{\Var{\vec{r}^T\MatrixForm{B}\vec{r}}} = -\frac{\trace{\MatrixForm{M}\MatrixForm{B}}}{\trace{\MatrixForm{B}^2}},
\end{align*}
with a variance reduction of $\frac{2\trace{\MatrixForm{M}\MatrixForm{B}}^2}{k\trace{\MatrixForm{B}^2}}$. The empirical covariance $\Cov{\vec{r}^T\MatrixForm{M}\vec{r}}{\vec{r}^T\MatrixForm{B}\vec{r}}$ and variance $\Var{\vec{r}^T\MatrixForm{B}\vec{r}}$ is computed from the observed data, since $\trace{\MatrixForm{M}\MatrixForm{B}}$ is not known. 

{\it Given the work in \cite{adams2018estimating}, a natural question would be to ask if a MLE could be found, which might yield a lower variance reduction. However, $Y$ is distributed as a linear combination of $\chi^2$ variables, and finding an MLE is non-trivial.}

{\bf Heuristic 1: Find $y_i$ in the CVE formulation which correspond to linearly independent sufficient statistics}

Under the heuristic of $y_i$ needing to be sufficient statistics (of parameters $\eta_i$), and since $\trace{\MatrixForm{M}} = \sum_{i=1}^d m_{ii}$, we look for sufficient statistics of each $m_{ii}$. We decompose $\vec{r}^T\MatrixForm{M}\vec{r} = \sum_{j=1}^d r_j \left(\MatrixForm{M}\vec{r}\right)_j$, where $(\MatrixForm{M}\vec{r})_j$ denotes the $j^{\text{th}}$ component of the vector $\MatrixForm{M}\vec{r}$, to get a CVE of the form
\begin{align*}
Z & = \sum_{j=1}^d \frac{ \left( \sum_{i=1}^k  r_{ij} \left(\MatrixForm{M}\vec{r}_i\right)_j \right)}{k} + \sum_{j=1}^d c_i\left( \frac{ \sum_{i=1}^k r_{ij} \left(\MatrixForm{B}\vec{r}_i\right)_j}{k} - b_{ii}\right)  \\
 & = \tilde{y} + \sum_{j=1}^d c_i\left( \frac{ \sum_{i=1}^k r_{ij} \left(\MatrixForm{B}\vec{r}_i\right)_j}{k} - b_{ii}\right), 
\end{align*}
where each $r_{is}$ is the $s^{\text{th}}$ entry of $\vec{r}_i$. 

\begin{theorem}
\label{diag_CVE_thm}
For the CVE given by 
\begin{align*}
Z & = \sum_{j=1}^d \left(  \frac{\sum_{i=1}^k r_{ij} \left(\MatrixForm{M}\vec{r}_i\right)_j}{k} \right) + \sum_{j=1}^d c_i\left(  \frac{\sum_{i=1}^k r_{ij} \left(\MatrixForm{B}\vec{r}_i\right)_j}{k} - b_{ii}\right),
\end{align*}
the optimal coefficients $\hat{c}_i$ are given by $\hat{c}_i = -\frac{m_{ii}}{b_{ii}}$, with variance reduction given by $\frac{2}{k}\sum_{s=1}^d m_{ss}^2 \geq \frac{2\trace{\MatrixForm{M}\MatrixForm{B}}^2}{k\trace{\MatrixForm{B}^2}}$ when $\MatrixForm{B}$ is a diagonal matrix, where $b_{ii} \neq 0$.  
\end{theorem}

\begin{proof}
The optimal coefficients $\hat{c}_i$ are given by
\begin{align}
\left(\begin{array}{c c c}
\Var{r_1 (\MatrixForm{B}\vec{r})_1}  & \hdots & \Cov{r_1(\MatrixForm{B}\vec{r})_1}{r_d(\MatrixForm{B}\vec{r})_d} \\
\vdots  & \ddots & \vdots \\
\Cov{r_d(\MatrixForm{B}\vec{r})_d}{r_1(\MatrixForm{B}\vec{r})_1} &  \hdots & \Var{r_d(\MatrixForm{B}\vec{r})_d}
\end{array}\right) \left(\begin{array}{c}
c_1 \\
\vdots \\
c_d
\end{array}\right) = -\left(\begin{array}{c}
\Cov{\tilde{y}}{r_1(\MatrixForm{B}\vec{r})_1} \\
\vdots \\
\Cov{\tilde{y}}{r_d(\MatrixForm{B}\vec{r})_1}
\end{array}\right), \label{optimal_CV_diag}
\end{align}

noting that $\frac{1}{k^2}$ cancels out on both sides. Computing the entries of the variance-covariance matrix in Equation~\eqref{optimal_CV_diag} gives
\begin{align*}
\Var{r_s(\MatrixForm{B}\vec{r})_s} & = \|\vec{b}_s\|^2 + b_{ss}^2 \\
\Cov{r_s(\MatrixForm{B}\vec{r})_s}{r_t(\MatrixForm{B}\vec{r})_t} & = b_{st}b_{ts}, 
\end{align*}
where $\vec{b}_s, \vec{b}_t$ denotes the $s^{\text{th}}, t^{\text{th}}$ row of $\MatrixForm{B}$. By linearity of covariances,
\begin{align}
\Cov{\tilde{y}}{r_s(\MatrixForm{B}\vec{r})_s} & = \sum_{i=1}^d \Cov{r_i(\MatrixForm{M}\vec{r})_i}{r_s(\MatrixForm{B}\vec{r})_s}\label{linearity_cov_diag_Hutch}.
\end{align}
Equation~\eqref{linearity_cov_diag_Hutch} leads to
\begin{align*}
\Cov{r_s(\MatrixForm{M}\vec{r})_s}{r_s(\MatrixForm{B}\vec{r})_s} & =  \left(\sum_{t=1}^d m_{st}b_{st}\right) + m_{ss}b_{ss}   \\
\Cov{r_s(\MatrixForm{M}\vec{r})_s}{r_t(\MatrixForm{B}\vec{r})_t} & = m_{st}b_{ts}, 
\end{align*}
and hence
\begin{align*}
\Cov{\tilde{y}}{r_s(\MatrixForm{B}\vec{r})_s} & = 2m_{ss}b_{ss}.
\end{align*}
As $\MatrixForm{B}$ is a diagonal matrix, the optimal coefficients $\hat{c}_i$ are now solved via
\begin{align*}
\left(\begin{array}{c c c}
2b_{11}^2  & \hdots & 0 \\
\vdots  & \ddots & \vdots \\
0 &  \hdots & 2b_{dd}^2
\end{array}\right) \left(\begin{array}{c}
c_1 \\
\vdots \\
c_d
\end{array}\right) = -\left(\begin{array}{c}
2m_{11}b_{11} \\
\vdots \\
2m_{dd}b_{dd}
\end{array}\right)
\end{align*}
leading to $\hat{c}_i = -  \frac{m_{ii}}{b_{ii}}$ and variance reduction given by
\begin{align*}
\ & \frac{1}{2} \left(\begin{array}{c}
2m_{11}b_{11} \\
\vdots \\
2m_{dd}b_{dd} \\
\end{array}\right)^T \left(\begin{array}{c c c}
\frac{1}{b_{11}^2} & \hdots & 0 \\
\vdots & \ddots & \vdots \\
0 & \hdots & \frac{1}{b_{dd}^2}
\end{array}\right)\left(\begin{array}{c}
2m_{11}b_{11} \\
\vdots \\
2m_{dd}b_{dd} \\
\end{array}\right) \\
 & = \frac{2}{k} \sum_{s=1}^d m_{ss}^2. 
\end{align*}
We now need to prove that the variance reduction is greater than the variance reduction given by Adams' CVE, and need to show that $2 \sum_{s=1}^d m_{ss}^2 - 2\frac{\trace{\MatrixForm{M}}^2}{d} \geq 0$ and do so by the Cauchy-Schwartz Inequality. The Cauchy–Schwarz Inequality states that for any $m_s, b_s \in \mathbb R$, 
\begin{align}
\left( \sum_{s=1}^d m_s^2 \right)\left( \sum_{s=1}^d b_s^2 \right) \geq \left(\sum_{s=1}^d m_sb_s\right)^2. \label{CS_inequality}
\end{align}

Setting $m_s = m_{ss}$, and $b_s = 1, 1 \leq s \leq d$ in Equation~\eqref{CS_inequality} gives the identity
\begin{align}
& \left( \sum_{s=1}^d m_{ss}^2 \right)\left( \sum_{s=1}^d 1^2 \right) \geq \left(\sum_{s=1}^d m_{ss}\right)^2    \notag \\
\Rightarrow & \sum_{s=1}^d m_{ss}^2 \geq \frac{1}{d} \left(\sum_{s=1}^d m_{ss} \right)^2   \notag \\
\Rightarrow & \sum_{s=1}^d m_{ss}^2 \geq \frac{1}{d} \left(\trace{M}\right)^2  \notag
\end{align}
with equality when $m_{ss}$ are equal for all $1 \leq s \leq d$.
\end{proof}

From the proof of Theorem~\ref{diag_CVE_thm}, the variance reduction from the CVE is independent of $b$, which implies $B=I$ should be used.

{\bf Heuristic 2: Apply Algorithm~\ref{FP_algo_format} given the CVE}
\label{heur2_appendix}

While $\hat{c}_i$ is in terms of $m_{ii}$, we {\it consider the heuristic of taking the limit of the fixed point iteration, similar to Algorithm~\ref{FP_algo_format}}. Suppose we have $d$ control variates of the form
\begin{align*}
\frac{\sum_{i=1}^k r_{is} \left(\MatrixForm{M}\vec{r}_i\right)_s}{k}  +  \sum_{s=1}^d c_i\left( \frac{r_{is} \left(\MatrixForm{B}\vec{r}_i\right)_s}{k} - b_{ss}\right) 
\end{align*}
for each $s$. The optimal $\hat{c}_i$ are given by
\begin{align*}
\left(\begin{array}{c c c}
2b_{11}^2  & \hdots & 0 \\
\vdots  & \ddots & \vdots \\
0 &  \hdots & 2b_{dd}^2
\end{array}\right) \left(\begin{array}{c}
c_1 \\
\vdots \\
c_d
\end{array}\right) = -\left(\begin{array}{c}
\Cov{r_s(\MatrixForm{M}\vec{r})_s}{r_1(\MatrixForm{B}\vec{r})_1} \\
\vdots \\
\Cov{r_s(\MatrixForm{M}\vec{r})_s}{r_s(\MatrixForm{B}\vec{r})_s} \\
\vdots \\ 
\Cov{r_s(\MatrixForm{M}\vec{r})_s}{r_d(\MatrixForm{B}\vec{r})_d} \\
\end{array}\right) = -\left(\begin{array}{c}
0 \\
\vdots \\
2m_{ss}b_{ss} \\
\vdots \\ 
0 \\
\end{array}\right),
\end{align*}
which implies the control variate coefficients $\hat{c}_i = 0, i \neq s$, giving the following estimate for $m_{ss}$
\begin{align}
m_{ss}  = \Expt{\sum_{i=1}^k \frac{r_{is}(\MatrixForm{M}\vec{r}_i)_s}{k} - \frac{m_{ss}}{b_{ss}} \left( \sum_{i=1}^k\frac{r_{is}(\MatrixForm{B}\vec{r}_i)_s}{k} - b_{ss}\right)} .\label{d_control_var_hutch_diag}
\end{align}
Dropping the expectation in Equation~\eqref{d_control_var_hutch_diag}, we get $d$ fixed point iterations
\begin{align*}
f_{n+1} & = \frac{\sum_{i=1}^k r_{is}(\MatrixForm{M}\vec{r}_i)_s}{k} - \frac{f_n}{b_{ss}} \left( \frac{\sum_{i=1}^k r_{is}(\MatrixForm{B}\vec{r}_i)_s}{k} - b_{ss}\right),~~1 \leq s \leq d.
\end{align*}
Assuming these fixed point iterations converge, we make $f_n$ the subject to get
\begin{align*}
\widehat{m}_{ss} = \frac{b_{ss}\frac{\sum_{i=1}^k r_{is}(\MatrixForm{M}\vec{r}_i)_s}{k} }{\frac{\sum_{i=1}^k r_{is}(\MatrixForm{B}\vec{r}_i)_s}{k} } = \frac{b_{ss}\sum_{i=1}^k r_{is}(\MatrixForm{M}\vec{r}_i)_s }{\sum_{i=1}^k r_{is}(\MatrixForm{B}\vec{r}_i)_s },~~~1 \leq s \leq d.
\end{align*}
and an estimate for $\trace{\MatrixForm{M}}$ with $k$ observations is
\begin{align}
\widehat{\trace{\MatrixForm{M}}}  = \sum_{s=1}^d  \frac{ b_{ss}\sum_{i=1}^k  r_{is}(\MatrixForm{M}\vec{r}_i)_s}{ \sum_{i=1}^k r_{is}(\MatrixForm{B}\vec{r}_i)_s }. \label{cv_diag_est_final_form}
\end{align}
which matches \citep{bekas2007estimator} when entries in $\vec{r}$ come from the Normal distribution. Moreover, the variance of the estimate of $\trace{\MatrixForm{M}}$ is
\begin{align}
\Var{\widehat{\trace{\MatrixForm{M}}}} = \frac{2\|\MatrixForm{M}\|_F^2 - 2\sum_{s=1}^d m_{ss}^2}{k}, \label{DIAG_EST_VAR}
\end{align}
with the intuition that error depends on the sum of the squares of the non-diagonal terms of $\MatrixForm{M}$. 

While we do not know if the conditions for Theorem~\ref{equal_variances_theorem_thingy} have been satisfied, taking the limit of the fixed point iterations gives us the estimator of \cite{bekas2007estimator}. Furthermore, this allows us to give some analysis on the approximate expectation and variance for this estimator, as \cite{bekas2007estimator} adopted a different approach.

Given that Equation~\eqref{cv_diag_est_final_form} is exact, we do not need to check for convergence. However, we do need to verify that Equation~\eqref{cv_diag_est_final_form} and Equation~\eqref{DIAG_EST_VAR} are reasonable. 

\label{heur2_appendix2}
\begin{proposition}
\label{convergence_theorem_trace_cv_diag}
The expression $\sum_{s=1}^d  \frac{ \sum_{i=1}^k b_{ss} r_{is}(\MatrixForm{M}\vec{r}_i)_s}{ \sum_{i=1}^k r_{is}(\MatrixForm{B}\vec{r}_i)_s }$ is an estimator (possibly biased) of $\trace{\MatrixForm{M}}$
with expectation (up to second order Taylor approximations) and variance (up to first-order Taylor approximations) given by
\begin{align*}
\Expt{ \sum_{s=1}^d  \frac{ \sum_{i=1}^k b_{ss} r_{is}(\MatrixForm{M}\vec{r}_i)_s}{ \sum_{i=1}^k r_{is}(\MatrixForm{B}\vec{r}_i)_s } } & = \trace{\MatrixForm{M}} \\
\Var{  \sum_{s=1}^d  \frac{ \sum_{i=1}^k b_{ss} r_{is}(\MatrixForm{M}\vec{r}_i)_s}{ \sum_{i=1}^k r_{is}(\MatrixForm{B}\vec{r}_i)_s } } & = \frac{2\|\MatrixForm{M}\|_F^2}{k} - \frac{2\sum_{s=1}^d m_{ss}^2}{k},
\end{align*}
matching the expectation and variance in Theorem~\ref{diag_CVE_thm}.
\end{proposition}

\begin{proof}
Let
\begin{align}
b_{ss}\frac{ \frac{\sum_{i=1}^k r_{is}(\MatrixForm{M}\vec{r}_i)_s}{k}}{ \frac{\sum_{i=1}^k r_{is}(\MatrixForm{B}\vec{r}_i)_s}{k}} & = b_{ss}\frac{X^{(s)}_1}{X^{(s)}_2} = b_{ss}g(X^{(s)}_1,X^{(s)}_2). \label{indiv_prop2_comp}
\end{align}

To find an approximation for the expectation and variance of $g(X_1^{(s)},X_2^{(s)})$, and covariance of $g(X_1^{(s)},X_2^{(s)}), g(X_1^{(t)},X_2^{(t)})$, we expand $g(X_1^{(s)},X_2^{(s)})$ around $(\Expt{X_1^{(s)}}, \Expt{X_2^{(s)}}) \equiv (m_{ss}, b_{ss})$ using the Taylor series approximation up to the second order terms. We write $\frac{\partial g^{(s)}}{\partial X_i^{(s)}}$ and $\frac{\partial^2 g^{(s)}}{\partial X_i^{(s)} \partial X_j^{(s)}}$ to mean the respective partial derivatives evaluated at $X_1^{(s)} = \Expt{X_1^{(s)}}, X_2^{(s)} = \Expt{X_2^{(s)}}$. Hence we expand
{\small
\begin{align}
g(X_1^{(s)},X_2^{(s)}) & = g\left(\Expt{X_1^{(s)}},\Expt{X_2^{(s)}}\right) + \frac{\partial g^{(s)}}{\partial X_1^{(s)}}(X_1^{(s)} - \Expt{X_1^{(s)}}) + \frac{\partial g^{(s)}}{\partial X_2^{(s)}}\left(X_2^{(s)} - \Expt{X_2^{(s)}}\right) \notag \\
 & \phantom{+} + \frac{1}{2}\left(\frac{\partial ^2 g^{(s)}}{\partial X_1^{(s)~2}}\left(X_1 - \Expt{X_1^{(s)}}\right)^2  + \frac{\partial ^2 g^{(s)}}{\partial X_2^{(s)~2}}\left(X_2^{(s)} - \Expt{X_2^{(s)}}\right)^2 \right. \notag \\
 & \phantom{+++++} \left. + 2\frac{\partial^2 g^{(s)}}{\partial X_1^{(s)} \partial X_2^{(s)}}\left(X_1^{(s)} - \Expt{X_1^{(s)}}\right)\left(X_2^{(s)} - \Expt{X_2^{(s)}}\right)\right) + R \label{taylor_expansion}
\end{align}
}

where $R$ is a remainder term. The partial derivatives evaluated at $X_1^{(s)} = \Expt{X_1^{(s)}}, X_2^{(s)} = \Expt{X_2^{(s)}}$ are
{\small
\begin{align*}
\frac{\partial g^{(s)}}{\partial X_1^{(s)}} = \frac{1}{\Expt{X_2^{(s)}}} = \frac{1}{b_{ss}}, \hspace*{1cm}\frac{\partial^2 g^{(s)}}{\partial X_1^{(s)~2}} = 0, & \hspace*{1cm} \frac{\partial g^{(s)}}{\partial X_2^{(s)}} = -\frac{\Expt{X_1^{(s)}}}{\Expt{X_2^{(s)}}^2} = - \frac{m_{ss}}{b_{ss}^2}  \\
\frac{\partial^2 g^{(s)}}{\partial X_1^{(s)} X_2^{(s)}} = \frac{\partial^2 g^{(s)}}{\partial X_2^{(s)} X_1^{(s)}} = -\frac{1}{\Expt{X_2^{(s)}}^2} = -\frac{1}{b_{ss}^2}, & \hspace*{1cm} \frac{\partial^2 g^{(s)}}{\partial X_2^{(s)~ 2}} = \frac{2\Expt{X_1^{(s)}}}{\Expt{X_2^{(s)}}^3} = \frac{2m_{ss}}{b_{ss}^3}.
\end{align*}
}

We use Equation~\eqref{taylor_expansion} to compute expectations $\Expt{g(X_1^{(s)}, X_2^{(s)}}$ and get
\begin{align}
\Expt{g(X_1^{(s)}, X_2^{(s)}} & = \frac{m_{ss}}{b_{ss}} + \frac{1}{2}\left( \frac{2m_{ss}}{b_{ss}^3}\Var{X_2} - \frac{2\Cov{X_1}{X_2}}{b_{ss}^2}       \right) +R\notag \\
& = \frac{m_{ss}}{b_{ss}} + \frac{1}{2}\left( \frac{4m_{ss}b_{ss}^2}{kb_{ss}^3} - \frac{4m_{ss}b_{ss}}{k b_{ss}^2}       \right)+R \notag \\
& \approx \frac{m_{ss}}{b_{ss}}. \label{prop2_for_meansz}
\end{align}
Expanding out $\Expt{\sum_{s=1}^d \frac{b_{ss}  \frac{\sum_{i=1}^kr_{is}(M\vec{r}_i)_s}{k} }{ \frac{\sum_{i=1}^k r_{is}(B\vec{r}_i)_s}{k}}}$ via Equation ~\eqref{prop2_for_meansz} gives
\begin{align*}
\Expt{\sum_{s=1}^d \frac{b_{ss}  \frac{\sum_{i=1}^kr_{is}(M\vec{r}_i)_s}{k} }{ \frac{\sum_{i=1}^k r_{is}(B\vec{r}_i)_s}{k}}} = \Expt{\sum_{s=1}^d b_{ss}\Expt{g(X_1^{(s)}, X_2^{(s)}}} \approx \sum_{s=1}^d m_{ss} = \trace{M} 
\end{align*}
as desired. Similarly,
\begin{align}
\Expt{g(X_1^{(s)}, X_2^{(s)})^2} & = \left(\frac{m_{ss}}{b_{ss}}\right)^2 + \frac{\Var{X_1}}{b_{ss}^2} + \frac{m_{ss}^2 \Var{X_2}}{b_{ss}^4} - 2\frac{m_{ss}\Cov{X_1}{X_2}}{b_{ss}^3} +R  \notag \\
 & = \left(\frac{m_{ss}}{b_{ss}}\right)^2 + \frac{\|\vec{a}_s\|^2 + m_{ss}^2}{k b_{ss}^2} + \frac{2m_{ss}^2b_{ss}^2}{k b_{ss}^4} - \frac{4m_{ss}^2b_{ss}}{k b_{ss}^3 } + R \notag \\
  & = \left(\frac{m_{ss}}{b_{ss}}\right)^2 + \frac{\|\vec{a}_s\|^2 + m_{ss}^2}{k b_{ss}^2} - \frac{2m_{ss}^2}{k b_{ss}^2 } + R, \label{prop2_for_var}
\end{align}
and we compute $\Var{g(X_1^{(s)}, X_2^{(s)})}$ by Equations~\eqref{prop2_for_meansz} and \eqref{prop2_for_meansz}, giving
\begin{align}
\Var{g(X_1^{(s)}, X_2^{(s)})} & = \Expt{g(X_1^{(s)}, X_2^{(s)})^2} - \Expt{g(X_1^{(s)}, X_2^{(s)})}^2 + R \notag \\
 & \approx \left(\frac{m_{ss}}{b_{ss}}\right)^2 + \frac{\|\vec{a}_s\|^2 + m_{ss}^2}{k b_{ss}^2} - \frac{2m_{ss}^2}{k b_{ss}^2 } - \left(\frac{m_{ss}}{b_{ss}}\right)^2 \notag \\
  & = \frac{\|\vec{a}_s\|^2 + m_{ss}^2}{k b_{ss}^2} - \frac{2m_{ss}^2}{k b_{ss}^2 }. \label{prop2_var_expr}
\end{align}
We now consider the expectation of the product of $g(X_1^{(s)}, X_2^{(s)})g(X_1^{(t)},X_2^{(t)})$ using first order terms in the Taylor expansion for any $1 \leq s,t \leq d$.
{\small
\begin{align}
\Expt{g(X_1^{(s)}, X_2^{(s)})g(X_1^{(t)},X_2^{(t)})} & = \frac{m_{ss}}{b_{ss}}\frac{m_{tt}}{b_{tt}} + \frac{\partial g^{(s)}}{\partial X_1^{(s)}}\frac{\partial g^{(t)}}{\partial X_1^{(t)}}\Cov{X_1^{(s)}}{X_1^{(t)}} \notag \\
& \phantom{=} + \frac{\partial g^{(s)}}{\partial X_1^{(s)}}\frac{\partial g^{(t)}}{\partial X_2^{(t)}}\Cov{X_1^{(s)}}{X_2^{(t)}} + \frac{\partial g^{(s)}}{\partial X_2^{(s)}}\frac{\partial g^{(t)}}{\partial X_1^{(t)}}\Cov{X_2^{(s)}}{X_1^{(t)}} \notag \\
& \phantom{=} + \frac{\partial g^{(s)}}{\partial X_2^{(s)}}\frac{\partial g^{(t)}}{\partial X_2^{(t)}}\Cov{X_2^{(s)}}{X_2^{(t)}} +R \notag \\
 & = \frac{m_{ss}}{b_{ss}}\frac{m_{tt}}{b_{tt}} + \frac{\Cov{X_1^{(s})}{X_1^{(t)}}}{b_{ss}b_{tt}} - \frac{\Cov{X_1^{(s)}}{X_2^{(t)}}m_{tt}}{b_{ss}b_{tt}^2} \notag \\
 & \phantom{=} - \frac{\Cov{X_2^{(s)}}{X_1^{(t)}}m_{ss}}{b_{ss}^2b_{tt}} + \frac{\Cov{X_2^{(s)}}{X_2^{(t)}}m_{ss}m_{tt}}{b_{ss}^2b_{tt}^2} + R \notag \\
  & = \frac{m_{ss}}{b_{ss}}\frac{m_{tt}}{b_{tt}} + \frac{m_{st}m_{ts}}{kb_{ss}b_{tt}} - \frac{m_{st}b_{ts}m_{tt}}{kb_{ss}b_{tt}^2}  - \frac{m_{ts}b_{st}m_{ss}}{kb_{ss}^2b_{tt}} + \frac{b_{st}b_{ts}m_{ss}m_{tt}}{kb_{ss}^2b_{tt}^2} + R \notag \\
  & \approx \frac{m_{ss}}{b_{ss}}\frac{m_{tt}}{b_{tt}} + \frac{m_{st}^2}{kb_{ss}b_{tt}}, \label{gx1gx2_prop2}
\end{align}
}

where Equation~\eqref{gx1gx2_prop2} follows due to the fact that $\MatrixForm{B}$ was chosen to be diagonal, hence $b_{st} = 0, s\neq t$, and $\MatrixForm{M}$ is symmetric, so $m_{st} = m_{ts}$. This gives $\Cov{g(X_1^{(s)}, X_2^{(s)})}{g(X_1^{(t)}, X_2^{(t)})}$ to be
\begin{align}
\Cov{g(X_1^{(s)}, X_2^{(s)})}{g(X_1^{(t)}, X_2^{(t)})} & = \frac{m_{ss}}{b_{ss}}\frac{m_{tt}}{b_{tt}} + \frac{m_{st}^2}{kb_{ss}b_{tt}} - \frac{m_{ss}}{b_{ss}}\frac{m_{tt}}{b_{tt}} = \frac{m_{st}^2}{kb_{ss}b_{tt}}. \label{prop2_cov_expr}
\end{align}
Combining Equations~\eqref{prop2_var_expr} and \eqref{prop2_cov_expr} finally gives
\begin{align*}
\Var{\sum_{s=1}^d \frac{b_{ss}  \frac{\sum_{i=1}^kr_{is}(\MatrixForm{M}\vec{r}_i)_s}{k} }{ \frac{\sum_{i=1}^k r_{is}(\MatrixForm{B}\vec{r}_i)_s}{k}}} & = \sum_{s=1}^d \Var{b_{ss} g(X_1^{(s)}, X_2^{(s)})} \notag \\
& \phantom{=} + \sum_{s,t}^d \Cov{b_{ss} g(X_1^{(s)}, X_2^{(s)})}{b_{tt}g(X_1^{(t)}, X_2^{(t)})} \\
 & = \sum_{s=1}^d \left(\frac{\|\vec{a}_s\|^2 + m_{ss}^2}{k} - \frac{2m_{ss}^2}{k} \right) + \sum_{s,t}^d \frac{m_{st}^2}{k} +R \\
 & \approx \frac{2\|\MatrixForm{M}\|_F^2}{k} - \frac{2m_{ss}^2}{k}.
\end{align*}
\end{proof}

The implications of Proposition~\ref{convergence_theorem_trace_cv_diag} implies Equation~\eqref{cv_diag_est_final_form} as an estimator for $\trace{\MatrixForm{M}}$. More precisely, this estimator from \citep{bekas2007estimator} naturally appears as the ``limit" of the optimal CVE. Equation~\eqref{DIAG_EST_VAR} also complements the probability bounds in \citep{baston2022stochastic} for each diagonal term of $\MatrixForm{M}$.

To conclude, treating the CVE as a fixed point iteration and finding its closed form can lead to more efficient estimators than by using CVE weights once, and then the estimators analyzed directly. With respect to our field in sketching algorithms, our heuristic can lead to re-examining past work on estimators to give alternate proofs for variance reduction (and using these variances in probability bounds), and future work on understanding new estimators along similar lines, e.g. applying our work to find better estimators for estimating Euclidean distances under composition of functions \citep{leroux2024euclidean}, or for the Jaccard similarity under circulant permutations \citep{pmlr-v162-li22m}. In both these cases, computing expectations and covariances (hence CVE weights) are much easier than analytically deriving the MLE involved.
\end{document}